\documentclass[10pt,twocolumn,letterpaper]{article}

\usepackage{cvpr}
\usepackage{times}
\usepackage{epsfig}
\usepackage{graphicx}
\usepackage{amsmath}
\usepackage{amssymb}

\usepackage{amsfonts,latexsym}
\usepackage{placeins}
\usepackage{enumitem}
\usepackage[english, algoruled, noline]{algorithm2e}
\usepackage{url}
\usepackage{subfig}
\usepackage{epstopdf}
\usepackage{authblk}
\usepackage{appendix}
\renewcommand\appendix{\par
\setcounter{section}{0}%
\setcounter{subsection}{0}%
\setcounter{table}{0}
\setcounter{figure}{0}
\gdef\thetable{\Alph{table}}
\gdef\thefigure{\Alph{figure}}
\section*{Appendix}
\gdef\thesection{\Alph{section}}
\setcounter{section}{0}}

\usepackage[pagebackref=true,breaklinks=true,letterpaper=true,colorlinks,bookmarks=false]{hyperref}

\newtheorem{theorem}{Theorem}[section]
\newtheorem{lemma}[theorem]{Lemma}
\newtheorem{definition}[theorem]{Definition}
\newtheorem{proposition}[theorem]{Proposition}

\newenvironment{proof}{\par\noindent{\bf Proof:\ }}{\hfill$\Box$\\[2mm]}

\def\RR{\mathbb{R}}

\def\FC{\mathcal{F}}

\def\GC{\mathcal{G}}

\def\bx{\mathbf{x}}
\def\by{\mathbf{y}}
\def\bz{\mathbf{z}}
\def\bt{\mathbf{t}}

\def\tbx{\tilde{\mathbf{x}}}
\def\tby{\tilde{\mathbf{y}}}
\def\tbz{\tilde{\mathbf{z}}}
\def\tbt{\tilde{\mathbf{t}}}

\newcommand{\bu}{\mathbf{u}}
\def\bv{\mathbf{v}}

\def\MF{\textit{F}}

\def\MG{\textit{G}}

\def\<{\langle}
\def\>{\rangle}

\newcommand{\inner}[1]{\left\langle#1\right\rangle}

\def\X{\mathcal{X}}

\def\R{\mathbb{R}}

\newcommand{\norm}[1]{\left\|#1\right\|}
\newcommand{\abs}[1]{\left|#1\right|}

\def\Nc{\mathcal{N}}

\def\argmax{\mathop{\rm arg\,max}\limits}

\def\max{\mathop{\rm max}\nolimits}

\cvprfinalcopy 

\def\cvprPaperID{1850} 

\ifcvprfinal\fi

\begin{document}

\title{A Flexible Tensor Block Coordinate Ascent Scheme for Hypergraph Matching}



\author{
Quynh Nguyen Ngoc$^{1,2}$, Antoine Gautier$^2$, Matthias Hein$^2$ \\
$^1$Max Planck Institute for Informatics, Saarbr\"ucken, Germany \\
$^2$Saarland University, Saarbr\"ucken, Germany \\
}

\maketitle
\thispagestyle{empty}

\begin{abstract}
     The estimation of correspondences between two images resp. point sets is a core problem in computer vision.
One way to formulate the problem is graph matching leading to the quadratic assignment problem which is NP-hard.
Several so called second order methods have been proposed to solve this problem.
In recent years hypergraph matching leading to a third order problem became popular
as it allows for better integration of geometric information.
For most of these third order algorithms no theoretical guarantees are known. In this paper we propose
a general framework for tensor block coordinate ascent methods for hypergraph matching. We propose two algorithms
which both come along with the guarantee of monotonic ascent in the matching score on 
the set of discrete assignment matrices. In the experiments we show that our new algorithms
outperform previous work both in terms of achieving better matching scores and matching accuracy.
This holds in particular for very challenging settings where one has
a high number of outliers and other forms of noise.
\end{abstract}

\section{Introduction}\label{sec:intro}
%

Graph resp. hypergraph matching has been used in a variety of problems in computer vision
such as object recognition \cite{Low1999}, feature correspondences \cite{ChoLee2012,TorEtal2008}, shape matching 
\cite{DucJouPon2011,SharEtal2011,ZasEtl2009} and surface registration \cite{Zeng2010}. 
Given two sets of points, the task is to find the correspondences between them based on extracted features and/or geometric
information. In general, the graph matching problem is NP-hard, therefore, many approximation algorithms have been proposed
over the years.
They can be grouped into second order and third order methods.

Among recent second order approaches, Leordeanu and Hebert \cite{LeoHeb2005} proposed the Spectral Matching (SM) algorithm, 
and Cour \etal ~\cite{ShiEtal2007} the Spectral Matching with Affine Constraint (SMAC). 
Both of these methods are based on the best rank-1 approximation of an affinity matrix. 
Leordeanu \etal later on proposed the Integer Projected Fixed Point (IPFP) algorithm \cite{LeoHebSuk2009}, where they optimize the 
quadratic objective using a gradient-type algorithm interleaved with projection onto the discrete constraint set using
e.g. the hungarian method.
Lee \etal ~\cite{LeeChoLee2010} tackle the graph matching problem using stochastic sampling, whereas
Cho \etal ~\cite{ChoLeeLee2010} propose a reweighted random walk (RRWM) where the reweighting jumps are aiming
at enforcing the matching constraints.
Recently, Cho \etal ~\cite{ChoEtAl2014} proposed a novel max pooling matching (MPM), where they tweak the power method to 
better cope with noise in the affinities.
Although the algorithm comes without theoretical guarantees, it turns out to perform very well in practice, in particular, when
one has a large number of outliers.
Other second order methods include the work of Zhou and Torre \cite{ZhoTor2013} and Zaslavskiy \etal ~\cite{ZasEtl2009}, 
where they propose deformable graph matching (DGM) and a path-following algorithm respectively.

As second order methods are limited to pairwise similarities between two correspondences, higher order methods can integrate
better geometric information. While they have higher computational complexity, they can cope better with geometric transformations such as scaling or other forms of noise. Compared to second order methods, higher order approaches are less well studied in literature due to the 
difficulty in handling the third order optimization problem.

Duchenne \etal ~\cite{DucEtAl2011} formulated the hypergraph matching problem 
as a third order tensor optimization problem and proposed a higher order power method for solving the problem.
This approach has shown significant improvements over second order ones which were proposed earlier.
Zass and Shashua \cite{ZasSha2008} introduced a probabilistic view to the problem, and proposed
a Hypergraph Matching (HGM) algorithm. Their idea is to marginalize the tensor to a vector and solve again
a lower dimensional problem.
Chertok and Keller \cite{CheKel2010} extended this idea and marginalized 
the tensor to a matrix, resulting in a second order matching problem which is solved by spectral methods.
Both methods are based on tensor marginalization, which leads to a loss of information. 
Moreover, 
they cannot effectively handle the one-to-one matching constraint during the iterations which is only considered at the final discretization step.
Jungminlee \etal ~\cite{LeeChoLee2011} extended the reweighted random walk approach of \cite{LeeChoLee2010}
to the third order setting. Their algorithm aims at enforcing the matching constraint via
a bi-stochastic normalization scheme done for each iteration.
In \cite{Zeng2010}, Zeng \etal proposed to use pseudo-boolean optimization \cite{Boros2002} for 3D surface matching,
where higher order terms are decomposed into second order terms and then 
a quadratic pseudo-boolean optimization (QPBO) \cite{Kolmogorov2007} algorithm is used to solve the problem.

In this paper, we propose an algorithmic framework for solving the hypergraph matching problem
based on a tensor block coordinate scheme.
The key idea of our framework is to use a multilinear reformulation of the original objective function.
In particular, we solve the third order graph matching problem by solving an equivalent fourth order one.
Based on this reformulation, we derive two new hypergraph matching algorithms.
We can guarantee for both algorithms monotonic ascent in the third order matching score. In the experiments we show that our new algorithms
   outperform previous work both in terms of achieving better matching score and accuracy. This holds in particular for very challenging settings where one has
   a high number of outliers and other forms of noise. 

\section{Hypergraph Matching Formulation}\label{sec:form}
%
We formulate the correspondence problem as a hypergraph matching problem. 
Hypergraphs allow modeling of relations which are not only pairwise as in graphs but involve groups of vertices. 
In this paper, we consider $3$-uniform hypergraphs, that is each hyperedge contains $3$ vertices. 
$k$-uniform hypergraphs can be modelled as $k$-th order tensors which is the point of view we adopt in this paper.

Given two attributed point sets $G=(V,A)$ and $G'=(V',A')$ with $n_1= |V|\leq n_2= |V'|$, the matching problem can be formulated as finding a subset $\X$ in the set of correspondences $V \times V'$. 
The subset $\X$ can be represented by the binary assignment matrix $X \in \{0,1\}^{n_1 \times n_2}$ where $X_{ij}=1$ if $v_{i} \in V$ matches $v'_{j} \in V'$ and $X_{ij}=0$ else. A one-to-one matching $X$ is an element of the set
\[ M=\big\{X \in \{0,1\}^{n_1 \times n_2} \,\big|\, \sum_{i=1}^{n_1} X_{ij}\leq 1, \quad \sum_{j=1}^{n_2} X_{ij} = 1\big\},\]
that is we consider matchings which assign each vertex of $V$ to exactly one vertex in $V'$.
We further assume that we have a function $\FC^3: (V \times V')^3\to\R_+$ which assigns to each chosen triple $\{v_{i_1},v_{i_2},v_{i_3}\}\subset V$ and 
$\{v'_{j_1},v'_{j_2},v'_{j_3}\} \subset V'$ its similarity weight $\FC^3_{(i_1,j_1),(i_2,j_2),(i_3,j_3)}$. 
Finally, the score $S$ of a matching $X$ can then be defined as
\[ S(X)=\sum_{i_1,i_2,i_3}^{n_1} \sum_{j_1,j_2,j_3}^{n_2} \FC^3_{(i_1,j_1),(i_2,j_2),(i_3,j_3)} X_{i_1 j_1} X_{i_2 j_2} X_{i_3 j_3}.\]
In order to facilitate notation we introduce a linear ordering in $V \times V'$  and thus can rewrite $X$ as a vector $\bx \in \{0,1\}^{n}$ with $n=n_1 n_2$
and $\FC^3$ becomes a tensor in $\R^{n \times n \times n}$ so that the matching score, $S^3:\R^n \to\R$, is finally written as
\begin{equation}\label{eq:score}
S^3(\bx)=\sum_{i,j,k=1}^n \,\FC^3_{ijk}\, \bx_i\, \bx_j\, \bx_k.
\end{equation}
An $m$-th order tensor $\GC\in\R^{n\times \ldots\times n}$ is called symmetric if its entries $\GC_{i_1\ldots i_m}$ are invariant under any permutation of their indices $\{i_1,\ldots,i_m\}$. Note that $S^3$ is the sum of the componentwise product of the tensor $\FC^3$ with the symmetric tensor $\GC^3_{ijk}=\bx_i\bx_j\bx_k$.
Thus any non-symmetric part of $\FC^3$ is ``averaged'' out
and we can without loss of generality assume that $\FC^3$ is a symmetric tensor\footnote{If $\FC^3$ is not already 
symmetric, then one can define 
$\tilde{\FC^3}_{ijk}=\frac{1}{3!}\sum_{\sigma \in\mathfrak{G}_3}\FC^3_{\sigma(i)\sigma(j)\sigma(k)}$, $i,j,k=1,\ldots,n$ 
where $\mathfrak{G}_3$ is the set of permutations of three elements.}.
In principle, one can integrate unary terms on the diagonal $\FC^3_{iii}$, $i=1,\ldots,n$
and pairwise potentials $\FC^3_{ijj}$, $i\neq j$. However, the tensors we consider in Section \ref{sec:experiments} have just terms of order $3$, that is
$\FC^3_{ijk}=0$ if $i=j$, $i=k$ or $j=k$.

\section{Mathematical Foundations for the Tensor-Block-Coordinate Ascent Method}
In this section we derive the basis for our tensor block coordinate ascent method to optimize $S^3$ directly over the set $M$ of assignment matrices. 
The general idea is to optimize instead of a homogeneous polynomial function the associated multilinear form.
\subsection{Lifting the Tensor and Multilinear Forms}
%
Instead of optimizing the third order problem in Equation \eqref{eq:score}, 
we define a corresponding fourth order problem by lifting the third order tensor $\FC^3$ to a fourth order tensor $\FC^4$,
\begin{equation}\label{eq:gmTensor4}
    \FC^4_{ijkl} = \FC^3_{ijk} + \FC^3_{ijl} + \FC^3_{ikl} + \FC^3_{jkl}.
\end{equation}
This might be counter-intuitive at first in particular as previous work \cite{ZasSha2008,CheKel2010} has done the opposite
way by reducing the third order problem to a second order problem.
However, lifting the tensor from a third order tensor to a fourth order tensor allows us to use structure of even order tensor which is not present
for tensors of odd order. In particular, the score function $S^3$ is not convex.
\begin{lemma}
Let $S^3$ be defined as in Equation \eqref{eq:score}. If $S^3$ is not constant zero, then $S^3:\R^n \to\R$ is not convex.
\end{lemma}
\begin{proof}
The Hessian $H\!S^3$ of $S^3$ satisfies 
\begin{equation*}H\!S^3(\bx)_{jk} = 6 \sum_{i=1}^n \FC^3_{ijk} \bx_i,\qquad \forall 1\leq j,k\leq n,\end{equation*} 
for every $\bx\in\R^n$, i.e. $H\!S^3(\bx)=6\MF^3(\bx,\cdot,\cdot)$.
Now, if $S^3$ is convex, then 
\begin{equation*}0\leq \inner{\by,H\!S^3(\bx)\by} = 6 \MF^3(\bx,\by,\by)\qquad \forall \bx,\by\in\R^n.\end{equation*} 
It follows that \begin{equation*}0\leq \MF(\bx,\by,\by)\ \text{and}\ 0\leq\MF(-\bx,\by,\by)=-\MF(\bx,\by,\by)\end{equation*} 
for every $\bx,\by\in\R^n$. In particular, for $\bx=\by$ we get $0=\MF(\bx,\bx,\bx)=S^3(\bx)$ for every $\bx\in\R^n$.
\end{proof}

The convexity is crucial to derive our block coordinate ascent scheme.
As mentioned earlier, we do not work with the score function directly but with the multilinear form associated to it.
\begin{definition}[Multilinear Form]
The multilinear form $\MF^m\colon\R^n \times \ldots \times \R^n \to\R$ associated to an $m$-th order tensor $\FC^m$ is given by:
\begin{equation}
    \MF^m(\bx^1, \ldots, \bx^m) = \sum_{i_1, \ldots, i_m}^n \FC^m_{i_1 \ldots i_m} \bx_{i_1}^1 \ldots \bx_{i_m}^m \nonumber,
\end{equation}
and the score function $S^m\colon \R^n\to\R$ associated to $\FC^m$ is defined by $S^m(\bx)=\MF^m(\bx,\ldots,\bx)$.
\end{definition}
For simplicity, the superscript in the multilinear form can be omitted when there is no ambiguity.
For example, $\MF(\bx,\by,\bz,\bt)$ can be used interchangeably with $\MF^4(\bx,\by,\bz,\bt)$ in the same context 
since the number of variables of this form is four which already implies a fourth order tensor.
Also, we write $\MF^4(\bx,\bx,\bx,\cdot)$ to denote the vector in $\RR^n$ such that
$$\MF^4(\bx,\bx,\bx,\cdot)_l = \sum_{i, j, k=1}^n \FC^4_{ijkl} \bx_i \bx_j \bx_k ,\qquad \forall 1 \leq l \leq n.$$ 
Similarly, we write $\MF^4(\bx,\bx,\cdot,\cdot)$ to denote the matrix in $\RR^{n\times n}$ such that 
$$\MF^4(\bx,\bx,\cdot,\cdot)_{kl} = \sum_{i, j=1}^n \FC^4_{ijkl} \bx_i \bx_j ,\qquad \forall 1 \leq k,l \leq n.$$
Note that if $\MF^4$ is associated to a symmetric tensor, then the position of the dots does not matter. 
For example, if $\FC^4$ is symmetric then one can check that 
$\MF^4(\bx,\bx,\bx,\cdot) = \MF^4(\bx,\bx,\cdot,\bx) = \MF^4(\bx,\cdot,\bx,\bx) = \MF^4(\cdot,\bx,\bx,\bx)$.
Similar properties hold for multilinear forms of other orders.

It might seem at first that the lift from third to fourth order defined in \eqref{eq:gmTensor4} 
should lead to a huge computational overhead. However, 
\begin{equation}\label{eq:eqS4S3}
	S^4(\bx)
	= 4\,\MF^3(\bx,\bx,\bx) \,\sum_{i=1}^n \bx_i
	= 4\,S^3(\bx)\,\sum_{i=1}^n \bx_i
\end{equation}
and $\sum_{i=1}^n \bx_i=n_1,$ for all $\bx \in M$, thus we have the following equivalence of the
optimization problems,
\begin{equation}\label{eq:liftedProblem}
    \max\limits_{\bx \in M} \MF^4(\bx,\bx,\bx,\bx)\;\equiv\; \max\limits_{\bx \in M} \MF^3(\bx,\bx,\bx)
\end{equation} 

\subsection{Convex Score Functions and Equivalence of Optimization
Problems}
The use of multilinear forms corresponding to convex score
functions is crucial for the proposed algorithms.
The following lemma shows that even if we optimize the multilinear form $\MF^4$ instead of $S^4$ we
get ascent in $S^4$. We can either fix all but one argument or all but two arguments, which leads to the two
variants of the algorithm in Sec. \ref{sec:algorithms}.
\begin{lemma}
    Let $\FC^4$ be a fourth order symmetric tensor.
    If $S^4\colon\R^n \to\R$ is convex, then for all $\bx, \by, \bz,\bt \in \RR^n$:
    \begin{enumerate}[leftmargin=*]
	\item 
		$
		    \MF^4(\bx,\bx,\by,\by) \leq \max\limits_{\bu \in \{ \bx,\by\}} \MF^4(\bu,\bu,\bu,\bu),
		$
		\item
		$
		    \MF^4(\bx,\by,\bz,\bt) \leq \displaystyle \max\limits_{\bu \in \{ \bx,\by,\bz,\bt \}} \MF^4(\bu,\bu,\bu,\bu).
		$
    \end{enumerate}
    \label{lem:main}
\end{lemma}
\begin{proof}
   First, we prove a characterization for the convexity of $S^4$.
   		The gradient and the Hessian of $S^4$ are given by
   		\begin{eqnarray*}
   		 \nabla S^4(\bx) \!\!\! &=& \!\!\! 4\MF^4(\bx,\bx,\bx,\cdot),\\
   			H\!S^4(\bx)\!\!\!  &=& \!\!\! 12\MF^4(\bx,\bx,\cdot,\cdot),
   		\end{eqnarray*}
   		where we used the symmetry of $\FC^4$.
   			$S^4$ is convex if and only if $H\!S^4(\bx)$
   			is positive semi-definite for all $\bx \in \R^n$ which is equivalent to
   			\begin{equation}\label{conv_charac} \inner{\by,H\!S^4(\bx)\,\by} = 12 \MF^4(\bx,\bx,\by,\by)\geq 0, \forall \bx,\by \in \R^n.
   			\end{equation}
	\begin{enumerate}
		    \item
			By \eqref{conv_charac} and the multilinearity of $\MF$, we have
	\begin{eqnarray}
	0\!\!\!&\leq &\!\!\!\MF(\bx-\by,\bx-\by,\bx+\by,\bx+\by)\\
	&= &\!\!\!\MF(\bx,\bx,\bx,\bx)+\MF(\by,\by,\by,\by)  -2\MF(\bx,\bx,\by,\by) \nonumber
	\end{eqnarray}
	for every $\bx, \by \in \RR^n.$ It follows that
			  \begin{eqnarray*}
				  2 \MF(\bx,\bx,\by,\by)\!\!\! &\leq&\!\!\! \MF(\bx,\bx,\bx,\bx) + \MF(\by,\by,\by,\by) \nonumber \\
				  &\leq& \!\!\! 2  \max\limits_{\bu \in \{ \bx,\by\}} \MF^4(\bu,\bu,\bu,\bu).
			  \end{eqnarray*}
			  for all $\bx, \by \in \RR^n.$
		    \item Similarly, for all $\bx,\by,\bz,\bt \in \RR^n$, we have
		    		  \begin{eqnarray} \label{ineq1}
		    			  0\!\!\! & \leq & \!\!\!\MF(\bx+\by, \bx+\by, \bz-\bt, \bz-\bt) \nonumber\\
		    			  &=&\!\!\! \MF(\bx,\bx,\bz,\bz) + \MF(\bx,\bx,\bt,\bt) +\MF(\by,\by,\bz,\bz)\nonumber \\ 
		    			  &&\!\!\!+\MF(\by,\by,\bt,\bt)+ 2\MF(\bx,\by,\bz,\bz)\\ 
		    			  &&\!\!\!+2\MF(\bx,\by,\bt,\bt)\nonumber-2\MF(\bx,\bx,\bz,\bt)\nonumber\\ 
		    			  &&\!\!\!- 2\MF(\by,\by,\bz,\bt)-4\MF(\bx,\by,\bz,\bt).\nonumber
		    		  \end{eqnarray}
			  Switching the variables from ($\bx,\by,\bz,\bt$) to ($\bz,\bt,\bx,\by$) and applying the same inequality, we get
			 \begin{eqnarray} \label{ineq2}
			 			  0\!\!\! &\leq& \!\!\!\MF(\bx-\by, \bx-\by, \bz+\bt, \bz+\bt) \nonumber\\
			 			  &=& \!\!\!\MF(\bx,\bx,\bz,\bz) + \MF(\bx,\bx,\bt,\bt) + \MF(\by,\by,\bz,\bz) \nonumber\\ 
			 			  && \!\!\!+ \MF(\by,\by,\bt,\bt) - 2\MF(\bx,\by,\bz,\bz) \\ 
			 			  &&\!\!\!- 2\MF(\bx,\by,\bt,\bt) + 2\MF(\bx,\bx,\bz,\bt) \nonumber\\
			 			 &&\!\!\!+ 2\MF(\by,\by,\bz,\bt) - 4\MF(\bx,\by,\bz,\bt). \nonumber
	        \end{eqnarray}
			  Summing up inequalities \eqref{ineq1} and \eqref{ineq2} we obtain:
			  \begin{eqnarray}
			    4\MF(\bx,\by,\bz,\bt)\!\!\! &\leq&\!\!\! \MF(\bx,\bx,\bz,\bz) + \MF(\bx,\bx,\bt,\bt) \\ 
				&& \!\!\!+ \MF(\by,\by,\bz,\bz) + \MF(\by,\by,\bt,\bt). \nonumber
			  \end{eqnarray}
			  Finally, applying the first result finishes the proof.
		\end{enumerate}
\end{proof}
The following theorem shows that the optimization of the multilinear form $\MF^4$ 
and the score function $S^4$ are equivalent in the sense that there exists
a globally optimal solution of the first problem which is also globally optimal for the second problem.
\begin{theorem}{
    Let $\FC^4$ be a fourth order symmetric tensor and suppose that $S^4\colon\R^n \to \R$ is convex. 
    Then it holds for any compact constraint set $D \subset \R^n$,
    \begin{eqnarray}
    \max\limits_{\bx \in D} \MF^4(\bx,\bx,\bx,\bx)\!\!\!
	    &=& \!\!\!\max\limits_{\bx,\by \in D} \MF^4(\bx,\bx,\by,\by) \\
	    &=& \!\!\!	\max\limits_{\bx,\by,\bz,\bt \in D} \MF^4(\bx,\by,\bz,\bt)  \nonumber
	  \end{eqnarray}
    \label{theo:main_order4}}
\end{theorem}
\begin{proof}
   For any compact set $D \subset \RR^n$ it holds:
     \begin{eqnarray}
     \max\limits_{\bx \in D} \MF^4(\bx,\bx,\bx,\bx) \!\!\!  &\leq& \!\!\! \max\limits_{\bx,\by \in D} \MF^4(\bx,\bx,\by,\by)\\
     &\leq& \!\!\! \max\limits_{\bx,\by,\bz,\bt \in D} \MF^4(\bx,\by,\bz,\bt). \nonumber
     \end{eqnarray}
     However, the second inequality in Lemma \ref{lem:main} shows
     \begin{equation}
        \MF^4(\bx,\by,\bz,\bt) \leq \displaystyle \max\limits_{\bu \in \{ \bx,\by,\bz,\bt \}} \MF^4(\bu,\bu,\bu,\bu) 
     \end{equation} which leads to
     \begin{equation}
        \max\limits_{\bx,\by,\bz,\bt \in D} \MF^4(\bx,\by,\bz,\bt) \leq \max\limits_{\bx \in D} \MF^4(\bx,\bx,\bx,\bx),
     \end{equation}
     and the proof is done.
\end{proof}

As we cannot guarantee that $S^4$ is convex for our chosen affinity tensor, we propose a modification making it convex.
Importantly, this modification turns out to be constant on the set of matchings $M$.
\begin{proposition}
    Let $\FC^4$ be a fourth order symmetric tensor. 
    Then for any $\alpha \geq 3\norm{\FC^4}_2$, where $\norm{\FC^4}_2:=\sqrt{\sum_{i,j,k,l=1}^n (\FC_{ijkl}^4)^2}$, the function
    \begin{equation*}\label{eq:convexMod}
	 S^4_\alpha(\bx) := S^4(\bx) + \alpha \,\norm{\bx}_2^4
    \end{equation*}
     is convex on $\RR^n$ and, for any $\bx \in M$, we have
     \begin{equation*}\label{eq:Modconstant}
      S^4_\alpha(\bx)=S^4(\bx) + \alpha\, n_1^2.
     \end{equation*}
     \label{prop:convexify}
\end{proposition}
\vspace{-20pt}
\begin{proof}
The gradient and the Hessian of $S^4_\alpha$ can be computed as:
\begin{eqnarray*}
      \nabla S^4_\alpha(\bx)\!\!\!  &=&\!\!\!  4 \MF^4(\bx,\bx,\bx,\cdot) + 4 \alpha \norm{\bx}_2^2 \bx,\\
      H\!S^4_\alpha(\bx)\!\!\!  &=& \!\!\! 12 \MF^4(\bx,\bx,\cdot,\cdot) + 8 \alpha \bx \bx^T + 4 \alpha \norm{\bx}_2^2 I,\qquad
\end{eqnarray*} where $I$ is the identity matrix.
$S^4_\alpha$ is convex if and only if \begin{equation*}\< \by, H\!S^4_\alpha(\bx) \by\> \geq 0\qquad \forall \bx,\by\in\R^n.\end{equation*} 
This is equivalent to
\begin{equation}\label{eq:hessian_S_alpha}
    12 \MF^4(\bx,\bx,\by,\by)  + 8 \alpha \inner{\bx,\by}^2 + 4 \alpha \norm{\bx}_2^2\norm{\by}^2_2 \geq 0
\end{equation} 
for every $\bx,\by\in\R^n$. By the Cauchy-Schwarz inequality we have 
\begin{eqnarray*}
\abs{\MF^4(\bx,\bx,\by,\by)}\!\!\! &=&\!\!\! \abs{\sum_{i,j,k,l=1}^n \FC^4_{ijkl} \bx_i \bx_j \by_k \by_l} \nonumber \\
		&\leq&\!\!\! \sqrt{\sum_{i,j,k,l=1}^n \big(\FC^4_{ijkl}\big)^2 \sum_{i,j,k,l=1}^n \bx_i^2 \bx_j^2 \by_k^2 \by_l^2} \nonumber\\
		&=&\!\!\! \norm{\FC^4}_2 \norm{\bx}^2_2 \norm{\by}^2_2,
\end{eqnarray*}
for every $\bx,\by\in\R^n$. 
It follows that for $\alpha \geq 3\norm{\FC^4}_2$, inequality \eqref{eq:hessian_S_alpha} is true for any $\bx,\by \in \R^n$. 
Finally, the second statement of the proposition follows from the fact that $\norm{\bx}_2^2=n_1$ for any $\bx \in M$.
\end{proof}

One of the key aspects of the algorithms we derive in the next section is that we optimize the multilinear form $\MF^4$ instead of the function $S^4$.
This requires that we are able to extend the modification $S^4_\alpha$ resp. $\norm{\bx}^4_2$ to a multilinear form.
\begin{proposition}
    The symmetric tensor $\GC^4\in\R^{n\times n \times n\times n}$ with corresponding symmetric multilinear form defined as
    \begin{equation*}
    	\MG^4(\bx,\by,\bz,\bt) = \displaystyle \frac{\<\bx,\by\>\<\bz,\bt\>+ \<\bx,\bz\>\<\by,\bt\> + \<\bx,\bt\>\<\by,\bz\>}{3}
        \end{equation*}
    satisfies $\MG^4(\bx,\bx,\bx,\bx) = \norm{\bx}_2^4$. 
    \label{prop:tensor_of_quad_term}
\end{proposition}
\begin{proof}
	The proof requires two parts:
	\begin{enumerate}
	\item There exists a fourth order symmetric tensor $\GC^4$ such that its multilinear form is given as above.
	\item The multilinear form associated to $\GC^4$ satisfies $\MG^4(\bx,\bx,\bx,\bx) = \norm{\bx}_2^4$.
	\end{enumerate}
		Let $\GC^4\in\R^{n\times n \times n \times n}$ be defined as follow:
		\begin{equation*}
		\GC^4_{ijkl}=\begin{cases} 1 & \text{ if } i = j = k = l,\\
		1/3 & \text{ if } i = j \neq k = l,\\
		1/3 & \text{ if } i = k \neq j = l,\\
		1/3 & \text{ if } i = l \neq j = k,\\
		0 & \text{ else},
		\end{cases} \qquad \forall 1\leq i,j,k,l\leq n.
		\end{equation*}
		The multilinear form associated to $\GC^4$ is then computed as
		    	    \begin{eqnarray}
		    	    	    	&& \hspace{-6mm} \MG^4(\bx,\by,\bz,\bt)=\sum_{i,j,k,l=1}^n \GC^4_{ijkl} \bx_i \by_j \bz_k \bt_l  \nonumber\\
		    	    	    		&&\hspace{-4mm}= \sum_{i=1}^n \bx_i \by_i \bz_i \bt_i + 
		    	    	    		\frac{1}{3}\sum^n_{\substack{i,j=1 \\ i \neq j}} \bx_i \by_i \bz_j \bt_j  \nonumber\\
		    	    	    		&&+ \frac{1}{3}\sum^n_{\substack{i,j=1 \\ i \neq j}} \bx_i \by_j \bz_i \bt_j+
		    	    	    		\frac{1}{3}\sum^n_{\substack{i,j=1 \\ i \neq j}} \bx_i \by_j \bz_j \bt_i  \nonumber\\
		    	    	    		&&\hspace{-4mm}= \frac{1}{3}\Big(\<\bx,\by\>\<\bz,\bt\> + \<\bx,\bz\>\<\by,\bt\> + \<\bx,\bt\>\<\by,\bz\>\Big) \nonumber
		    	    \end{eqnarray}
		As a result, we have $\MG^4(\bx,\bx,\bx,\bx) = \norm{\bx}_2^4$.
\end{proof}

Thus in the algorithm we optimize finally 
\[ \MF^4_\alpha(\bx,\by,\bz,\bt)=\MF^4(\bx,\by,\bz,\bt) + \alpha \,\MG^4(\bx,\by,\bz,\bt).\]
\textbf{Discussion}: 
In \cite{Joa2003} they proposed a general convexification strategy for
arbitrary score functions where, similar to our modification, the added term is constant on the
set of assignment matrices $M$. However, as the added term is inhomogeneous,
it cannot be extended to a symmetric multilinear form and thus does not work
in our framework. 
 In second order graph matching also several methods use convexified score functions
in various ways \cite{ZasEtl2009, ZhoTor2012, ZhoTor2013}. However, for none of these methods
it is obvious how to extend it to a third order approach for hypergraph matching. 
%

\section{Tensor Block Coordinate Ascent for Hypergraph Matching}\label{sec:algorithms}
The key idea of both algorithms is that instead of directly optimizing the score function $S^4$ we optimize
the corresponding multilinear form. Lemma \ref{lem:main} allows us then to connect the latter problem to the 
problem we are actually interested in. In both variants we directly optimize over the discrete set $M$ of possible 
matchings, that is, there is no relaxation involved. Moreover, we show monotonic ascent for both methods. In most cases we achieve higher scores
than all other higher order methods, which is reflected in significantly better performance in the experiments
in Section \ref{sec:experiments}.

The original graph matching problem is to maximize the score function $S^3$ over all assignment matrices in $M$.
This combinatorial optimization problem is NP-hard.
As stated above, instead of working with the original score function $S^3$, we use the lifted
tensor and thus the lifted score function $S^4$. 
As we cannot guarantee that $S^4$ is convex, we optimize finally $S^4_\alpha$ resp.
the associated multilinear form $\MF^4_\alpha$. However, we emphasize that even though we work with 
the lifted objects, both algorithms prove monotonic ascent with respect to the original 
score function $S^3$ over the set $M$. The central element of both algorithms
is Lemma \ref{lem:main} and the idea to optimize the multilinear form, instead of the score function directly, combined with
the fact that for assignment matrices both modifications (lifting and convexification) are constant and thus do not change the problem. 
In order to simplify the notation in the algorithms we discard the superscript and write $\MF_\alpha$ instead of
$\MF^4_\alpha$ as the algorithms only involve fourth order tensors.
\subsection{Tensor Block Coordinate Ascent via Linear Assignment Problems}\label{subsec:mult4}
The first algorithm uses a block coordinate ascent scheme to optimize the multilinear form $\MF_\alpha$.
In particular, we solve the following optimization problem 
\begin{equation*}
    \max\limits_{\bx,\by,\bz,\bt \in M} \MF^4_\alpha(\bx,\by,\bz,\bt).
\end{equation*}
Fixing all but one argument and maximizing that over all assignment matrices leads to a linear
assignment problem which can be solved globally optimal by the Hungarian algorithm \cite{Kuh1955,BurAmiMar2012}
- this is used for steps 1)-4) in Algorithm \ref{algo:linear}. However, optimization of the multilinear form
is not necessarily equivalent to optimizing the score function. 
This is the reason why we use Lemma \ref{lem:main} in step 5) to come back to the original problem. 
The following theorem summarizes the properties of Algorithm \ref{algo:linear}.
\begin{algorithm}
    \caption{BCAGM} 
    \KwIn {Lifted affinity tensor $\FC^4$, $\alpha = 3\norm{\FC^4}_2$,\\
           $(\bx^0,\by^0,\bz^0,\bt^0) \in M\times M\times M\times M,\, k = 0,\, m=0$}
    \KwOut {$\bx^* \in M$}
	  \textbf{Repeat} 
	    \begin{enumerate}
		\item $\tbx^{k+1} = \arg\max_{\bx \in M} \MF_\alpha(\bx,\by^{k},\bz^{k},\bt^{k})$ 
		\item $\tby^{k+1} = \arg\max_{\by \in M} \MF_\alpha(\tbx^{k+1},\by,\bz^{k},\bt^{k})$ 
		\item $\tbz^{k+1} = \arg\max_{\bz \in M} \MF_\alpha(\tbx^{k+1},\tby^{k+1},\bz,\bt^{k})$ 
		\item $\tbt^{k+1} = \arg\max_{\bt \in M} \MF_\alpha(\tbx^{k+1},\tby^{k+1},\tbz^{k+1},\bt)$ 
		\item \textbf{if} $\MF_\alpha (\tbx^{k+1},\tby^{k+1},\tbz^{k+1},\tbt^{k+1}) = 
					\MF_\alpha (\bx^{k},\by^{k},\bz^{k},\bt^{k})$ \textbf{then }
		    \begin{itemize}
			\item[--] $\bu^{m+1} = \argmax_{\bu \in \{ \tbx^{k+1},\tby^{k+1},\tbz^{k+1},\tbt^{k+1} \}} 
					  \MF_\alpha(\bu,\bu,\bu,\bu)$
		  \item[--]\textbf{if} $\MF_\alpha(\bu^{m+1},\bu^{m+1},\bu^{m+1},\bu^{m+1})=\MF_\alpha (\tbx^{k+1},\tby^{k+1},\tbz^{k+1},\tbt^{k+1})$ 
		        \textbf{then return}
			\item[--] $\bx^{k+1} = \by^{k+1} = \bz^{k+1} = \bt^{k+1} = \bu^{m+1}$
			\item[--] $m = m + 1$
		  \end{itemize}
		  \textbf{else} $\bx^{k+1} = \tbx^{k+1}$, $\by^{k+1} = \tby^{k+1}$, $\bz^{k+1}=\tbz^{k+1}$, $\bt^{k+1} = \tbt^{k+1}$
		\item[] \textbf{end}
		\item $k = k + 1$
	    \end{enumerate}
     \label{algo:linear}
     \vspace{-2mm}
\end{algorithm}
\begin{theorem}\label{th:Alg1}
    Let $\FC^4$ be a fourth order symmetric tensor.
    Then the following holds for Algorithm \ref{algo:linear}:
    \begin{enumerate}
     \itemsep0em
    \item The sequence $\MF_\alpha (\bx^{k},\by^{k},\bz^{k},\bt^{k})$ for $k=1,2,\ldots$ is strictly monotonically increasing or terminates.
	\item The sequence of scores $S^4(\bu^m)$ for $m=1,2,\ldots$ is strictly monotonically increasing or terminates. For every $m$, $\bu^m \in M$ is a valid assignment matrix.
	  \item The sequence of original third order scores $S^3(\bu^m)$ for $m=1,2,\ldots$ is strictly monotonically increasing or terminates.
	  \item The algorithm terminates after a finite number of iterations.
    \end{enumerate}
    \label{theo:linear}
\end{theorem}
   \begin{proof} 
      From the definition of steps $1)-4)$ in Algorithm \ref{algo:linear}, we get
   	    \begin{eqnarray*}
   		    \MF_{\alpha,k} \!\!\! &:= &\!\!\! \MF_\alpha(\bx^k,\by^{k},\bz^{k},\bt^{k}) \nonumber \\
   		    &\leq &\!\!\!  \MF_\alpha(\tbx^{k+1},\by^{k},\bz^{k},\bt^{k}) \nonumber \\
   		    &\leq&\!\!\!  \MF_\alpha(\tbx^{k+1},\tby^{k+1},\tbz^{k},\bt^{k}) \\
   		    &\leq&\!\!\!  \MF_\alpha(\tbx^{k+1},\tby^{k+1},\tbz^{k+1},\bt^{k}) \nonumber\\
   		    &\leq&\!\!\!  \MF_\alpha(\tbx^{k+1},\tby^{k+1},\tbz^{k+1},\tbt^{k+1})=:\ \tilde{\MF}_{\alpha,k+1}.\nonumber
   		\end{eqnarray*}
   		Either $\tilde{\MF}_{\alpha,k+1}>\MF_{\alpha,k}$ in which case
   		$$\bx^{k+1} = \tbx^{k+1},\by^{k+1} = \tby^{k+1}, \bz^{k+1}=\tbz^{k+1},\bt^{k+1} = \tbt^{k+1}$$
   		and
   		\begin{equation*} \MF_\alpha(\bx^{k+1},\by^{k+1},\bz^{k+1},\bt^{k+1}) > \MF_\alpha(\bx^{k},\by^{k},\bz^{k},\bt^{k}),\end{equation*}
   		or $\tilde{\MF}_{\alpha,k+1}=\MF_{\alpha,k}$ and the algorithm enters
   		step 5. Since, by Proposition \ref{prop:convexify}, $S_\alpha^4$ is convex for the chosen value of $\alpha$, applying Lemma \ref{lem:main} we get
   		\begin{eqnarray*}
   			      \tilde{\MF}_{\alpha,k+1} \!\!\! &\leq &\!\!\! \max\limits_{\bv\in\{\tilde\bx^{k+1},\tilde\by^{k+1},\tilde\bz^{k+1},\tilde\bt^{k+1}\}}\MF_\alpha(\bv,\bv,\bv,\bv)\nonumber\\
   				     &=&\!\!\! \MF_\alpha(\bu^{m+1},\bu^{m+1},\bu^{m+1},\bu^{m+1}) \\ &=&\!\!\!  S^4_\alpha(\bu^{m+1}).\nonumber
		\end{eqnarray*}
   	  If the inequality is an equality, then the termination condition of the algorithm is met. 
   	  Otherwise, we get
   	  \begin{eqnarray*}
	      \tilde{\MF}_{\alpha,k+1}\!\!\! &<& \!\!\! \MF_\alpha(\bu^{m+1},\bu^{m+1},\bu^{m+1},\bu^{m+1})\\
	      &=& \!\!\! S^4_\alpha(\bu^{m+1})=\MF_\alpha(\bx^{k+1},\by^{k+1},\bz^{k+1},\bt^{k+1}).\nonumber
   	  \end{eqnarray*}
   	  This proves the first statement of the theorem.
   
   	  From
   	  \begin{eqnarray*}
	      S^4_\alpha(\bu^{m+1})\!\!\! &=&\!\!\! \MF_\alpha(\bu^{m+1},\bu^{m+1},\bu^{m+1},\bu^{m+1})\nonumber\\ &=&\!\!\! \MF_\alpha(\bx^{k+1},\by^{k+1},\bz^{k+1},\bt^{k+1})
   	  \end{eqnarray*}
   	  it follows that 
   	  $S^4_\alpha(\bu^{m}),m=1,2,\ldots$ is a subsequence of $\MF_\alpha (\bx^{k},\by^{k},\bz^{k},\bt^{k}),k=1,2,\ldots$ and thus it holds either $S^4_\alpha(\bu^{m})=S^4_\alpha(\bu^{m+1})$ in which case the algorithm terminates or $S^4_\alpha(\bu^{m})<S^4_\alpha(\bu^{m+1})$.
   	  However, by Equation \eqref{eq:Modconstant}, the additional term which has been added to $S^4$ to get a convex function is constant on $M$, that is
   	  \[ S^4_\alpha(\bx)=S^4(\bx) + \alpha n_1^2\qquad \forall \bx\in M.\]
   	  It follows that either $S^4(\bu^{m})=S^4(\bu^{m+1})$ and the algorithm terminates or $S^4(\bu^{m})<S^4(\bu^{m+1})$ which proves the second part
   	  of the theorem.
   	  
   	  By Equation \eqref{eq:eqS4S3}, we have $S^4(\bx)=4\,n_1\,S^{3}(\bx)$ for any $\bx \in M$. 
   	  Thus the statements made for $S^4$ directly translate into corresponding statements for the original third order score $S^3$. 
   	  This proves the penultimate statement.  
   	   
   	   Finally, the algorithm yields a strictly monotonically increasing sequence $S^4(\bu^m),m=1,2,\ldots$ or it terminates. 
   	   Since there is only a finite number of possible assignment matrices, the sequence has to terminate after a finite number of steps.
\end{proof}
We would like to note that all statements of Theorem \ref{th:Alg1} remain valid for $\alpha=0$
if $S^4$ is convex. This is also the motivation why we run the algorithm first with $\alpha=0$
until we get no further ascent  and only then we set $\alpha=3\|\FC\|_2$. It turns out that in the experiments often the phase of the algorithm
with $\alpha=0$ leads automatically to a homogeneous solution and no further improvement is achieved
when setting $\alpha=3\|\FC\|_2$. This suggests that the constructed score functions in the experiments are already convex
or at least close to being convex.
\subsection{Tensor Block Coordinate Ascent via a Sequence of Quadratic Assignment Problems}\label{subsec:mult4rr}
The second algorithm uses a block coordinate ascent scheme to optimize the multilinear form $\MF_\alpha$
where now two arguments are fixed and one maximizes over the other two. 
The resulting problem is equivalent to a quadratic assignment problem (QAP) which is known to be NP-hard. 
Thus a globally optimal
solution as for the linear assignment problem in Algorithm \ref{algo:linear} is out of reach.
Instead we require a sub-algorithm $\Psi$ which delivers monotonic ascent for the QAP
\begin{equation}\label{eq:quadratic}
\vspace{-1mm} \max\limits_{\bx \in M} \inner{\bx,A\bx}, \vspace{-1mm}
\end{equation}
where $A \in \R^{n\times n}$ is a nonnegative symmetric matrix. 
As in Algorithm \ref{algo:linear}, we go back to the optimization of the score function in step 3)
using Lemma \ref{lem:main}. The following theorem summarizes the properties of Algorithm \ref{algo:quadratic}.

\FloatBarrier
\begin{algorithm}
    \caption{BCAGM-$\Psi$}
    \KwIn {Lifted affinity tensor $\FC^4$, $\alpha = 3\norm{\FC^4}_2$,\\
    $(\bx^0,\by^0) \in M\times M,\, k = 0,\, m=0$,\\ $\bz=\Psi(A,\bx^k)$ is an algorithm for the QAP in \eqref{eq:quadratic}
    which provides monotonic ascent, that is $\inner{\bz,A\bz}\geq \inner{\bx^k,A\bx^k}$}
    \KwOut {$\bx^* \in M$}
    \textbf{Repeat} 
	    \begin{enumerate}
		\item $\tbx^{k+1} = \Psi\big(\MF_\alpha(\cdot,\cdot,\by^{k},\by^{k}),\,\bx^k\big)$ 
		\item $\tby^{k+1} = \Psi\big(\MF_\alpha(\tbx^{k+1},\tbx^{k+1},\cdot,\cdot),\by^k\big)$
		\item \textbf{if} $\MF_\alpha (\tbx^{k+1},\tbx^{k+1},\tby^{k+1},\tby^{k+1}) = 
					\MF_\alpha (\bx^{k},\bx^{k},\by^{k},\by^{k})$ \textbf{then }
		    \begin{itemize}
		    	\item[--] $\bu^{m+1} = \displaystyle\argmax_{\bu \in \{ \tbx^{k+1},\tby^{k+1} \}} 
					  \MF_\alpha(\bu,\bu,\bu,\bu)$
			\item[--] \textbf{if} $\MF_\alpha (\tbx^{k+1},\tbx^{k+1},\tby^{k+1},\tby^{k+1})=\MF_\alpha(\bu^{m+1},\bu^{m+1},\bu^{m+1},\bu^{m+1})$ 
				  \textbf{then return}
			\item[--] $\bx^{k+1} = \by^{k+1} = \bu^{m+1}$
			\item[--] $m=m+1$
		    \end{itemize}
	     \textbf{else} $\bx^{k+1}=\tbx^{k+1}$, $\by^{k+1}=\tby^{k+1}$.\\
	      \textbf{end}
		\item $k = k + 1$
	    \end{enumerate}
    \label{algo:quadratic} 
     \vspace{-2mm}
\end{algorithm}
\begin{theorem}\label{th:Alg2}
    Let $\FC^4$ be a fourth order symmetric tensor
    and let $\Psi$ be an algorithm for the QAP which yields monotonic ascent.
    Then the following holds for Algorithm \ref{algo:quadratic}:
    \begin{enumerate}
         \itemsep0em
        \item The sequence $\MF_\alpha (\bx^{k},\bx^{k},\by^{k},\by^{k})$ for $k=1,2,\ldots$ is strictly monotonically increasing or terminates.
    	\item The sequence of scores $S^4(\bu^m)$ for $m=1,2,\ldots$ is strictly monotonically increasing or terminates. For every $m$, $\bu^m \in M$ is a valid assignment matrix.
    	  \item The sequence of original third order scores $S^3(\bu^m)$ for $m=1,2,\ldots$ is strictly monotonically increasing or terminates.
    	  \item The algorithm terminates after a finite number of iterations.
        \end{enumerate}
    \label{theo:quadratic}
   
\end{theorem}
   \begin{proof} 
From the definition of steps $1)-2)$ in Algorithm \ref{algo:quadratic}, we get
	    \begin{eqnarray*}
		    \MF_{\alpha,k} \!\!\! &:= &\!\!\! \MF_\alpha(\bx^k,\bx^{k},\by^{k},\by^{k})\nonumber\\
		    &\leq &\!\!\!  \MF_\alpha(\tbx^{k+1},\tbx^{k+1},\by^{k},\by^{k}) \\
		    &\leq&\!\!\!  \MF_\alpha(\tbx^{k+1},\tbx^{k+1},\tby^{k+1},\tby^{k+1})=:\ \tilde{\MF}_{\alpha,k+1}. \nonumber
		\end{eqnarray*}
		Either $\tilde{\MF}_{\alpha,k+1}>\MF_{\alpha,k}$ in which case
		$$\bx^{k+1} = \tbx^{k+1},\qquad\quad\by^{k+1} = \tby^{k+1}$$
		and
		\[ \MF_\alpha(\bx^{k+1},\bx^{k+1},\by^{k+1},\by^{k+1}) > \MF_\alpha(\bx^{k},\bx^{k},\by^{k},\by^{k}),\]
		or $\tilde{\MF}_{\alpha,k+1}=\MF_{\alpha,k}$ and the algorithm enters
		step 3. Since, by Proposition \ref{prop:convexify}, $S_\alpha^4$ is convex for the chosen value of $\alpha$, applying Lemma \ref{lem:main} we get
		\begin{eqnarray*}
			\tilde{\MF}_{\alpha,k+1} \!\!\! &\leq &\!\!\! \max\limits_{\bv\in\{\tilde\bx^{k+1},\tilde\by^{k+1}\}}\MF_\alpha(\bv,\bv,\bv,\bv) \nonumber\\
				     &=&\!\!\! \MF_\alpha(\bu^{m+1},\bu^{m+1},\bu^{m+1},\bu^{m+1}) \\ &=&\!\!\!  S^4_\alpha(\bu^{m+1}). \nonumber
		\end{eqnarray*}
	  If the inequality is an equality, then the termination condition of the algorithm is met. 
	  Otherwise, we get
	  \begin{eqnarray*}
	      \tilde{\MF}_{\alpha,k+1}\!\!\! &<&\!\!\! \MF_\alpha(\bu^{m+1},\bu^{m+1},\bu^{m+1},\bu^{m+1})\\
	      \!\!\! &=&\!\!\!  S^4_\alpha(\bu^{m+1})=\MF_\alpha(\bx^{k+1},\bx^{k+1},\by^{k+1},\by^{k+1}).\nonumber
	  \end{eqnarray*}
	  This proves the first statement of the theorem.
	  
	  From
	  \begin{eqnarray*}
	      S^4_\alpha(\bu^{m+1})\!\!\! &=&\!\!\! \MF_\alpha(\bu^{m+1},\bu^{m+1},\bu^{m+1},\bu^{m+1})\!\!\! \nonumber\\ &=&\!\!\! \MF_\alpha(\bx^{k+1},\bx^{k+1},\by^{k+1},\by^{k+1})
	  \end{eqnarray*}
	  it follows that 
	  $S^4_\alpha(\bu^{m}),m=1,2,\ldots$ is a subsequence of $\MF_\alpha (\bx^{k},\bx^{k},\by^{k},\by^{k}),k=1,2,\ldots$ and thus it holds either $S^4_\alpha(\bu^{m})=S^4_\alpha(\bu^{m+1})$ in which case the algorithm terminates or $S^4_\alpha(\bu^{m})<S^4_\alpha(\bu^{m+1})$.
	  However, by Equation \eqref{eq:Modconstant}, the additional term which has been added to $S^4$ to get a convex function is constant on $M$, that is
	  \[ S^4_\alpha(\bx)=S^4(\bx) + \alpha n_1^2\qquad \forall \bx\in M.\]
	  It follows that either $S^4(\bu^{m})=S^4(\bu^{m+1})$ and the algorithm terminates or $S^4(\bu^{m})<S^4(\bu^{m+1})$ which proves the second part
	  of the theorem.
	  
	  By Equation \eqref{eq:eqS4S3}, we have $S^4(\bx)=4\,n_1\,S^{3}(\bx)$ for any $\bx \in M$. Thus the statements made for $S^4$ directly translate into corresponding statements for the original third order score $S^3$. 
	  This proves the penultimate statement.  
	   
	   Finally, the algorithm yields a strictly monotonically increasing sequence $S^4(\bu^m),m=1,2,\ldots$ or it terminates. 
	   Since there is only a finite number of possible assignment matrices, the sequence has to terminate after a finite number of steps.
   \end{proof}

In analogy to Theorem \ref{th:Alg1} all statements of Theorem \ref{th:Alg2} remain valid for $\alpha=0$ if $S^4$ is convex.
Thus we use the same initialization strategy with $\alpha=0$ as described above.
There are several methods available which we could use for the sub-routine $\Psi$
in the algorithm. We decided to use the recent max pooling algorithm \cite{ChoEtAl2014} and the IPFP algorithm \cite{LeoHebSuk2009}
and use the Hungarian algorithm to turn their output into a valid
assignment matrix in $M$. 
It turns out that the combination of our tensor block coordinate ascent scheme using
their algorithms as a sub-routine yields very good performance on all datasets.
\section{Experiments}\label{sec:experiments}
\setlength{\tabcolsep}{0.3pt}
\setlength{\belowcaptionskip}{-5pt}

We evaluate our hypergraph matching algorithms on standard synthetic benchmarks and realistic image datasets 
by comparing them with state-of-the-art third order and second order approaches.
In particular, we use the following third order approaches:
Tensor Matching (TM) \cite{DucEtAl2011}, Hypergraph Matching via Reweighted Random Walks (RRWHM) \cite{LeeChoLee2011}, 
Hypergraph Matching (HGM) \cite{ZasSha2008}. 
For second order approaches, 
Max Pooling Matching (MPM) ~\cite{ChoEtAl2014}, Reweighted Random Walks for Graph Matching (RRWM) \cite{ChoLeeLee2010}, 
Integer Projected Fixed Point (IPFP) \cite{LeoHebSuk2009}, and Spectral Matching (SM) \cite{LeoHeb2005} 
are used. We denote our tensor block coordinate ascent methods as BCAGM from Algorithm \ref{algo:linear} which 
uses the Hungarian algorithm as subroutine, and BCAGM+MP and BCAGM+IPFP from Algorithm \ref{algo:quadratic} 
which uses MPM \cite{ChoEtAl2014} and IPFP \cite{LeoHebSuk2009} respectively as subroutine.
MPM has recently outperformed other second order methods in the presence of a large number of outliers 
without deformation. Thus, it serves as a very competitive candidate with third order approaches in this setting.
For all the algorithms, we use the authors' original implementation.

In all experiments below, we use the all-ones vector as starting point for all the methods. Moreover,
we apply the Hungarian algorithm to the output of every algorithm to turn it into a proper matching.

\textbf{Generation of affinity tensor/matrix:}
To build the affinity tensor for the third order algorithms, we follow the approach of Duchenne \etal ~\cite{DucEtAl2011}: for each triple of points in the same image we compute a feature vector $f$ based on
the angles of the triangle formed by those three points. 
Let $f_{i_1,i_2,i_3}$ and $f_{j_1,j_2,j_3}$ denote two feature vectors for two triples $(i_1,i_2,i_3)$ and $(j_1,j_2,j_3)$
in two corresponding images.
Then we compute the third order affinity tensor as:
\begin{equation}
    \FC^3_{(i_1,j_1),(i_2,j_2),(i_3,j_3)} = \exp(-\gamma \norm{f_{i_1,i_2,i_3} - f_{j_1,j_2,j_3}}_2^2)
\end{equation} where $\gamma$ is the inverse of the mean of all squared distances.
As shown by Duchenne \etal ~\cite{DucEtAl2011} these affinities increase the geometric invariance compared to second order models 
making it more robust to transformations, like translations, rotations and scalings.
This partially explains why third order approaches are preferred to second order ones in this regard.
As the computation of the full third order affinity tensor is infeasible in practice, we adopt the sampling strategy 
proposed by \cite{DucEtAl2011,LeeChoLee2011}. 

The affinity matrix for all the second order methods are built by following \cite{ChoEtAl2014, ChoLeeLee2010},
where the Euclidean distance of two point pairs is used to compute the pairwise similarity:
\begin{equation}
    \FC^2_{(i_1,j_1),(i_2,j_2)} = \exp(- {| d_{i_1 i_2}^P - d_{j_1 j_2}^Q |}^2 / \sigma_s^2)
\end{equation} where $d_{i_1 i_2}^P$ is the Euclidean distance between two points $i_1,\,i_2\in P$, $d_{j_1 j_2}^Q$  is the distance between the points $j_1,j_2\in Q$ and $\sigma_s$ is a normalization parameter specified below.

\subsection{Synthetic Dataset}\label{subsec:exp_synthetic}
This section presents a comparative evaluation on matching point sets $P$ and $Q$ in $\R^2$. 
We follow the approach in \cite{ChoEtAl2014} for the creation of the point sets.
That is, for $P$ we sample $n_{in}$ inlier points from a Gaussian distribution $\Nc(0,1)$.
The points in $Q$ are generated by adding Gaussian deformation noise $\Nc(0,\sigma^2)$ to the point set $P$.
Furthermore, $n_{out}$ outlier points are added to $Q$ sampled from $\Nc(0,1)$. 
We test all matching algorithms under different changes in the data: outliers, deformation and scaling (i.e. we multiply the coordinates of the points in $Q$ by a constant factor).

In the outlier setting, we perform two test cases as follows.
In the first test case, we vary the number of outliers from $0$ to very challenging $200$, 
which is $20$-times more than the number of inliers $n_{in}$, while $\sigma$ is set to $0.01$ and no scaling is used. 
In the second case, we set $\sigma$ to $0.03$ and scale all the points in $Q$ by a factor of $1.5$, 
making the problem much more difficult.
This test shows that third order methods are very robust against scaling compared to second order methods.
In all the plots in this section, each quantitative result was obtained by averaging over $100$ trials. 
The accuracy is computed as the ratio between the number of correct matches and the number of inlier points.
We use $\sigma_s = 0.5$ in the affinity construction for the second order methods as suggested by \cite{ChoEtAl2014}.
The results in Figure \ref{fig:exp_out_def} show that our algorithms are robust w.r.t. outliers 
even if deformation and scaling are involved. 
When the deformation is negligible as shown in Figure \ref{fig:out_def_noscal}, MPM and BCAGM+MP perform best 
followed by BCAGM, BCAM+IPFP and RRWHM. 
In a more realistic setting as shown in Figure \ref{fig:out_def_scal}, where outlier, deformation and scaling are present in the data, 
our algorithms outperform all other methods. 
One can state that BCAGM+MP transfers the robustness of MPM to a third order method
and thus is able to deal with scaling which is difficult for second order methods.

In the deformation setting, the deformation noise $\sigma$ is varied from $0$ to challenging $0.4$ while 
the number of inliers $n_{in}$ is fixed to $20$. There are no outlier and scale changes in this setting.
This type of experiment has been used in previous works \cite{ChoEtAl2014, LeeChoLee2011}. 
where $\sigma$ has been varied from $0$ to $0.2$. 
The result in Figure \ref{fig:def} shows that our algorithms are quite competitive with
RRWHM which has been shown to be robust to deformation \cite{LeeChoLee2011}. 
It is interesting to note that while MPM is very robust against outliers (without scaling) 
it is outperformed if the deformation is significantly increased (without having outliers). 
However, note that even though our BCAGM+MP uses MPM as a subroutine, it is not affected by this slight weakness.
Considering all experiments our BCAGM+MP and to a slightly weaker extent BCAGM and BCAGM+IPFP outperform all other methods 
in terms of robustness to outliers, deformation and scaling. 
\textbf{Runtime:} 
Figure \ref{fig:exp_out_def} shows that the running time of our algorithms is competitive in comparison to other higher order
and even second order methods. In particular, BCAGM is always among the methods with lowest running time (more results can be found in the appendix).
\begin{figure*}
    \vspace{-10pt}
    \subfloat{\includegraphics[width=0.31\linewidth]{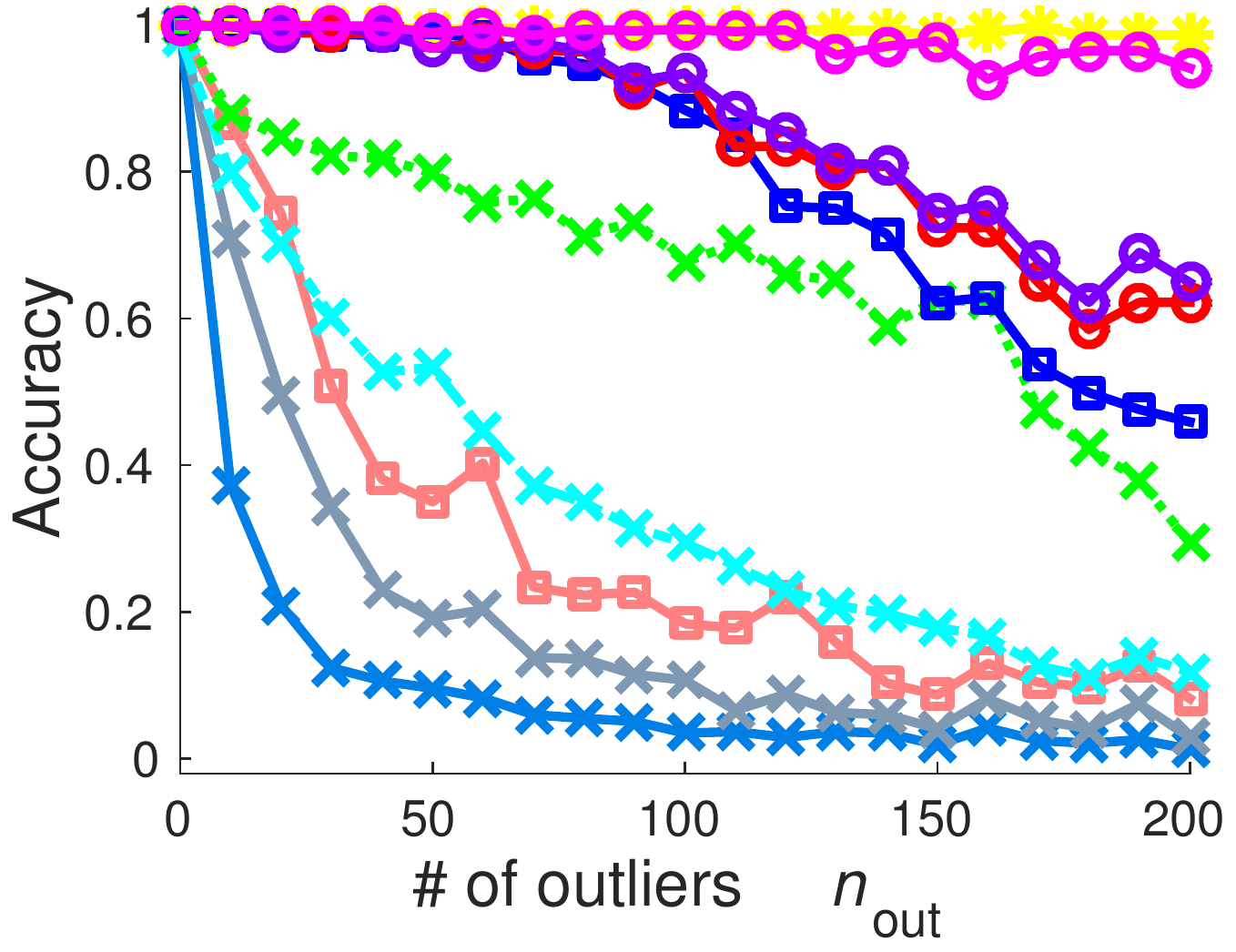} }
    \setcounter{subfigure}{0}
    \subfloat[$n_{in}=10,\sigma=0.01,scale=1$]{\includegraphics[width=0.31\linewidth]{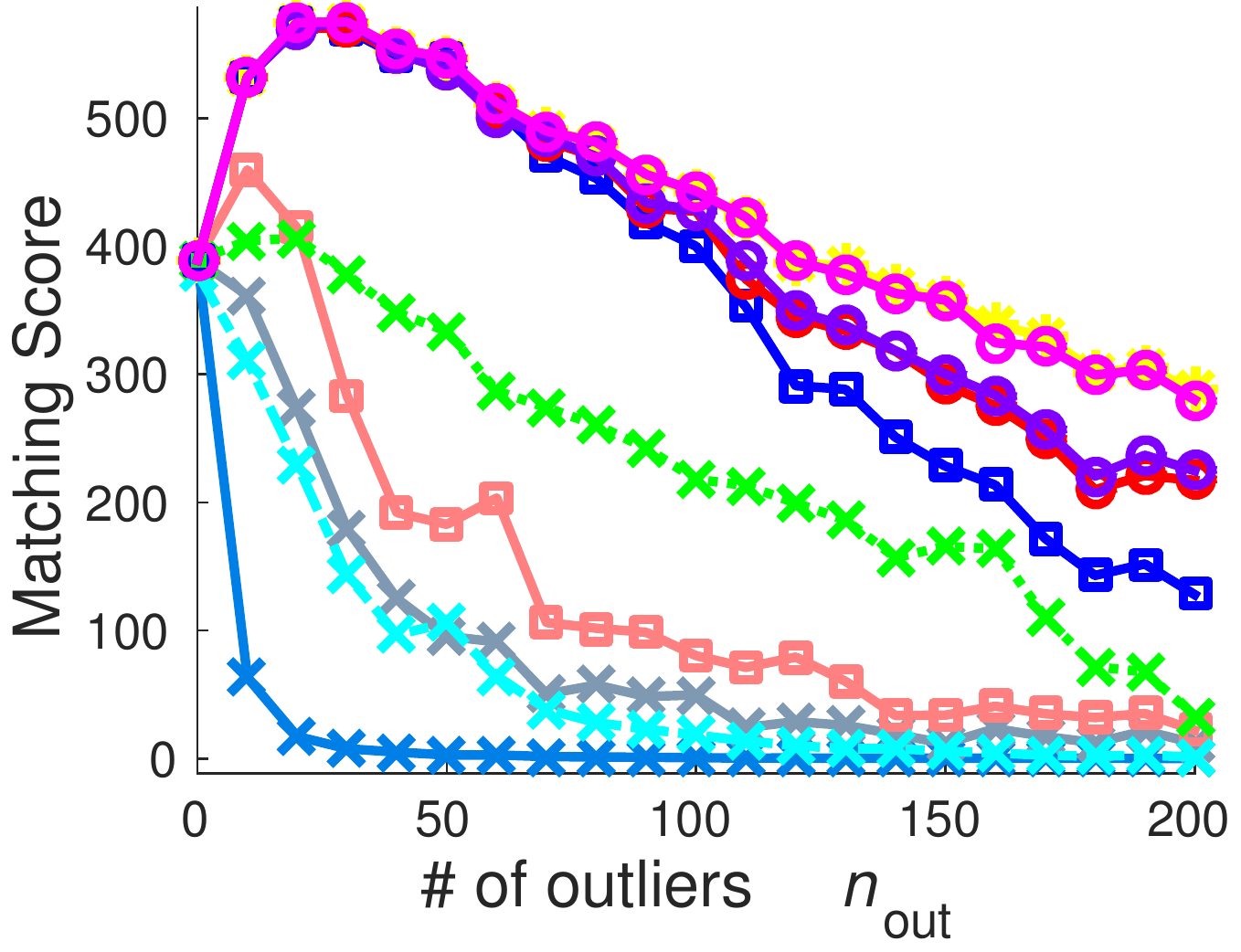} \label{fig:out_def_noscal} }
    \vspace{-5pt}
    \subfloat{\includegraphics[width=0.31\linewidth]{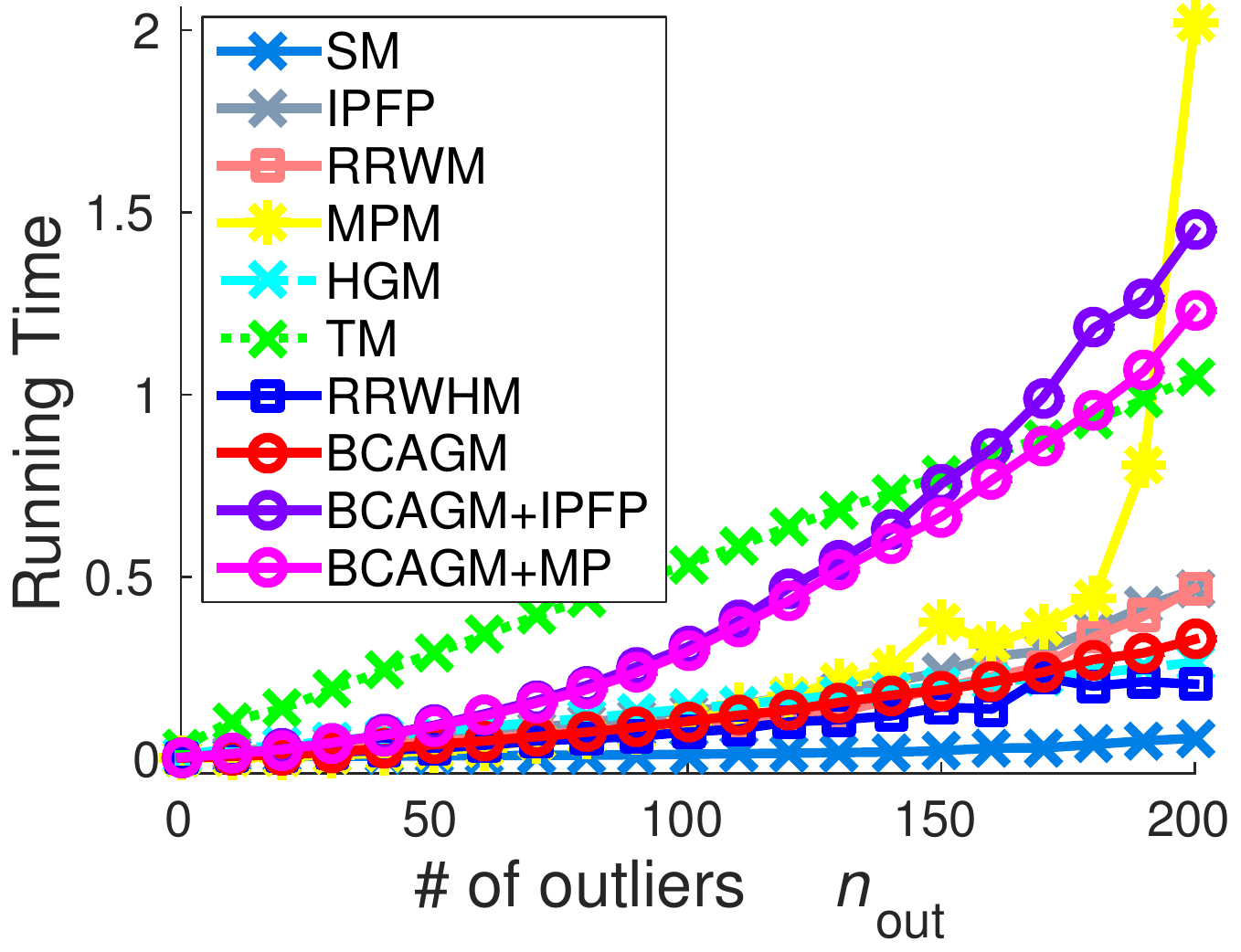} } \\
    \setcounter{subfigure}{0}
    \subfloat{\includegraphics[width=0.31\linewidth]{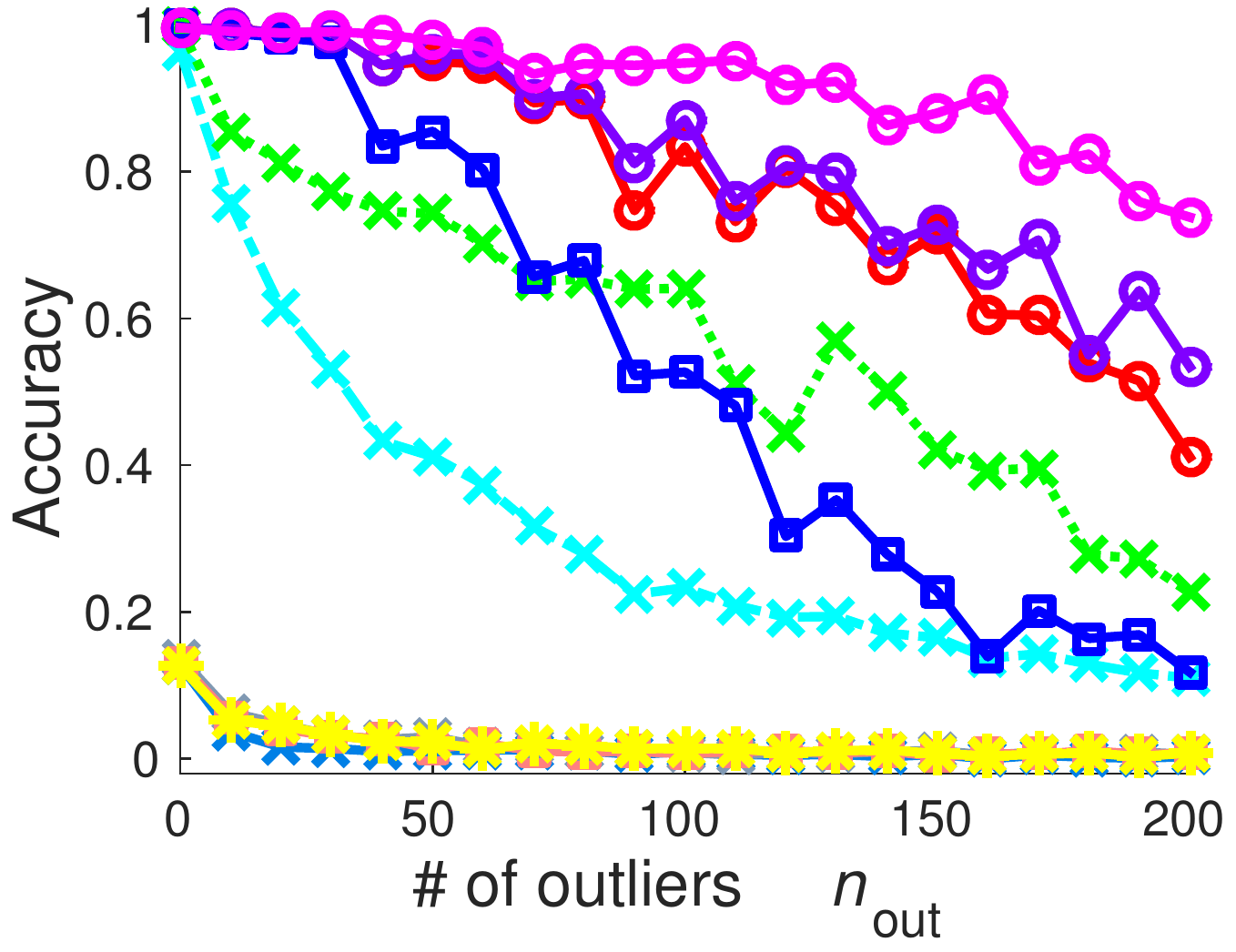} } 
    \subfloat[$n_{in}=10,\sigma=0.03,scale=1.5$]{\includegraphics[width=0.31\linewidth]{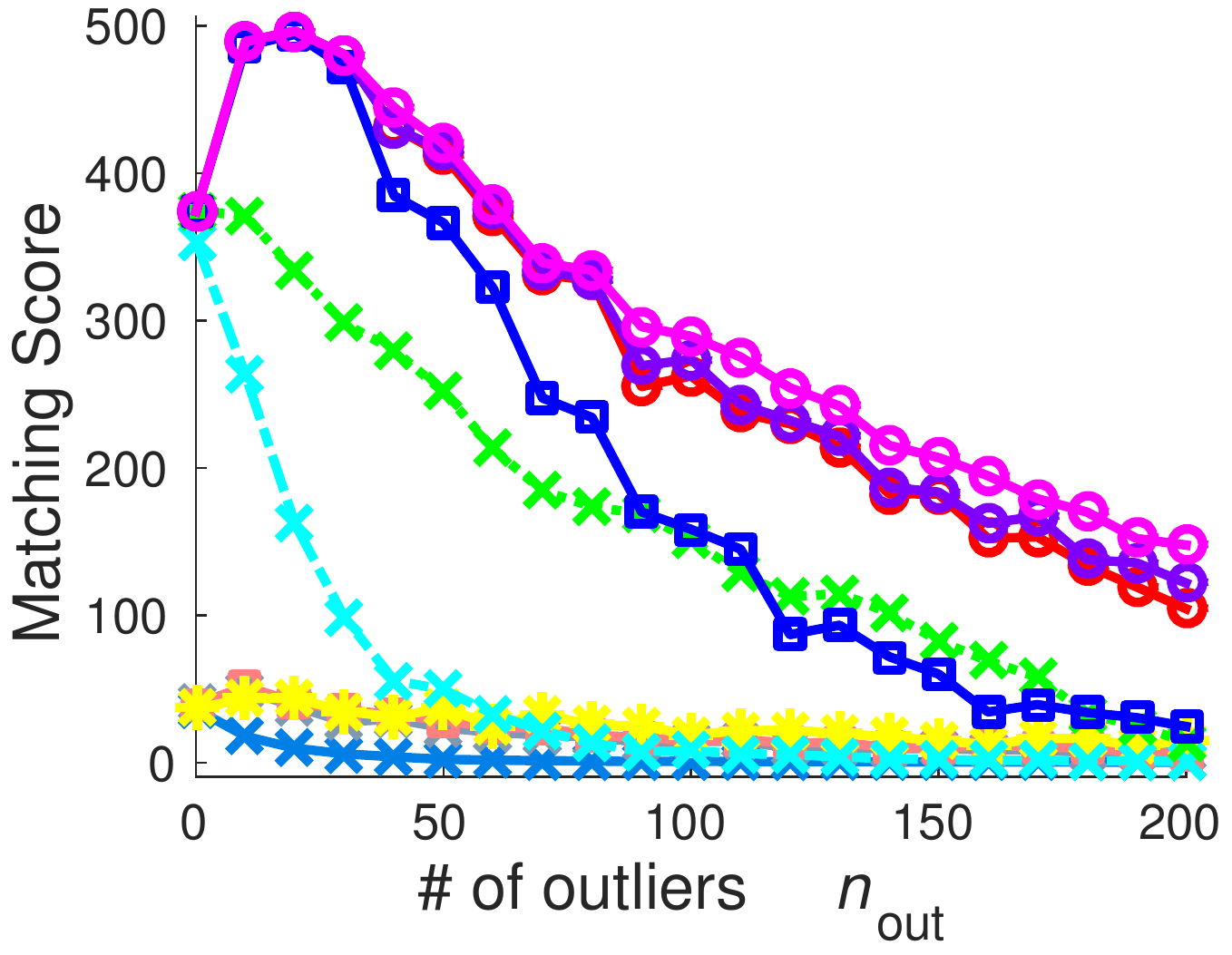} \label{fig:out_def_scal} } 
    \vspace{-5pt}
    \subfloat{\includegraphics[width=0.31\linewidth]{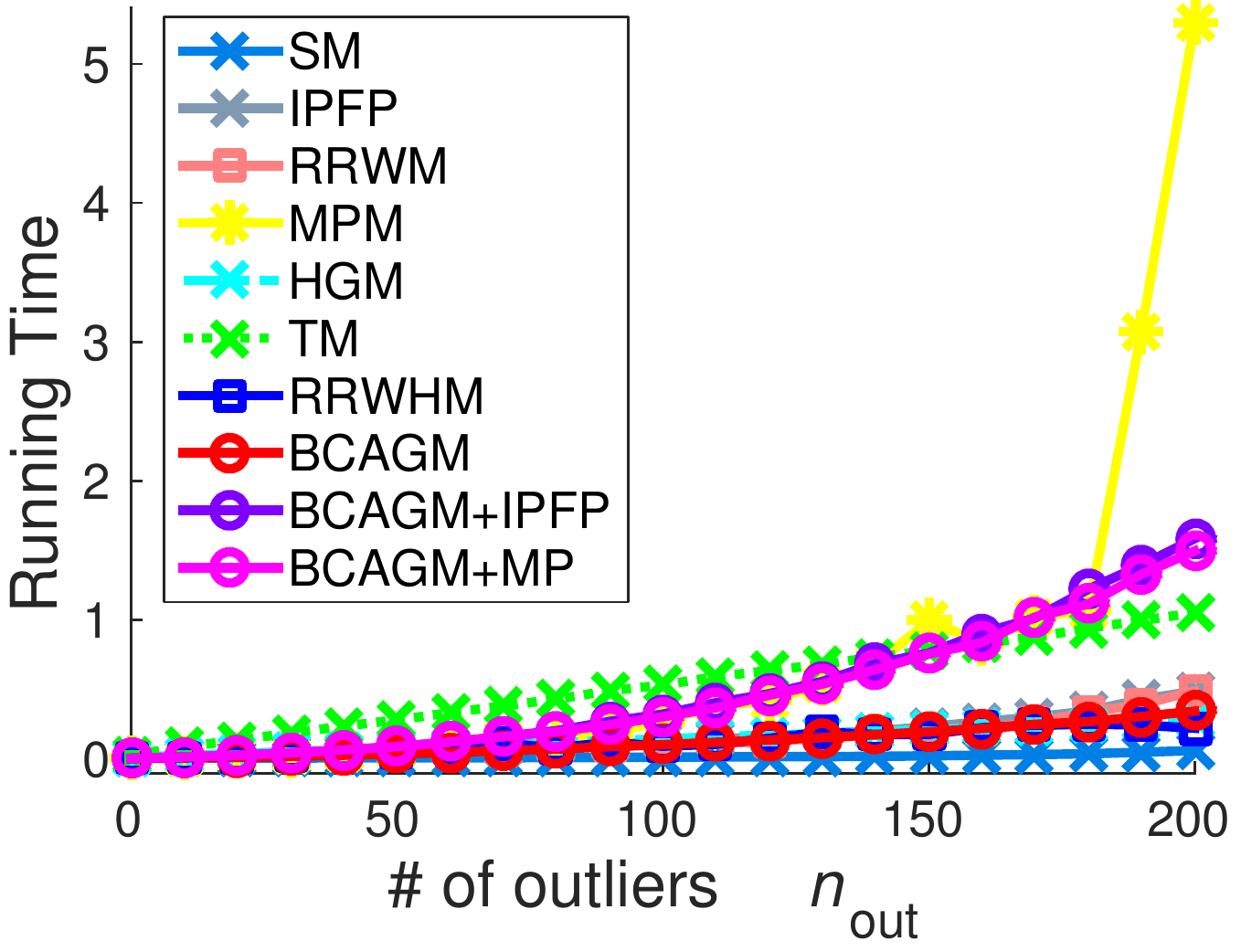} } 
\caption{
Matching point sets in $\R^2$:
each row shows the accuracy, matching score and running time of all algorithms. 
The number of outliers has been varied from $0$ to $200$. 
(a) Increasing number of outliers with slight deformation and no scaling.
(b) Increasing number of outliers with larger deformation and scaling.
(Best viewed in color.)
}
\label{fig:exp_out_def}
\end{figure*}
\begin{figure}
    \vspace{-10pt}
    \subfloat{ \includegraphics[width=0.48\linewidth]{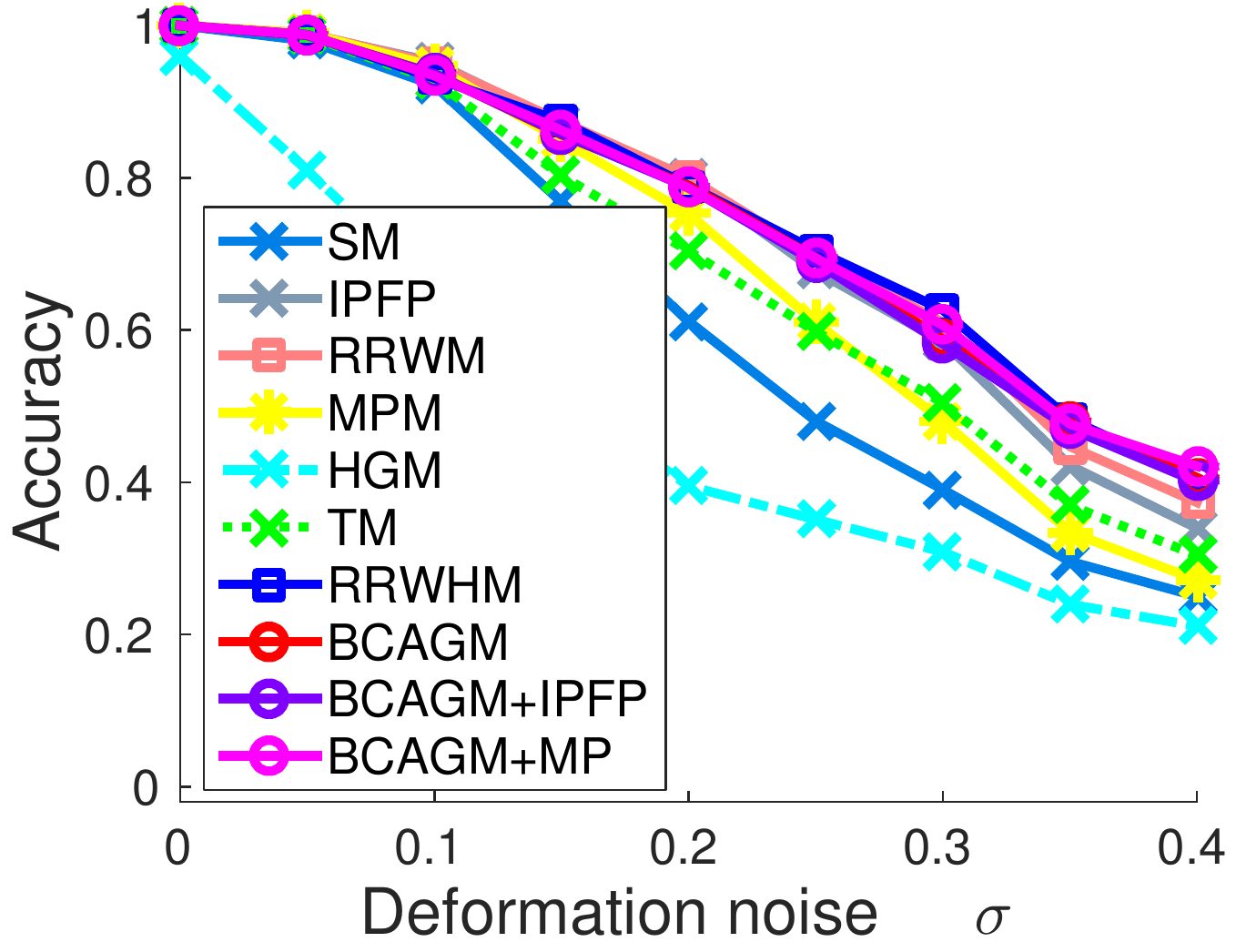} } 
    \subfloat{ \includegraphics[width=0.48\linewidth]{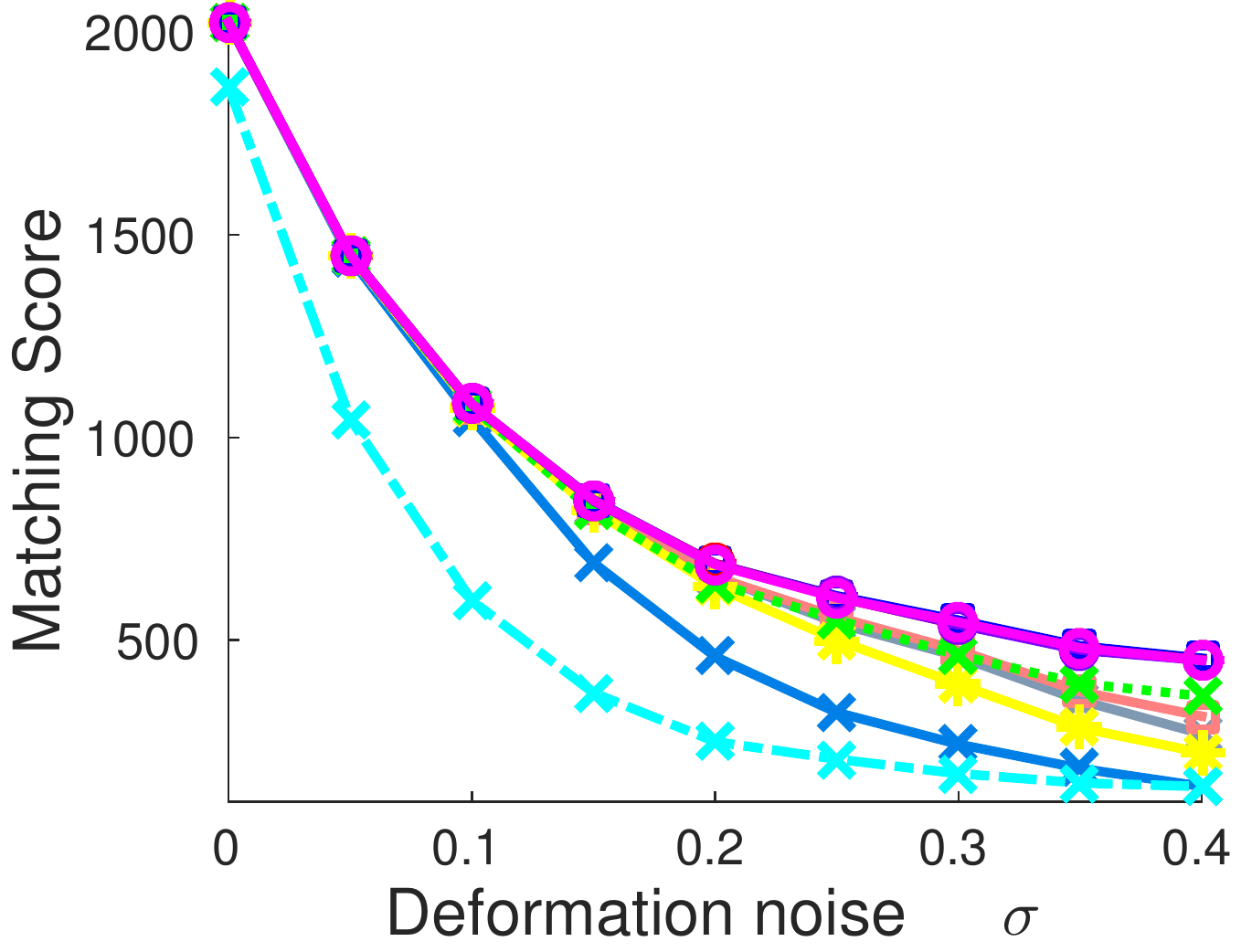} } 
\caption{
    Robustness of matching algorithms w.r.t. deformation. ($n_{in}=20,n_{out}=0,scale=1.$)
}
\vspace{-5pt}
\label{fig:def}
\end{figure}
\begin{figure*}
    \vspace{-10pt}
    \subfloat[An input pair: $10$ pts vs $30$ pts, $baseline = 50$]{ \includegraphics[width=0.32\linewidth]{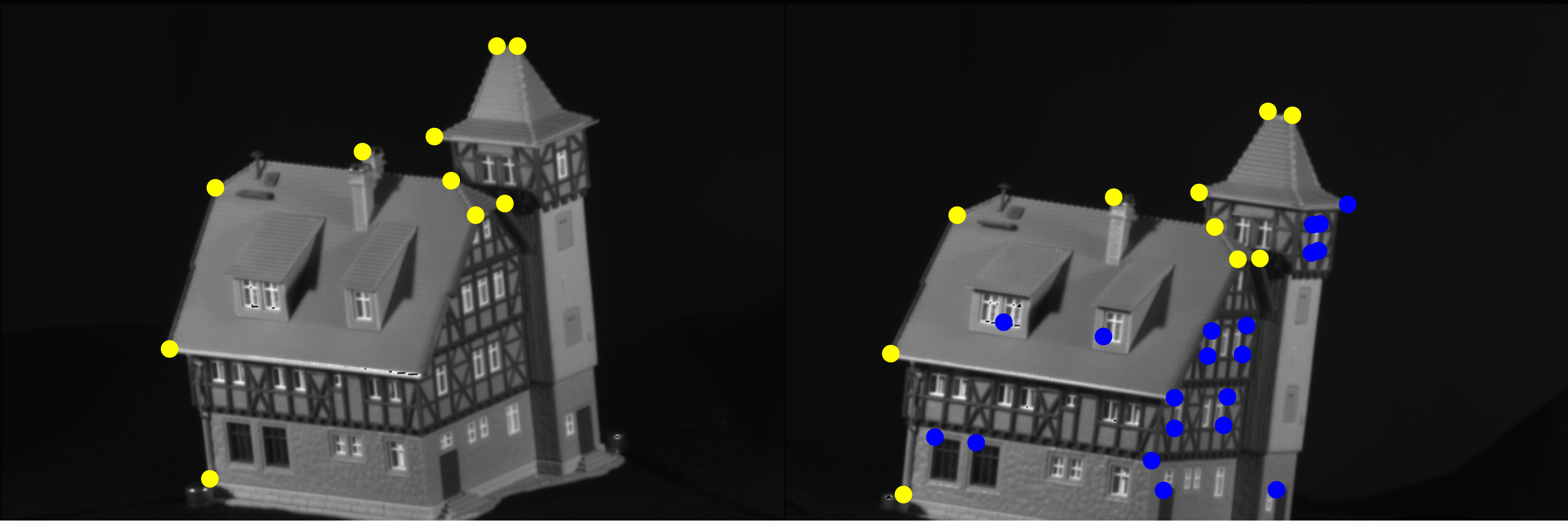} }
    \subfloat[MPM 4/10 (15.6)]{ \includegraphics[width=0.32\linewidth]{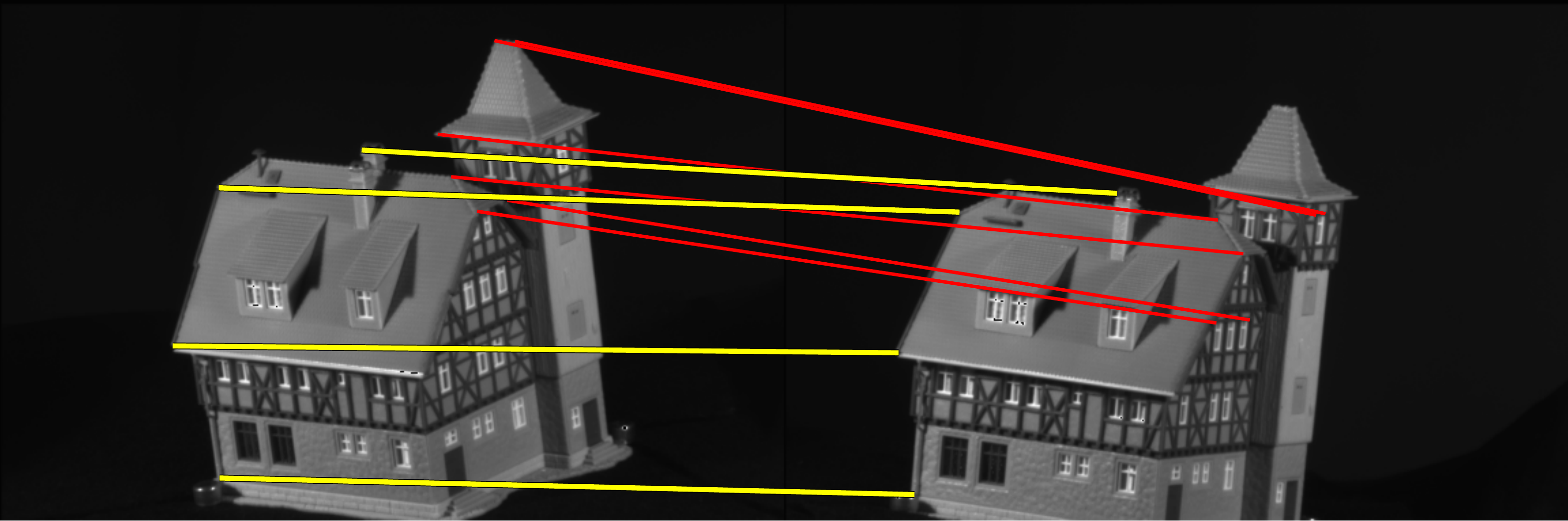} }
    \vspace{-5pt}
    \subfloat[TM 5/10 (18.4)]{ \includegraphics[width=0.32\linewidth]{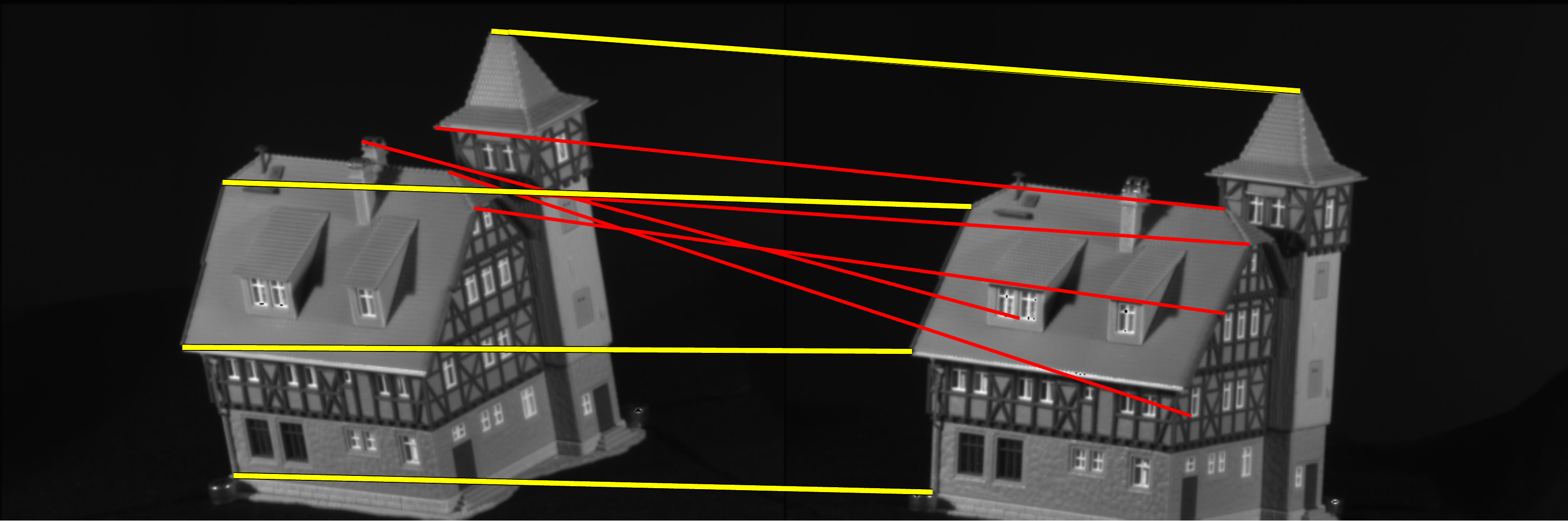} } \\
    \subfloat[RRWHM 7/10 (26.4)]{ \includegraphics[width=0.32\linewidth]{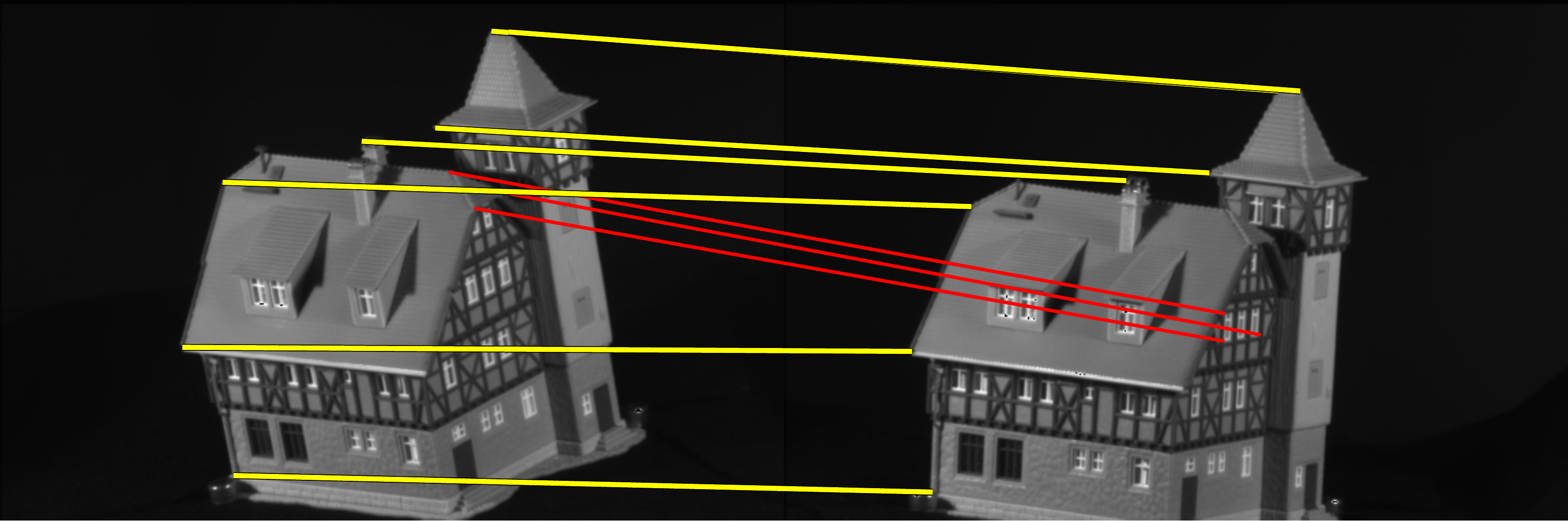} }
    \subfloat[BCAGM 10/10 (43.1)]{ \includegraphics[width=0.32\linewidth]{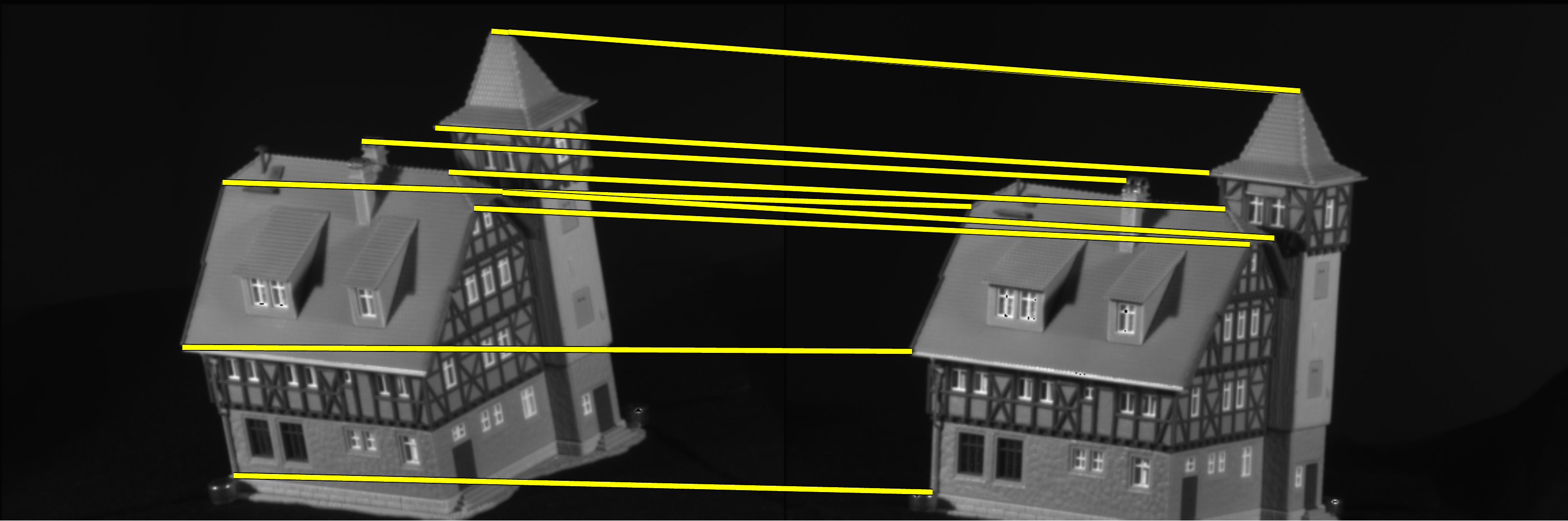} }
    \vspace{-5pt}
    \subfloat[BCAGM+MP 10/10 (43.1)]{ \includegraphics[width=0.32\linewidth]{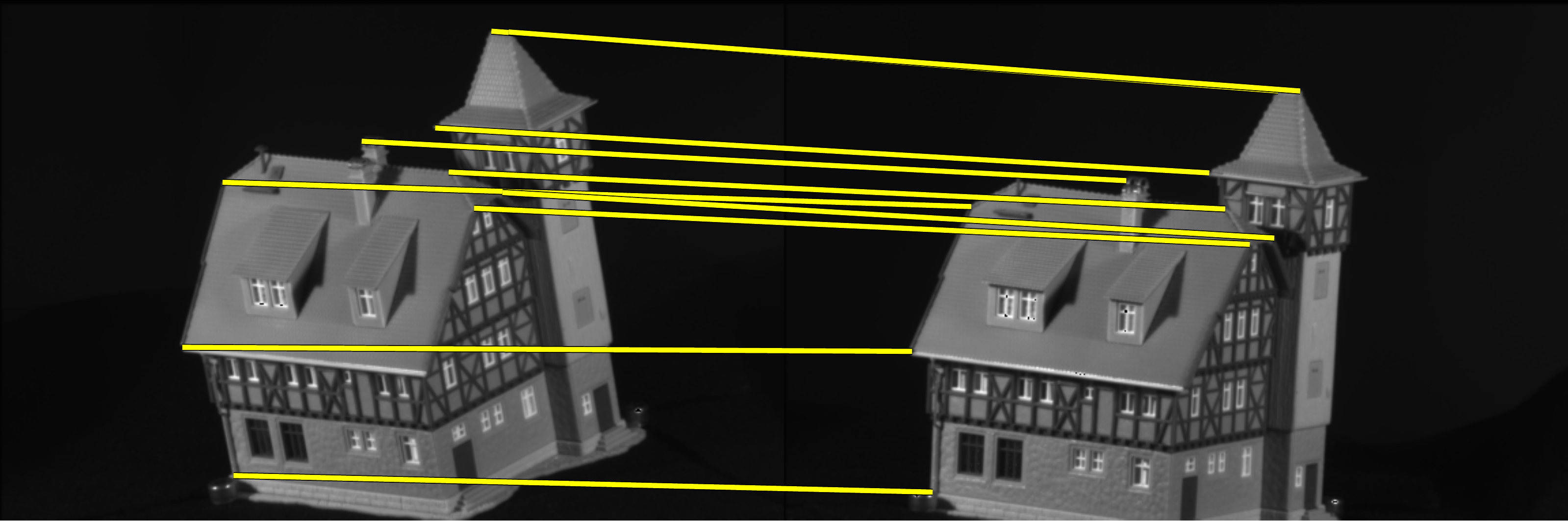} } \\
    \subfloat[10 pts vs 30 pts]{ \includegraphics[width=0.31\linewidth]{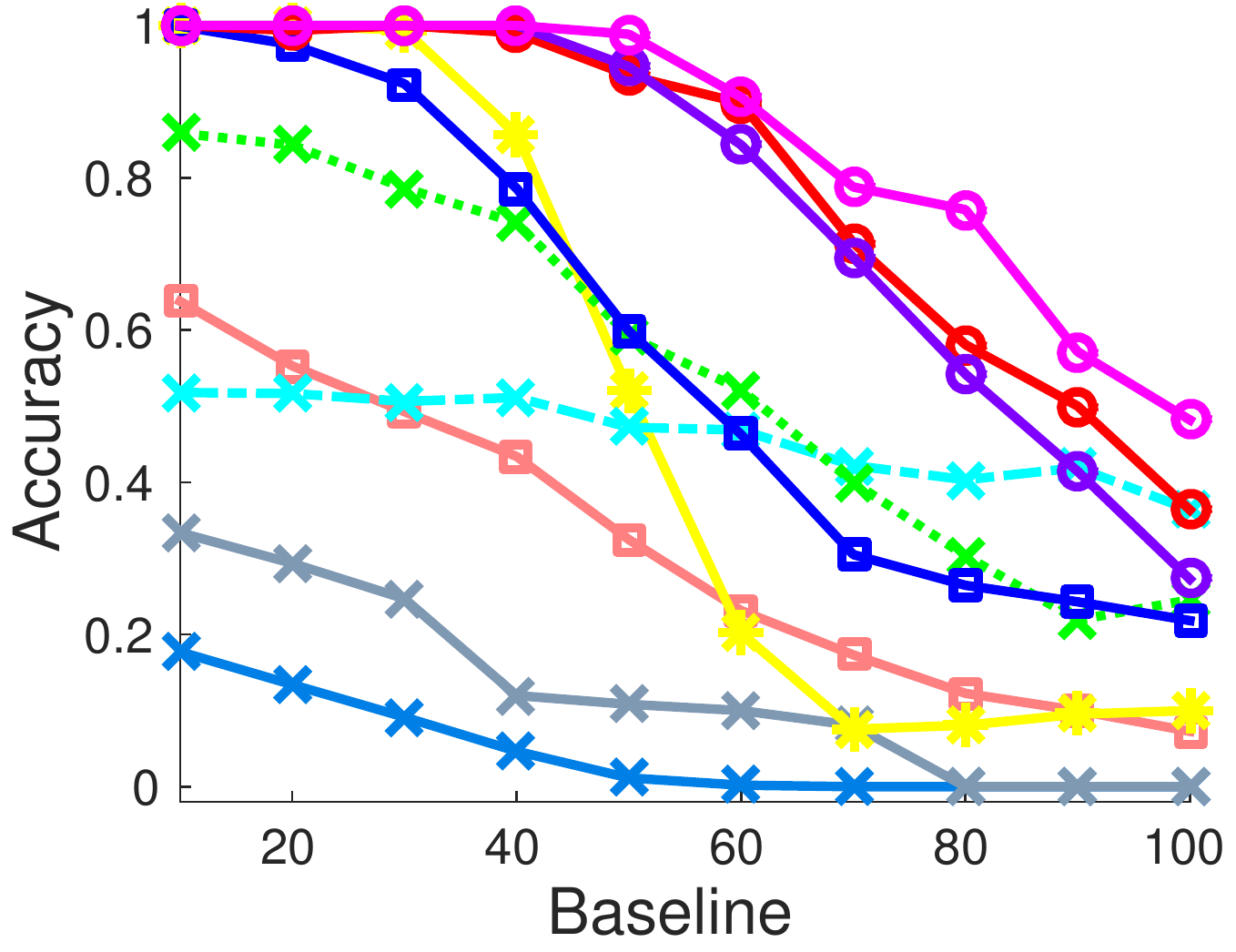} }  
    \subfloat[20 pts vs 30 pts]{ \includegraphics[width=0.31\linewidth]{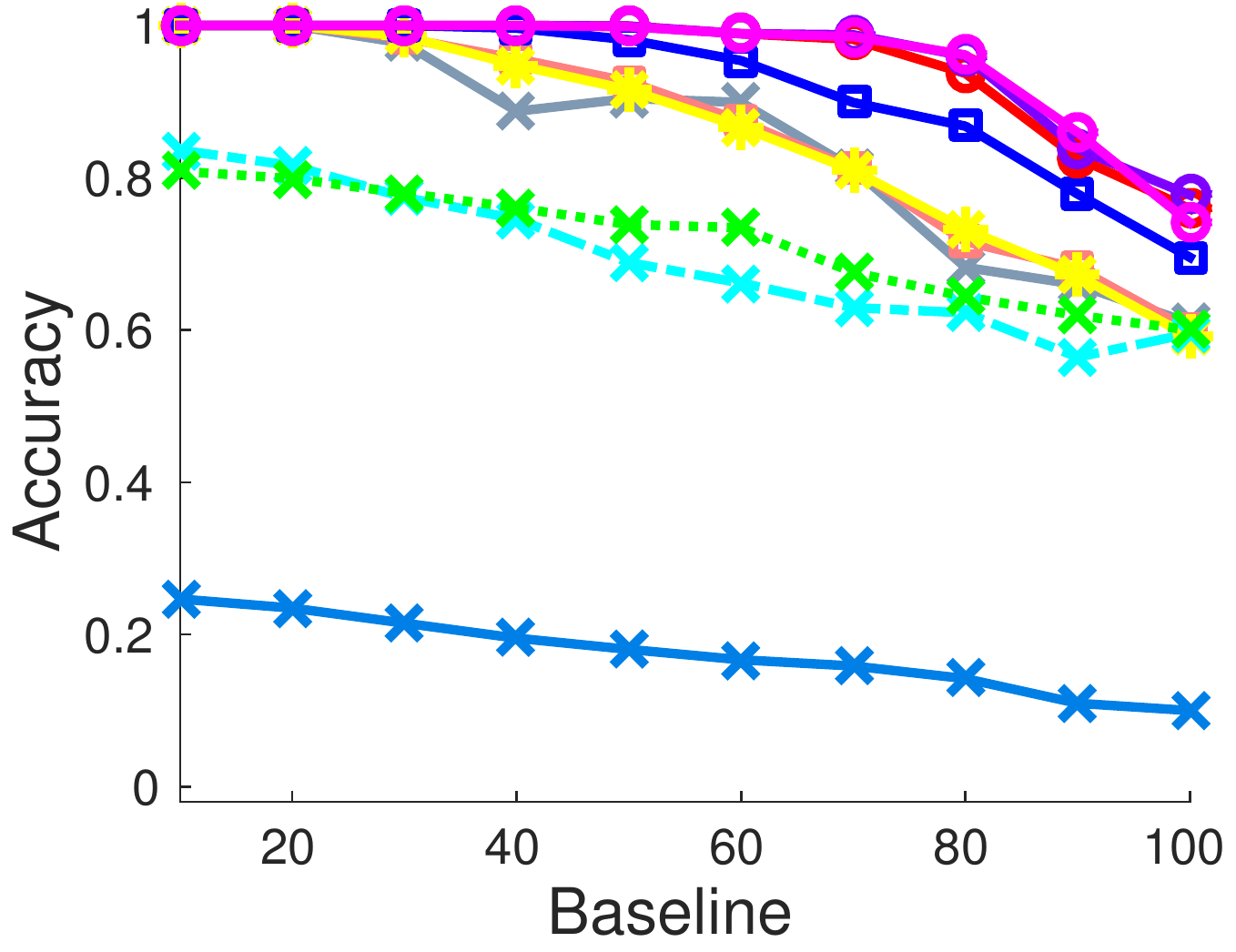} }  
    \vspace{-5pt}
    \subfloat[30 pts vs 30 pts]{ \includegraphics[width=0.31\linewidth]{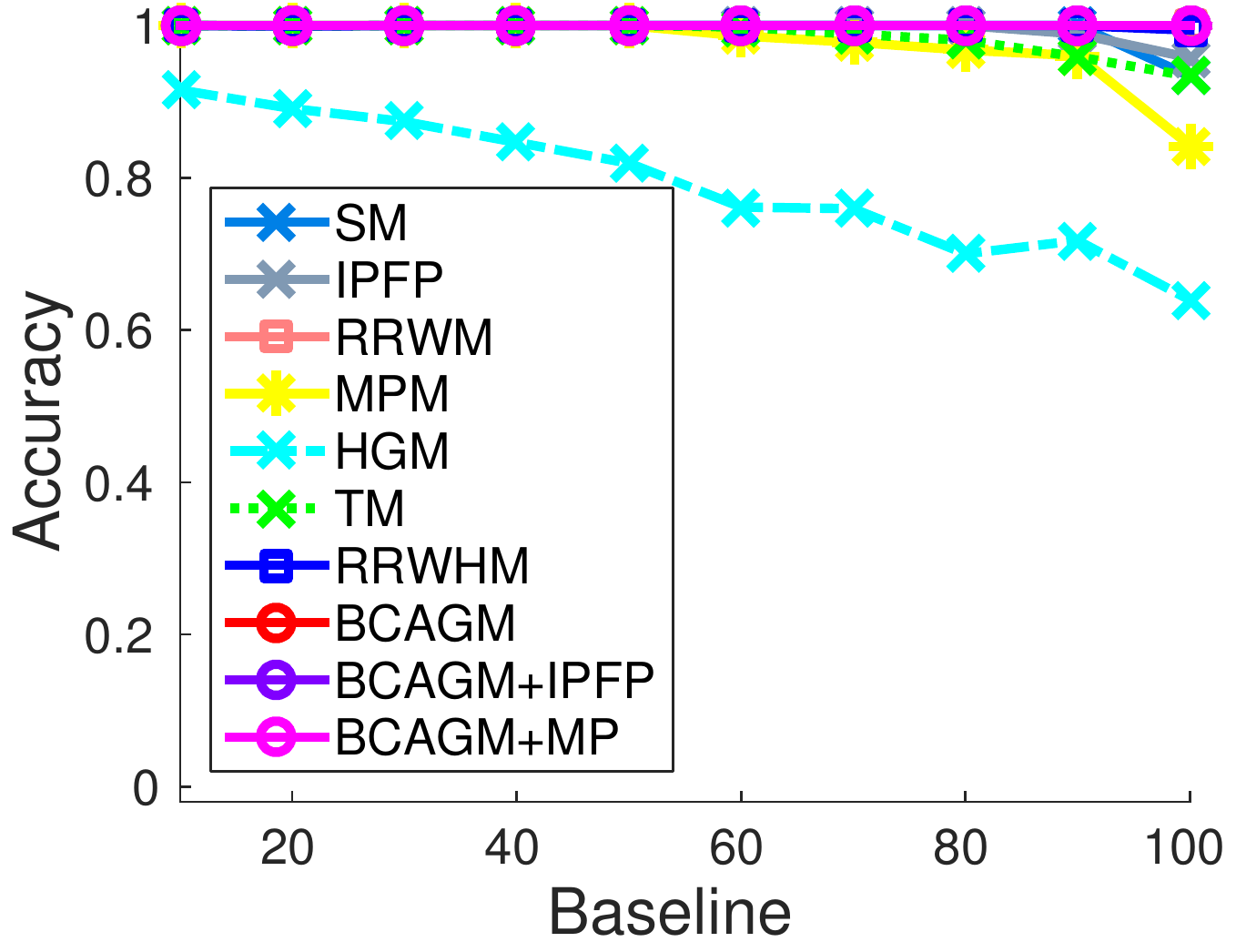} } 
\caption{
CMU house dataset:
the first two rows show the matching output of several algorithms on
an example pair of images where the baseline is $50$. 
The yellow/red lines indicate correct/incorrect matches.
The third row shows the average performance of algorithms with different number of points in the first image.
(Best viewed in color.)
}
\label{fig:exp_house}
\end{figure*}

\begin{figure}
    \vspace{-10pt}
    \subfloat[$34$ pts vs $44$ pts, $10$ outliers]{ \includegraphics[width=0.4749\linewidth]{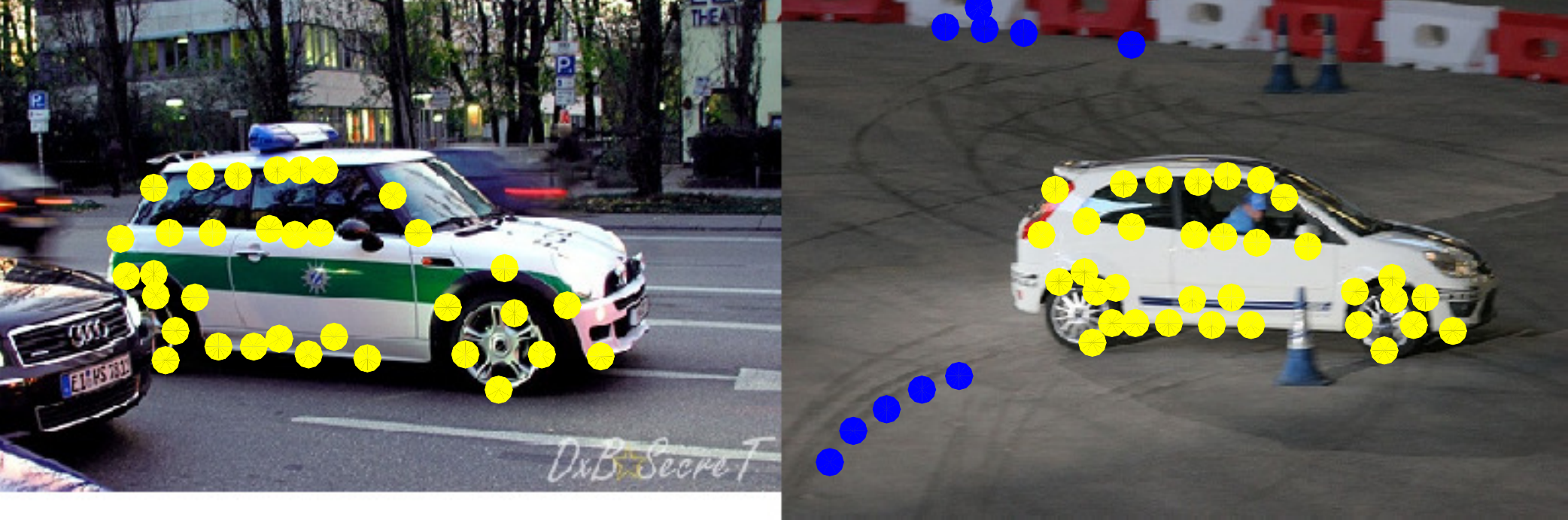} }
    \vspace{-5.2pt}
    \subfloat[TM 10/34 (1715.0)]{ \includegraphics[width=0.4749\linewidth]{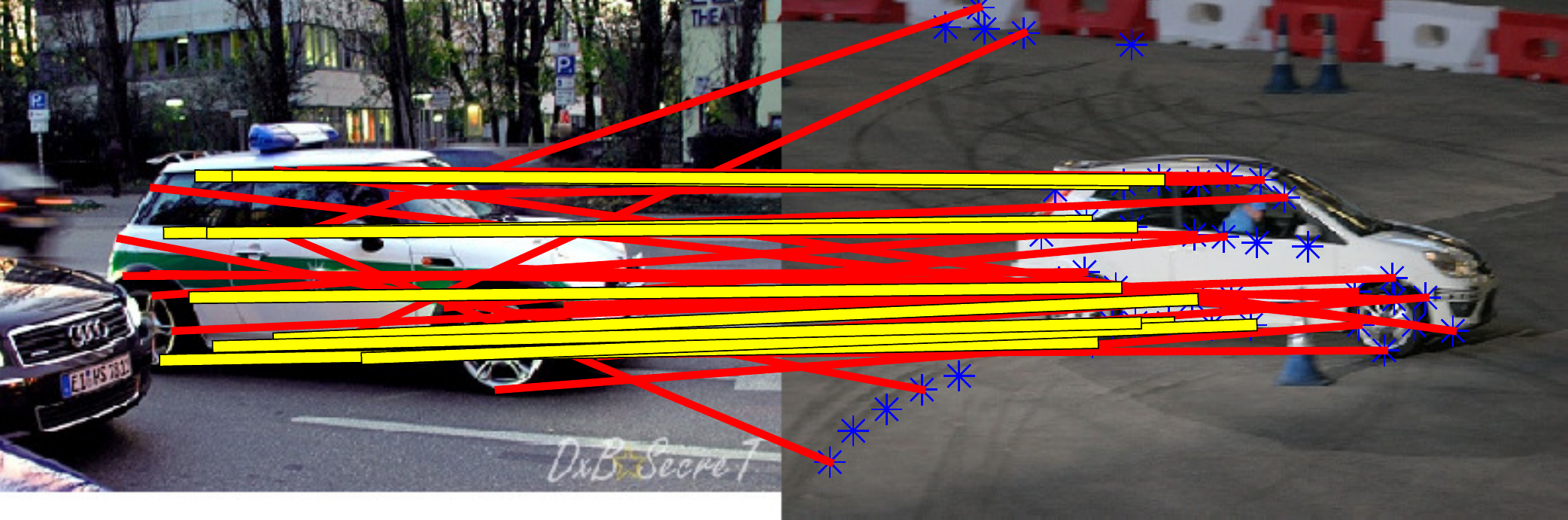} } \\
    \subfloat[HGM 9/34 (614.7)]{ \includegraphics[width=0.4749\linewidth]{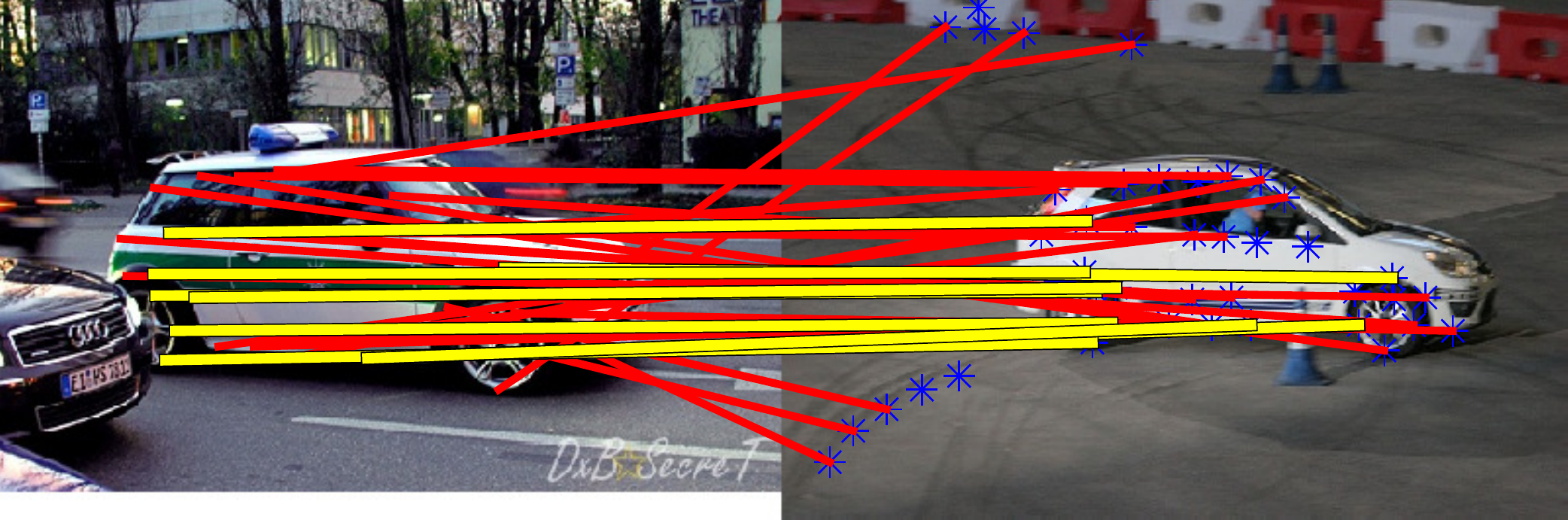} }  
    \vspace{-5.2pt}
    \subfloat[RRWHM 28/34 (5230.5)]{ \includegraphics[width=0.4749\linewidth]{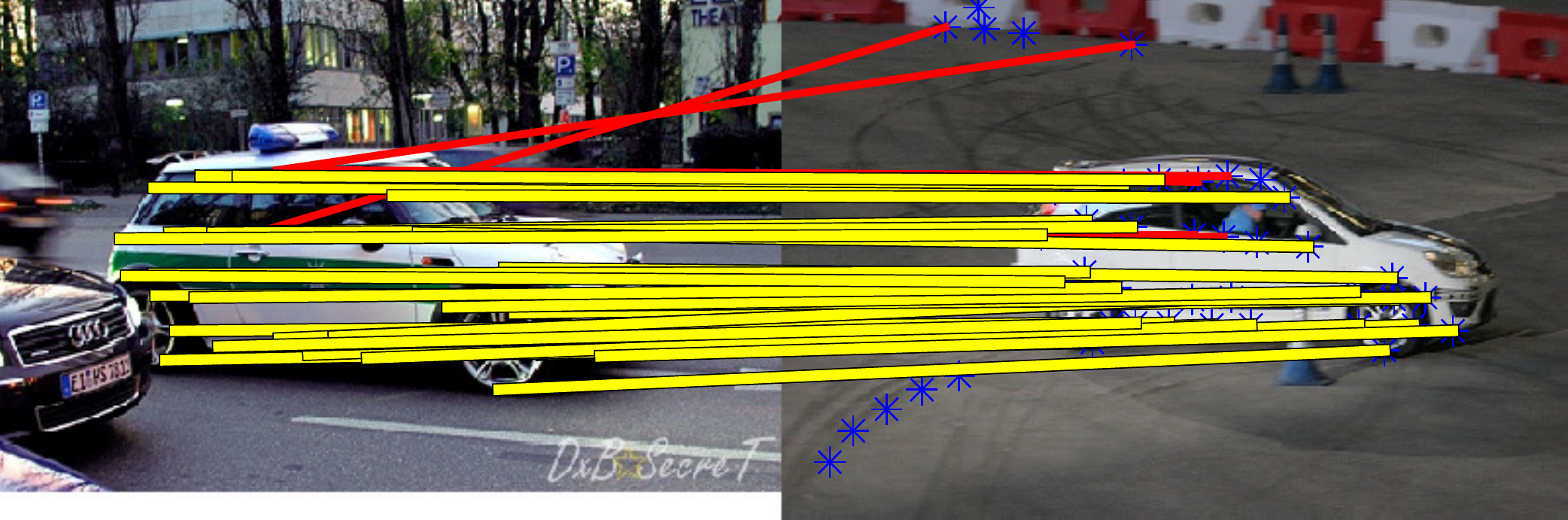} } \\
    \subfloat[BCAGM 28/34 (5298.8)]{ \includegraphics[width=0.4749\linewidth]{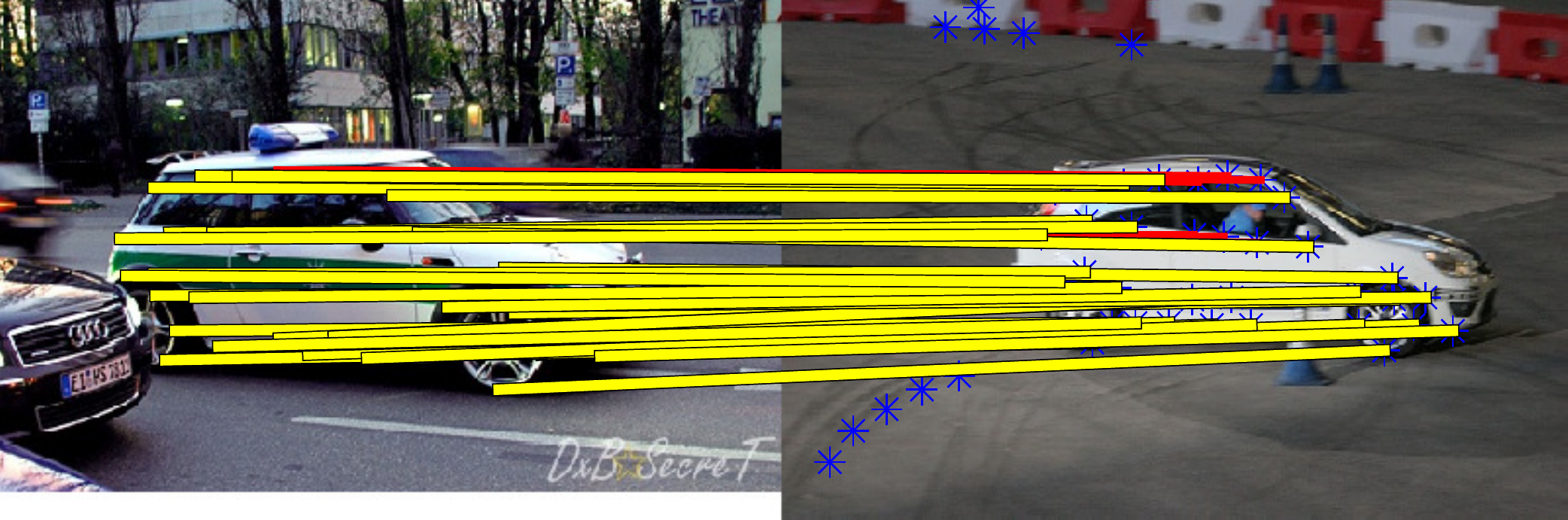} }
    \vspace{-5.2pt}
    \subfloat[BCAGM+MP 34/34 (5377.3)]{ \includegraphics[width=0.4749\linewidth]{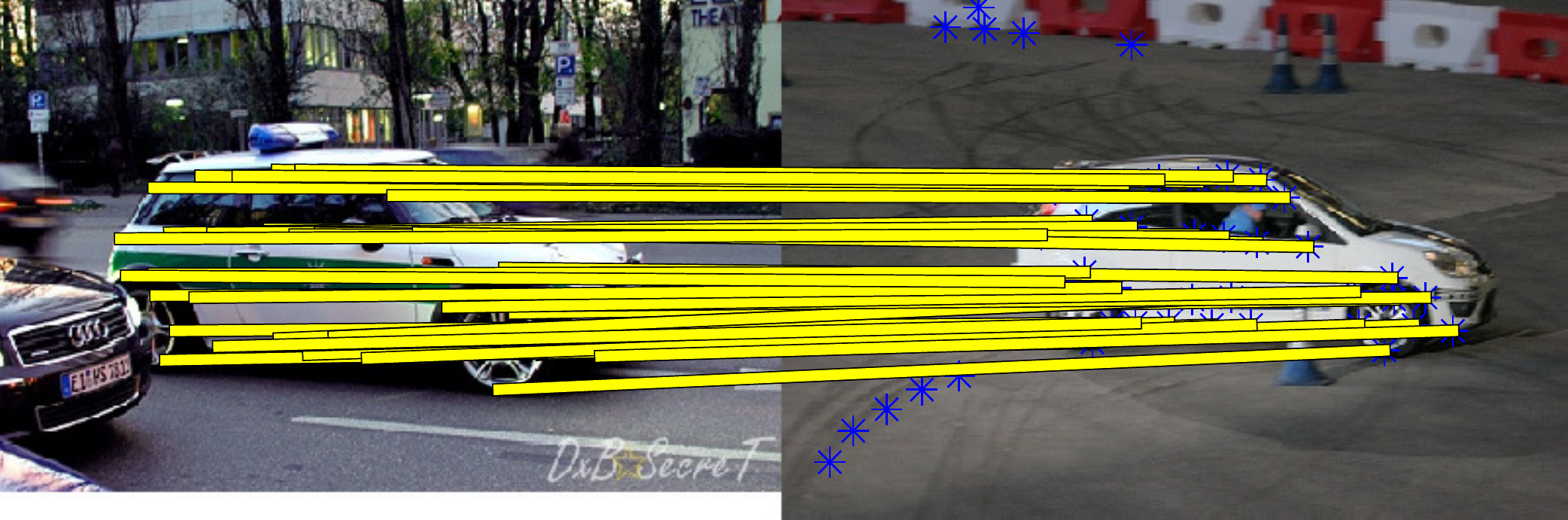} }
\caption{
    Car dataset: 
    the number of correct matches and the objective score are reported.
    (Best viewed in color.)
}
\vspace{-0pt}
\label{fig:exp_car}
\end{figure}
\begin{figure}
    \vspace{-11pt}
    \subfloat[$23$ pts vs $28$ pts, $5$ outliers]{ \includegraphics[width=0.4749\linewidth]{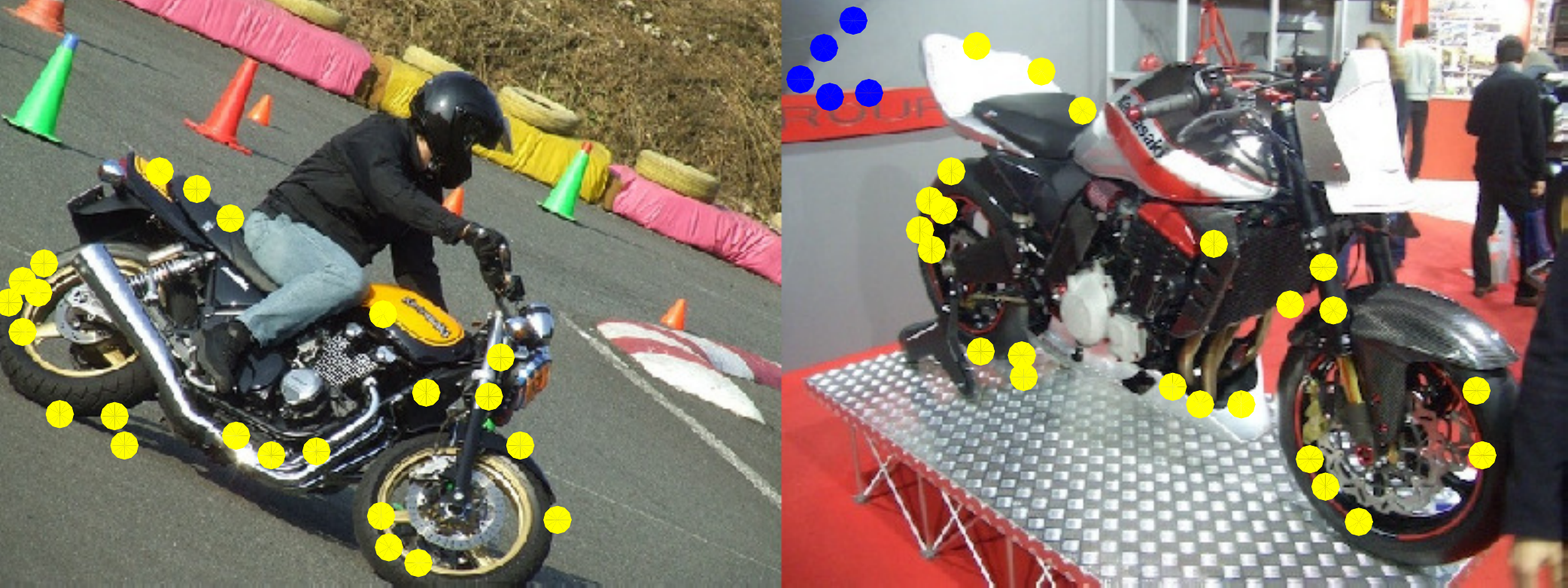} }
    \vspace{-5.2pt}
    \subfloat[TM 9/23 (2502.6)]{ \includegraphics[width=0.4749\linewidth]{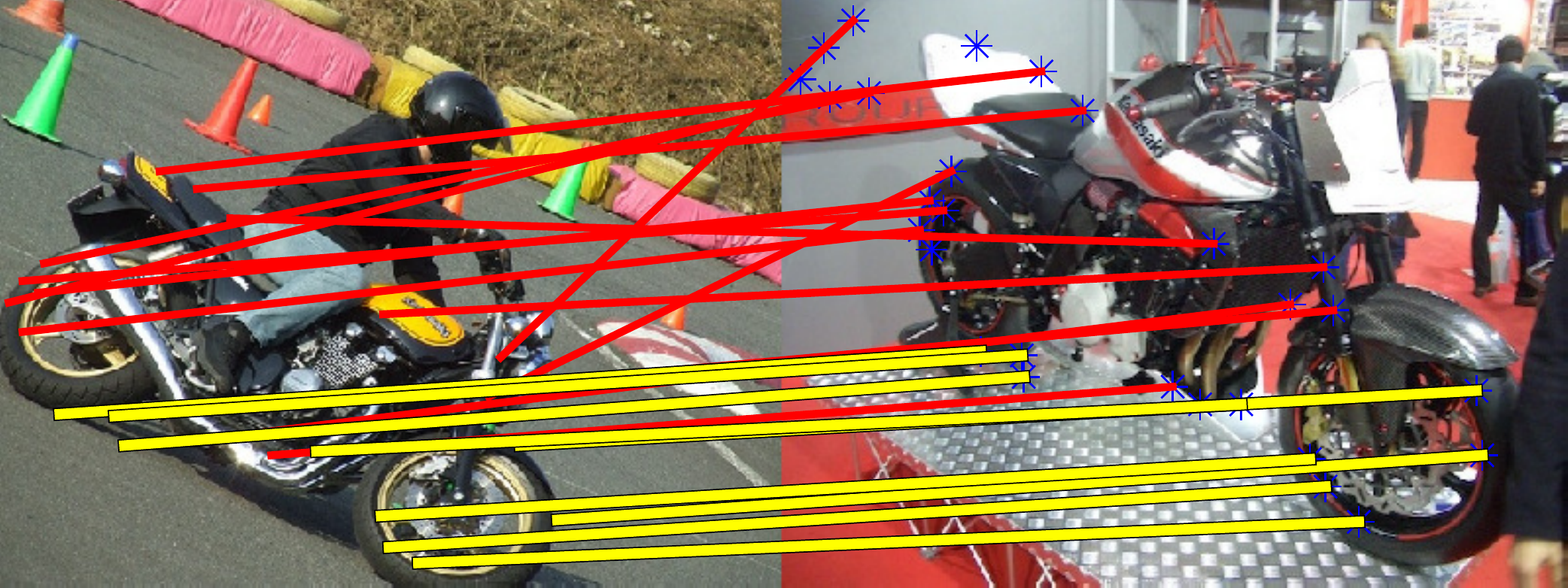} } \\
    \subfloat[HGM 7/23 (1435.8)]{ \includegraphics[width=0.4749\linewidth]{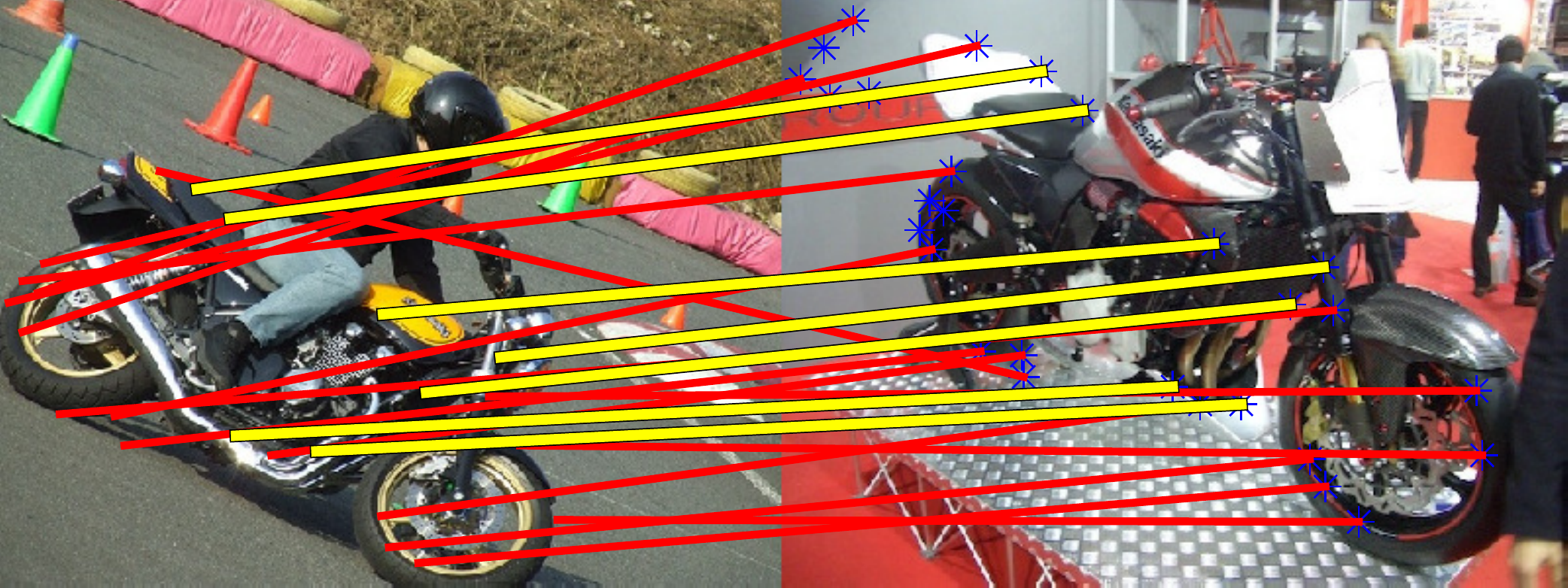} }  
    \vspace{-5.2pt}
    \subfloat[RRWHM 16/23 (4059.4)]{ \includegraphics[width=0.4749\linewidth]{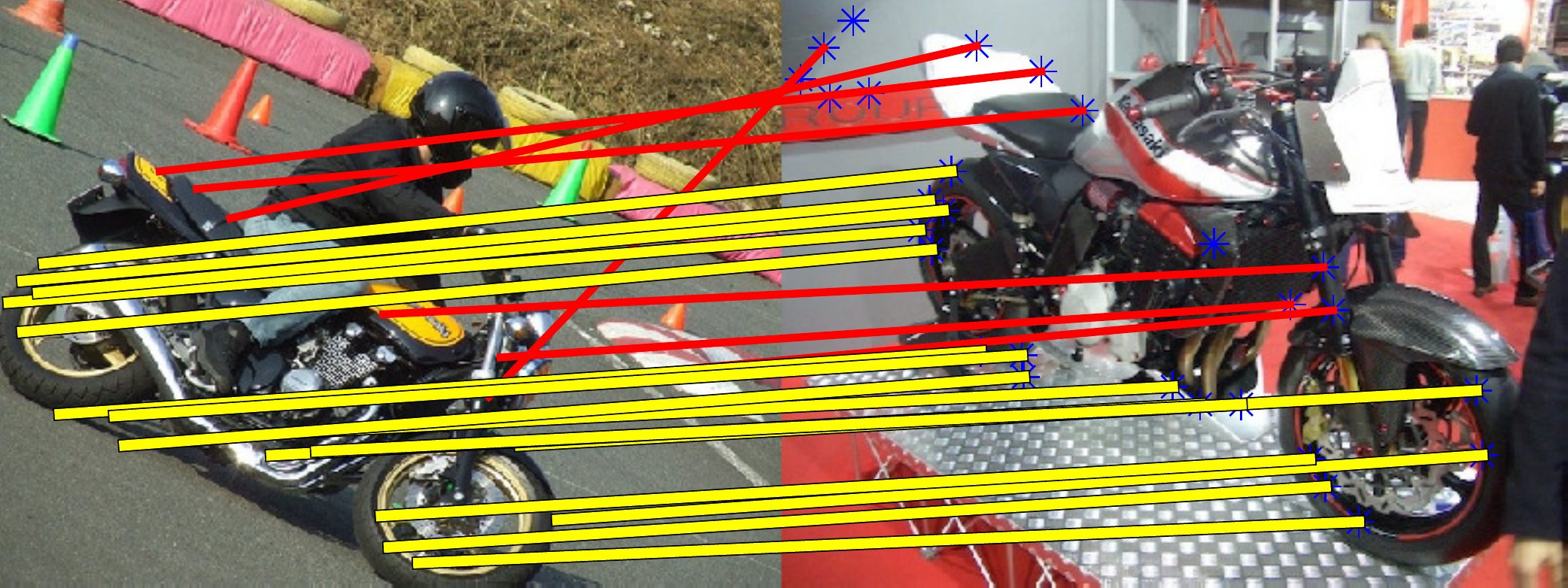} } \\
    \subfloat[BCAGM 19/23 (4016.2)]{ \includegraphics[width=0.4749\linewidth]{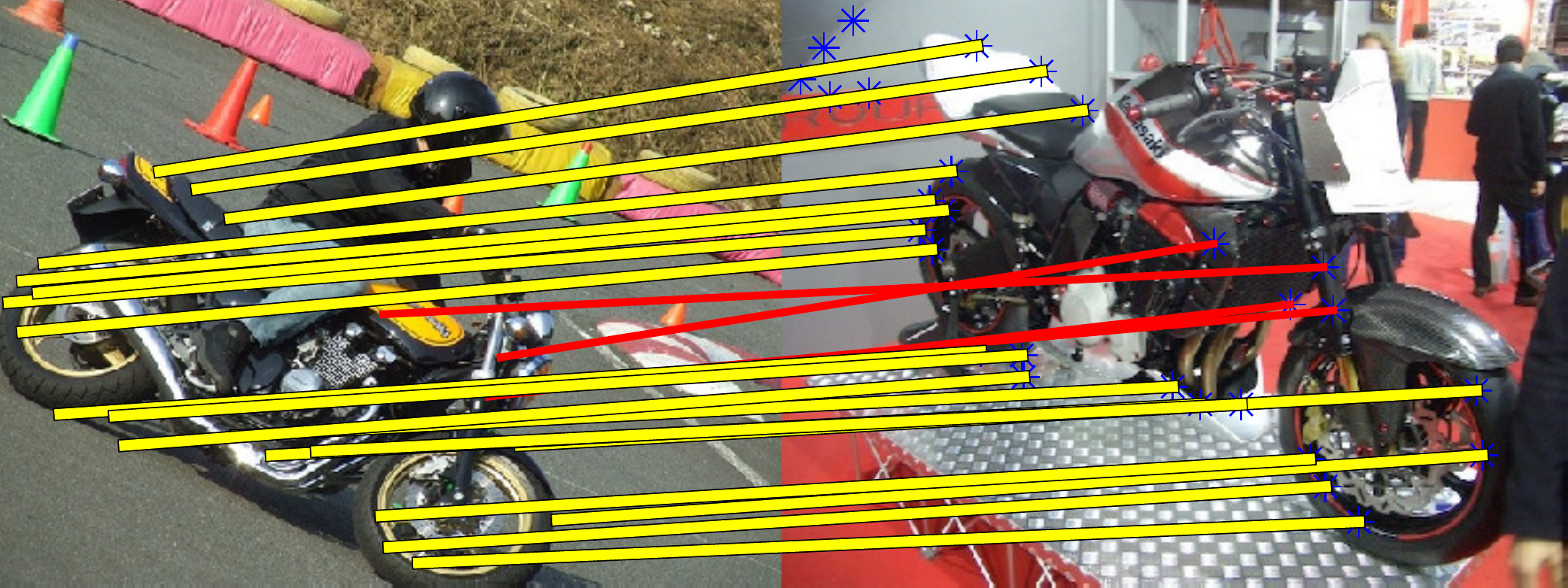} }
    \vspace{-5.2pt}
    \subfloat[BCAGM+MP 16/23 (4133.6)]{ \includegraphics[width=0.4749\linewidth]{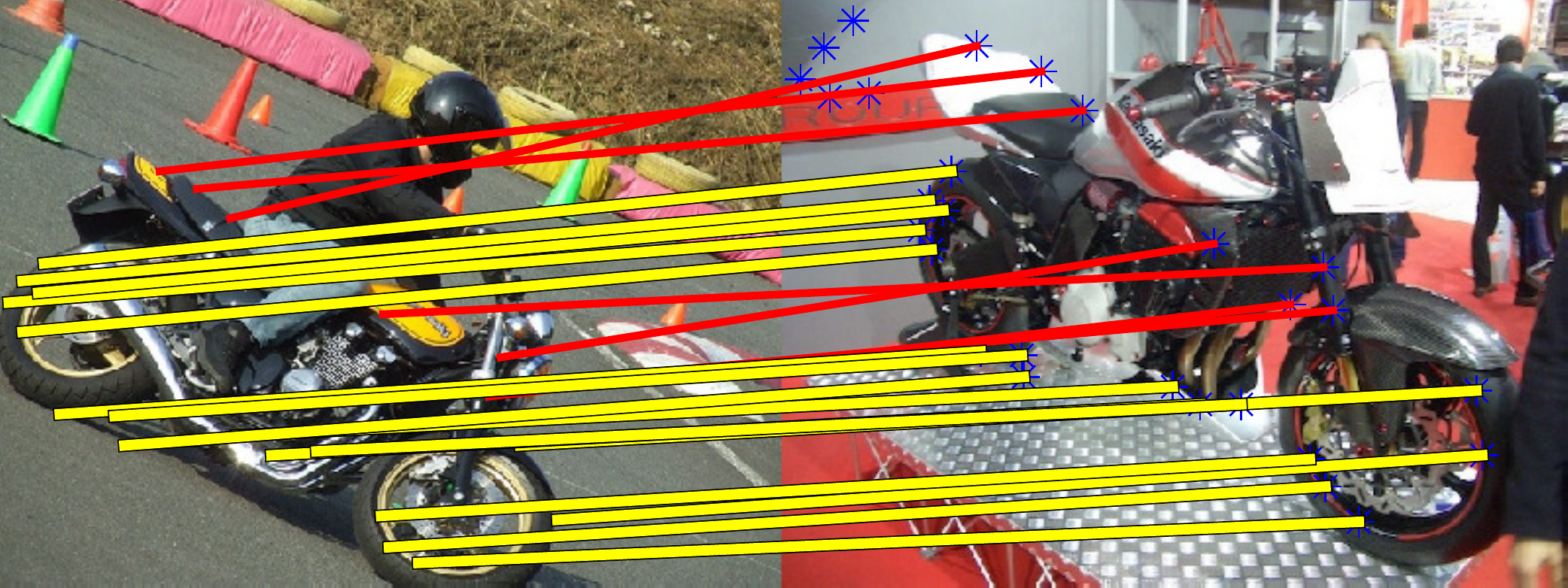} }
\caption{
    Motorbike dataset:
    the number of correct matches and the objective score are reported.
    (Best viewed in color.)
}
\vspace{-8pt}
\label{fig:exp_motor}
\end{figure}

\subsection{CMU House Dataset}\label{subsec:exp_house}
The CMU house dataset has been widely used in previous work ~\cite{ChoLeeLee2010,LeeChoLee2011,DucEtAl2011,ZhoTor2013} 
to evaluate matching algorithms. 
In this dataset, 30 landmark points are manually tracked over a sequence of 111 images, which are taken from the same object under different view points.
In this experiment, ``baseline'' denotes the distance of the frames in the sequence and thus correlates
well with the difficulty to match the corresponding frames.

We matched all possible image pairs with ``baseline'' of $10, 20, 30, \ldots, 100$ frames 
and computed the average matching accuracy for each algorithm. 
The algorithms are evaluated in three settings. In the first experiment, we match $30$ points to $30$ points.
Then we make the problem significantly harder by randomly removing points from one image motivated by a scenario 
where one has background clutter in an image and thus not all points can be matched. 
This results in two matching experiments, namely $10$ points to $30$ points, and $20$ points to $30$ points. 
For the choice of $\sigma_s$ in the affinity tensor for second order methods, 
we follow ~\cite{ChoLeeLee2010,ZhoTor2013} by setting $\sigma_s = 2500$.

The experimental results are shown in Figure \ref{fig:exp_house}. While most algorithms perform rather well on the $30$ to $30$ task, our methods perform significantly better than all other methods in the more difficult tasks, thus showing as for the synthetic datasets that our methods are quite robust
to different kind of noise in the matching problem.
\begin{figure}
    \vspace{-11pt}
    \subfloat{ \includegraphics[width=0.4749\linewidth]{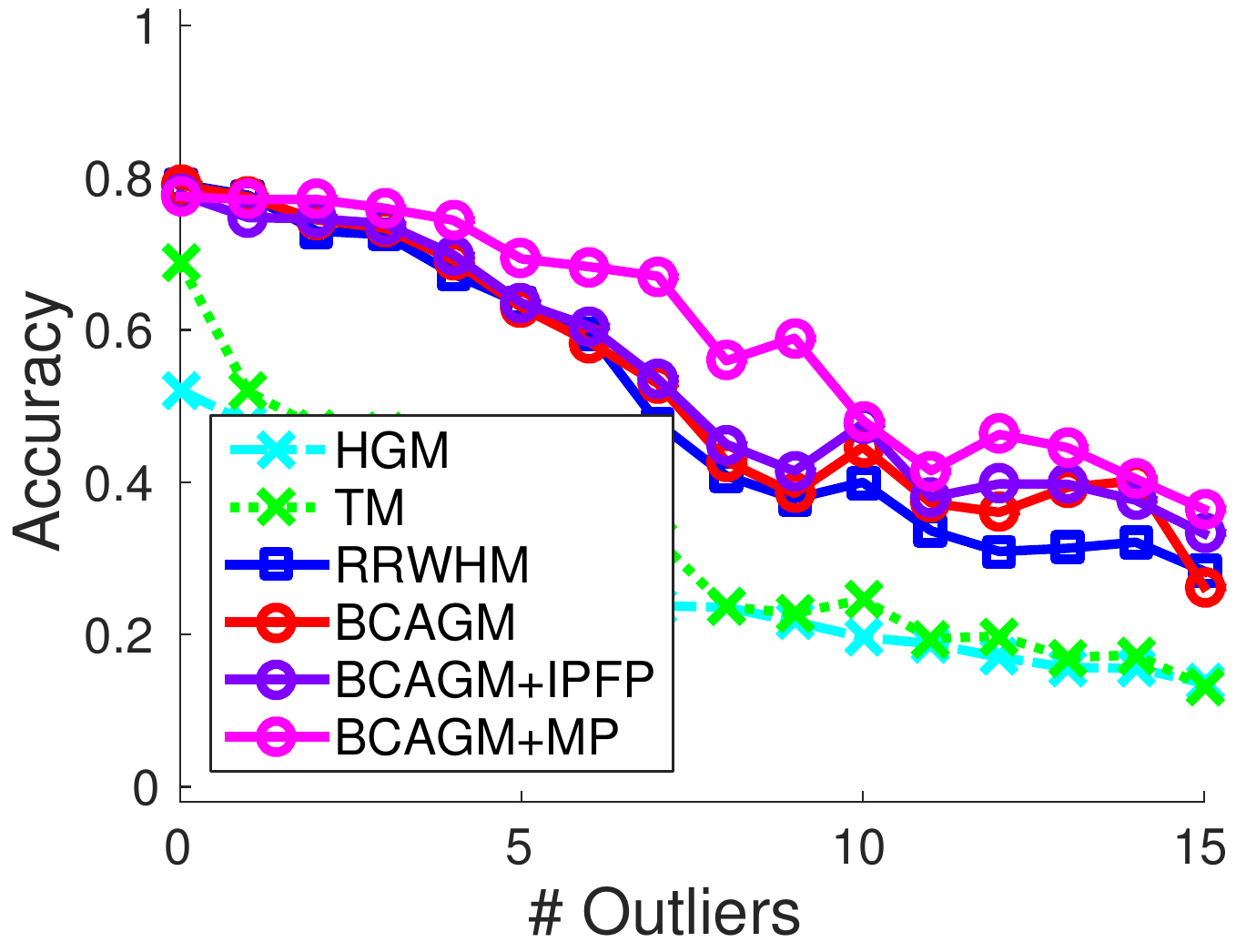} }
    \vspace{-5.2pt}
    \subfloat{ \includegraphics[width=0.4749\linewidth]{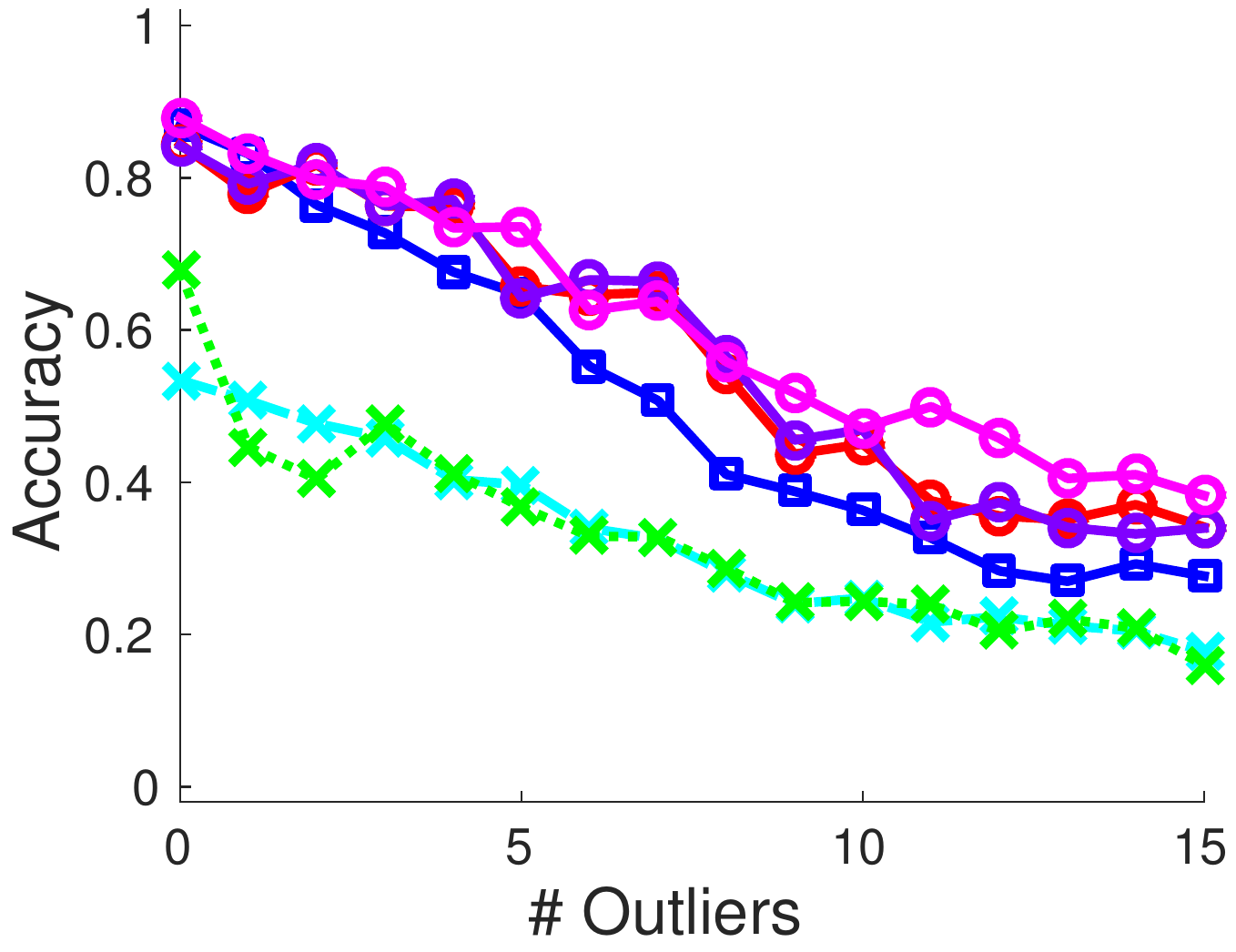} } \\
    \setcounter{subfigure}{0}
    \subfloat[Car]{ \includegraphics[width=0.4749\linewidth]{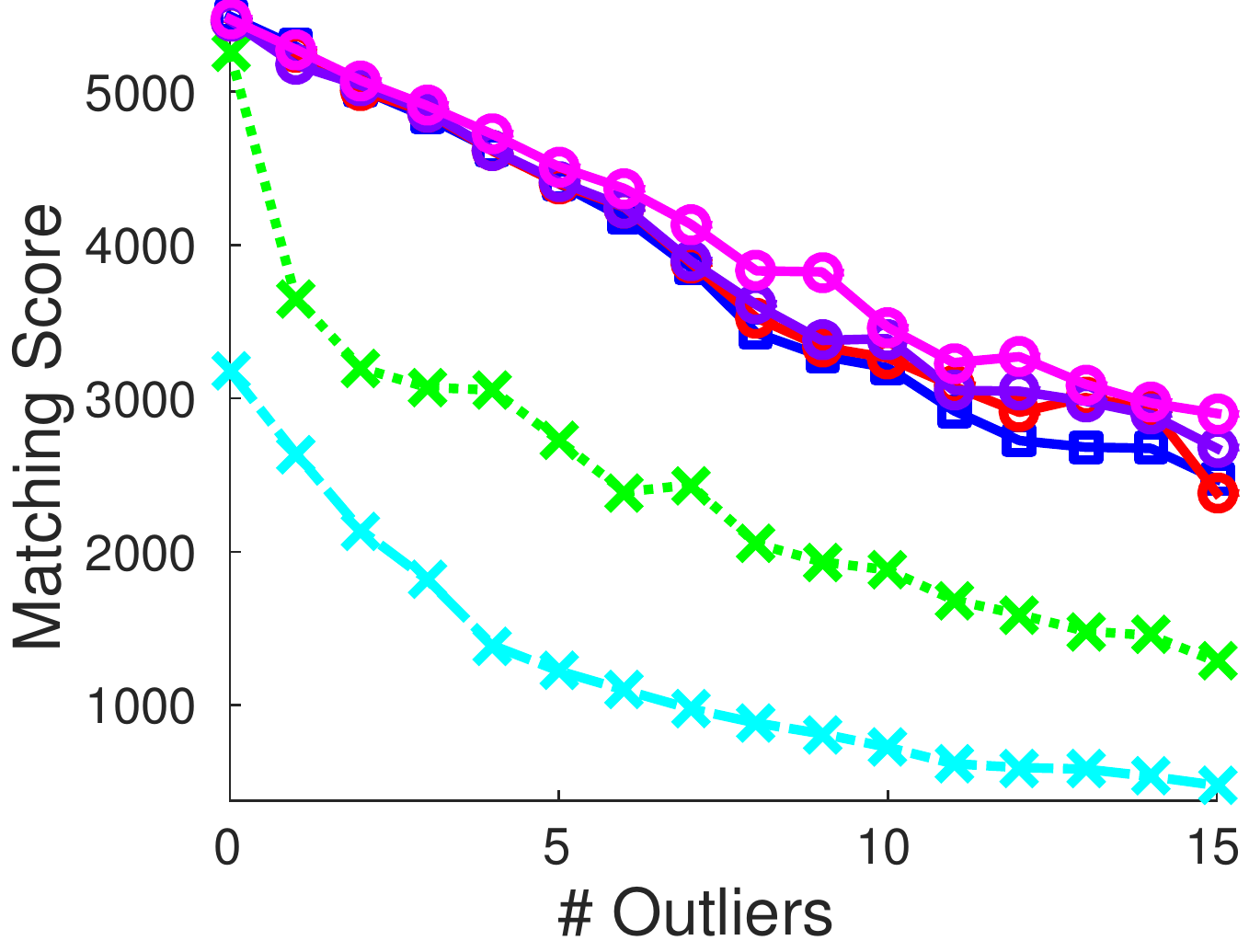} }  
    \vspace{-5.2pt}
    \subfloat[Motorbike]{ \includegraphics[width=0.4749\linewidth]{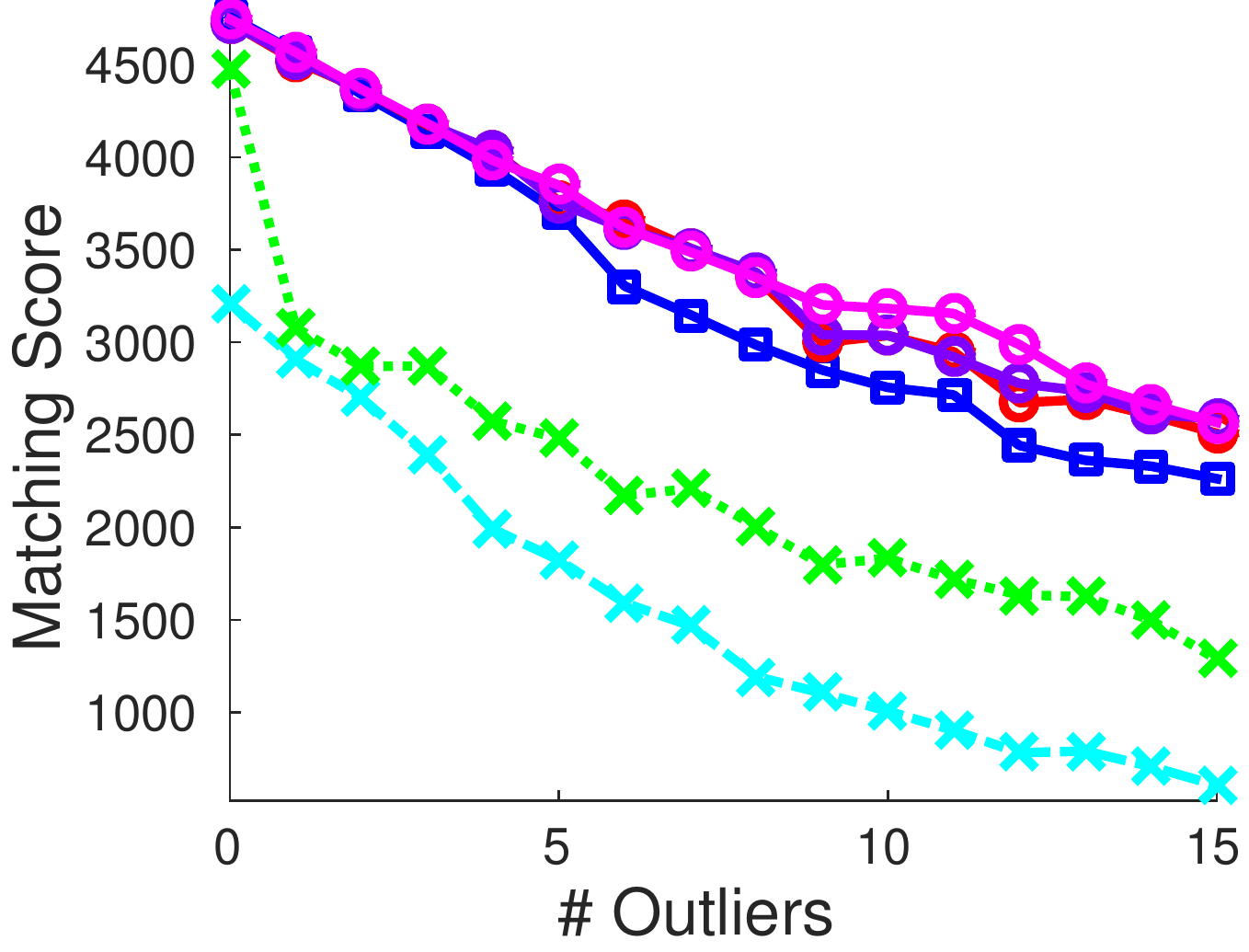} }
\caption{
    Evaluation of higher order GM algorithms on the car and motorbike dataset. 
    (Best viewed in color.)
}
\vspace{-8pt}
\label{fig:exp_car_motor_acc}
\end{figure}

\subsection{Car Motorbike Dataset}\label{subsec:exp_car_motor}
In this experiment, we compare our algorithms with other third order approaches
on the car and motorbike dataset introduced in \cite{LeoHeb2009}.
The dataset contains $30$ resp. $20$ pairs of car resp. motorbike images.
Each image contains a number of inliers and outliers from the cluttered background. 
Ground-truth correspondences are given for the inlier points in both images.
Figures \ref{fig:exp_car} and \ref{fig:exp_motor} show some examples of matching results.
To test the algorithms against noise, we kept all the inlier points in both images and randomly added a number of outliers
from the cluttered background to the second image (between $0$ to $15$).
For each number of outliers, we test all algorithms on all the image pairs of each dataset 
and report the average performance in Figure \ref{fig:exp_car_motor_acc}. 
It shows that our methods achieve better overall performance than other third order approaches in both datasets.
However, we expect that the performance can be significantly improved 
if also other features are integrated rather than just relying on  geometric information \cite{ZhoTor2012, ZhoTor2013},
or graph learning methods can be also be integrated \cite{ChoEtal2013, CaeEtal2009, LeoHeb2009, Deepti2012, TorEtal2008}.
\section{Conclusion and Outlook}
We have presented a new optimization framework for higher order graph matching. 
Compared to previous methods for higher order matching
we can guarantee monotonic ascent for our algorithms in the matching score on the set of assignment matrices. 
Our new algorithms achieve superior performance in terms of objective but also yield competitive or 
significantly better matching accuracy. 
This is particularly true for large number of outliers and other forms of noise. Moreover, both algorithms
are also competitive in terms of runtime compared to other methods.

An interesting line of future research is the use of globally optimal solutions of relaxations
of the hypergraph matching problem. We computed via the method of \cite{Antoine2015} the
maximal singular vectors of the fourth order tensor and used these as initialization of Algorithm \ref{algo:linear}.
This lead to an improvement in matching score of up to 2\% and in accuracy of up to 3\%. 
We will explore this further in future work.

\section*{Acknowledgment}
The authors acknowledge support by the ERC starting grant NOLEPRO.

{\small
\bibliographystyle{ieee}
\bibliography{regul}

\begin{thebibliography}{10}\itemsep=-1pt

\bibitem{Boros2002}
E.~Boros and P.~L. Hammer.
\newblock Pseudo-boolean optimization.
\newblock {\em Discrete Applied Mathematics}, 123:155--225, 2002.

\bibitem{BurAmiMar2012}
R.~E. Burkhard, M.~Dell'Amico, and S.~Martello.
\newblock {\em Assignment problems}.
\newblock SIAM, Philadelphia, 2012.

\bibitem{CaeEtal2009}
T.~S. Caetano, J.~J. McAuley, L.~Cheng, Q.~V. Le, and A.~J. Smola.
\newblock Learning graph matching.
\newblock {\em PAMI}, 31:1048--1058, 2009.

\bibitem{CheKel2010}
M.~Chertok and Y.~Keller.
\newblock Efficient high order matching.
\newblock {\em PAMI}, 32:2205--2215, 2010.

\bibitem{ChoEtal2013}
M.~Cho, K.~Alahari, and J.~Ponce.
\newblock Learning graphs to match.
\newblock In {\em ICCV}, 2013.

\bibitem{ChoLeeLee2010}
M.~Cho, J.~Lee, and K.~M. Lee.
\newblock Reweighted random walks for graph matching.
\newblock In {\em ECCV}, 2010.

\bibitem{ChoLee2012}
M.~Cho and K.~M. Lee.
\newblock Progressive graph matching: Making a move of graphs via probabilistic
  voting.
\newblock In {\em CVPR}, 2012.

\bibitem{ChoEtAl2014}
M.~Cho, J.~Sun, O.~Duchenne, and J.~Ponce.
\newblock Finding matches in a haystack: A max-pooling strategy for graph
  matching in the presence of outliers.
\newblock In {\em CVPR}, 2014.

\bibitem{ShiEtal2007}
T.~Cour, P.~Srinivasan, and J.~Shi.
\newblock Balanced graph matching.
\newblock In {\em NIPS}, 2007.

\bibitem{DucEtAl2011}
O.~Duchenne, F.~Bach, I.~Kweon, and J.~Ponce.
\newblock A tensor-based algorithm for high-order graph matching.
\newblock {\em PAMI}, 33:2383--2395, 2011.

\bibitem{DucJouPon2011}
O.~Duchenne, A.~Joulin, and J.~Ponce.
\newblock A graph-matching kernel for object categorization.
\newblock In {\em ICCV}, 2011.

\bibitem{Antoine2015}
A.~Gautier and M.~Hein.
\newblock Tensor norm and maximal singular vectors of non-negative tensors - a
  perron-frobenius theorem, a collatz-wielandt characterization and a
  generalized power method.
\newblock {\em arXiv:1503.01273}, 2015.

\bibitem{Kolmogorov2007}
V.~Kolmogorov and C.~Rother.
\newblock Minimizing nonsubmodular functions with graph cuts-a review.
\newblock {\em PAMI}, 29:1274--1279, 2007.

\bibitem{Kuh1955}
H.~W. Kuhn.
\newblock The hungarian method for the assignment problem.
\newblock {\em Naval Research Logistics Quarterly}, 2:83--97, 1955.

\bibitem{LeeChoLee2010}
J.~Lee, M.~Cho, and K.~M. Lee.
\newblock A graph matching algorithm using data-driven markov chain monte carlo
  sampling.
\newblock In {\em ICPR}, 2010.

\bibitem{LeeChoLee2011}
J.~Lee, M.~Cho, and K.~M. Lee.
\newblock Hyper-graph matching via reweighted random walks.
\newblock In {\em CVPR}, 2011.

\bibitem{LeoHeb2005}
M.~Leordeanu and M.~Hebert.
\newblock A spectral technique for correspondence problems using pairwise
  constraints.
\newblock In {\em ICCV}, 2005.

\bibitem{LeoHeb2009}
M.~Leordeanu and M.~Hebert.
\newblock Unsupervised learning for graph matching.
\newblock In {\em CVPR}, 2009.

\bibitem{LeoHebSuk2009}
M.~Leordeanu, M.~Hebert, and R.~Sukthankar.
\newblock An integer projected fixed point method for graph matching and map
  inference.
\newblock In {\em NIPS}, 2009.

\bibitem{Low1999}
D.~G. Lowe.
\newblock Object recognition from local scale-invariant features.
\newblock In {\em ICCV}, 1999.

\bibitem{Joa2003}
J.~Maciel and J.~P. Costeira.
\newblock A global solution to sparse correspondence problems.
\newblock {\em PAMI}, 25:187--199, 2003.

\bibitem{Deepti2012}
D.~Pachauri, M.~Collins, V.~Singh, and R.~Kondor.
\newblock Incorporating domain knowledge in matching problems via harmonic
  analysis.
\newblock In {\em ICML}, 2012.

\bibitem{SharEtal2011}
A.~Sharma, R.~Horaud, J.~Cech, and E.~Boyer.
\newblock Topologically- robust 3d shape matching based on diffusion geometry
  and seed growing.
\newblock In {\em CVPR}, 2011.

\bibitem{TorEtal2008}
L.~Torresani, V.~Kolmogorov, and C.~Rother.
\newblock Feature correspondence via graph matching: Models and global
  optimization.
\newblock In {\em ECCV}, 2008.

\bibitem{ZasEtl2009}
M.~Zaslavskiy, F.~Bach, and J.~Vert.
\newblock A path following algorithm for the graph matching problem.
\newblock {\em PAMI}, 31:2227--2242, 2009.

\bibitem{ZasSha2008}
R.~Zass and A.~Shashua.
\newblock Probabilistic graph and hypergraph matching.
\newblock In {\em CVPR}, 2008.

\bibitem{Zeng2010}
Y.~Zeng, C.~Wang, Y.~Wang, X.~Gu, D.~Samaras, and N.~Paragios.
\newblock Dense non-rigid surface registration using high-order graph matching.
\newblock In {\em CVPR}, 2010.

\bibitem{ZhoTor2012}
F.~Zhou and F.~{De la Torre}.
\newblock Factorized graph matching.
\newblock In {\em CVPR}, 2012.

\bibitem{ZhoTor2013}
F.~Zhou and F.~{De la Torre}.
\newblock Deformable graph matching.
\newblock In {\em CVPR}, 2013.

\end{thebibliography}
}

\newpage
\appendix
In this section, we show some additional experimental results to the main paper.

\section{Synthetic Dataset}
Figure \ref{appendix_fig:exp_out} shows additional results in the outlier setting where 
the number of inliers was fixed to $10$ while $\sigma$ and $scale$ were set to different values.
It is interesting to see that when there is no deformation and scaling, our algorithms together with
RRWHM \cite{LeeChoLee2011} and MPM \cite{ChoEtAl2014} achieve an almost perfect result.
However, our algorithms outperform all other higher order approaches when deformation and scaling are slightly present.
Compared to second order methods, our algorithms can take advantage of higher order features, therefore, achieve superior performance
when transformations such as scaling are present.

Figure \ref{appendix_fig:exp_def} shows further results in the deformation setting, where the number of inliers was set to $30$ and $40$ accordingly,
and no other form of noise was used. 
As we show the runtime for all the experiments, the result for $20$ inliers points is repeated from the paper as well.
One can observe from Figure \ref{appendix_fig:exp_def} that our algorithms always stay competitive with other state-of-the-art higher order methods,
in particular RRWHM \cite{LeeChoLee2011}, even when the deformation is significant.

\begin{figure*}[htb!]
\begin{center}
    \subfloat{\includegraphics[width=0.31\linewidth]{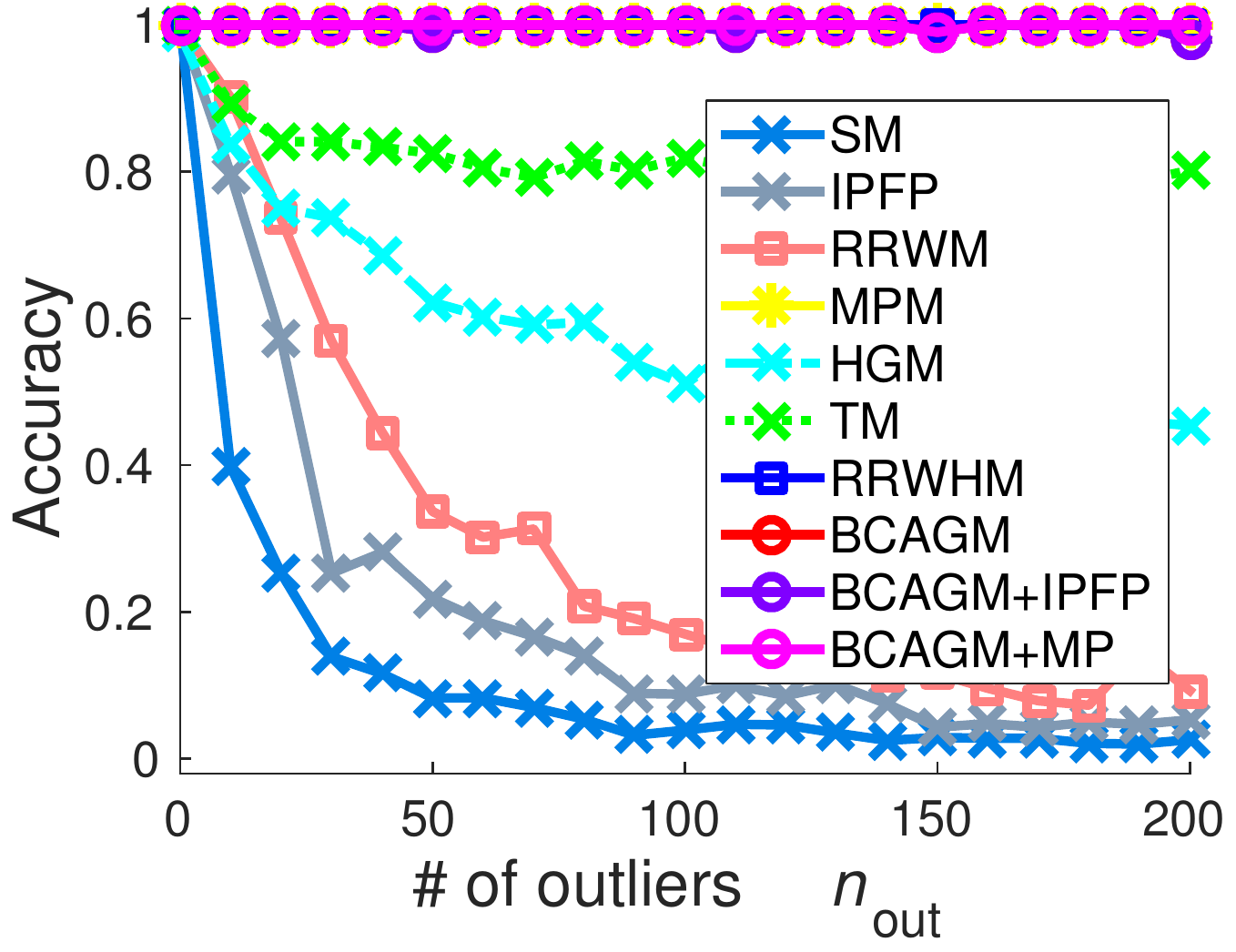} }
    \subfloat{\includegraphics[width=0.31\linewidth]{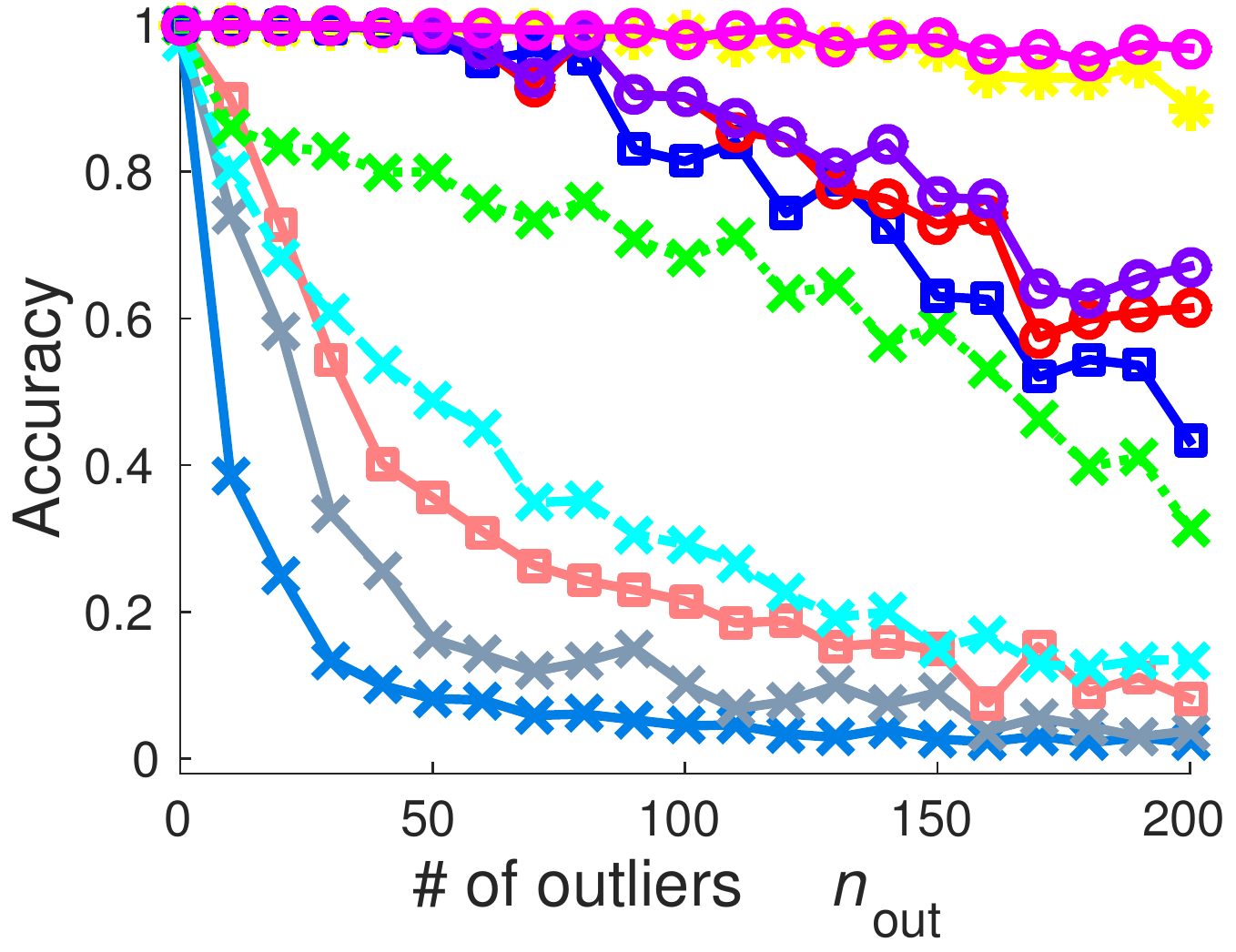} }  
    \subfloat{\includegraphics[width=0.31\linewidth]{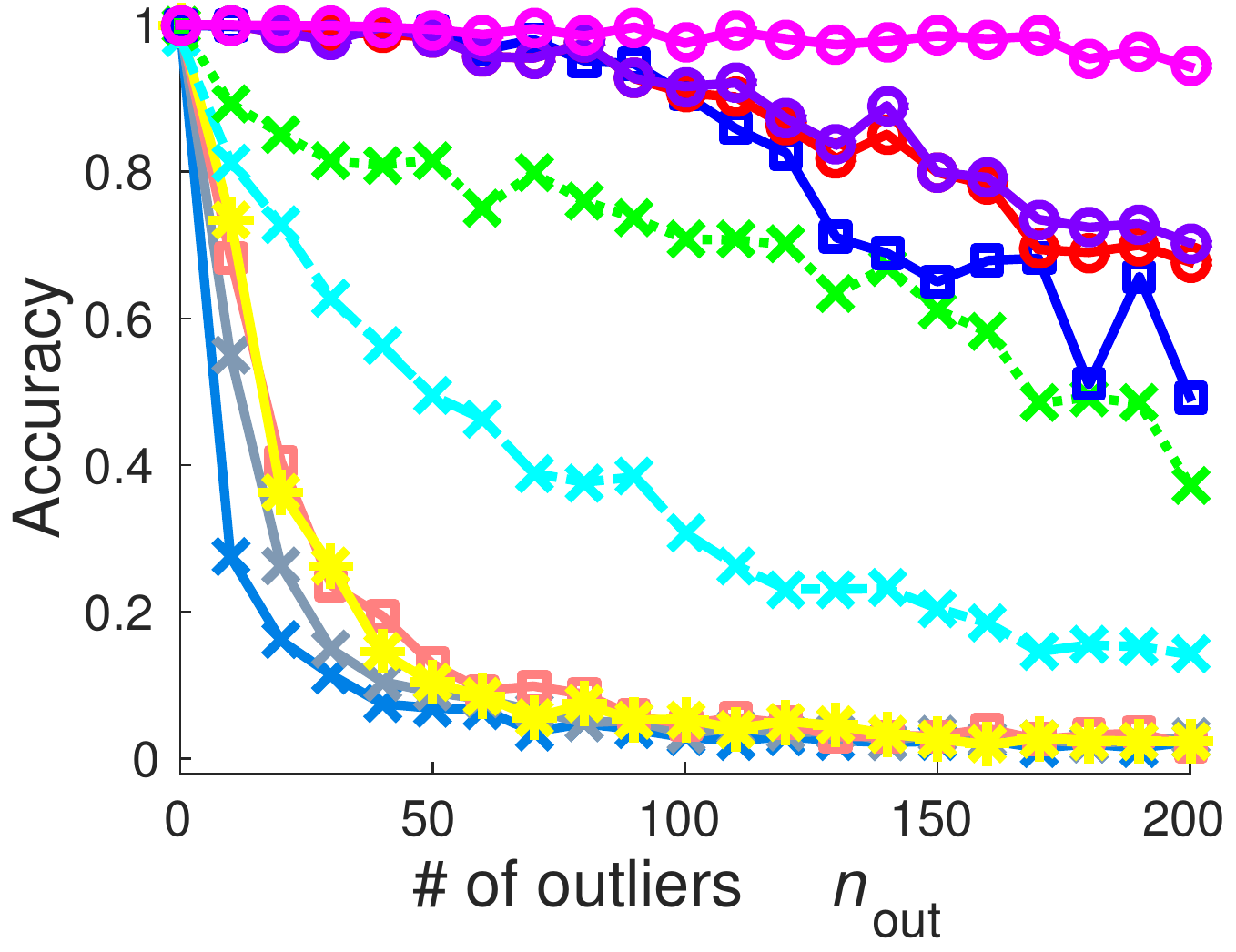} } \\
    
    \subfloat{ \includegraphics[width=0.31\linewidth]{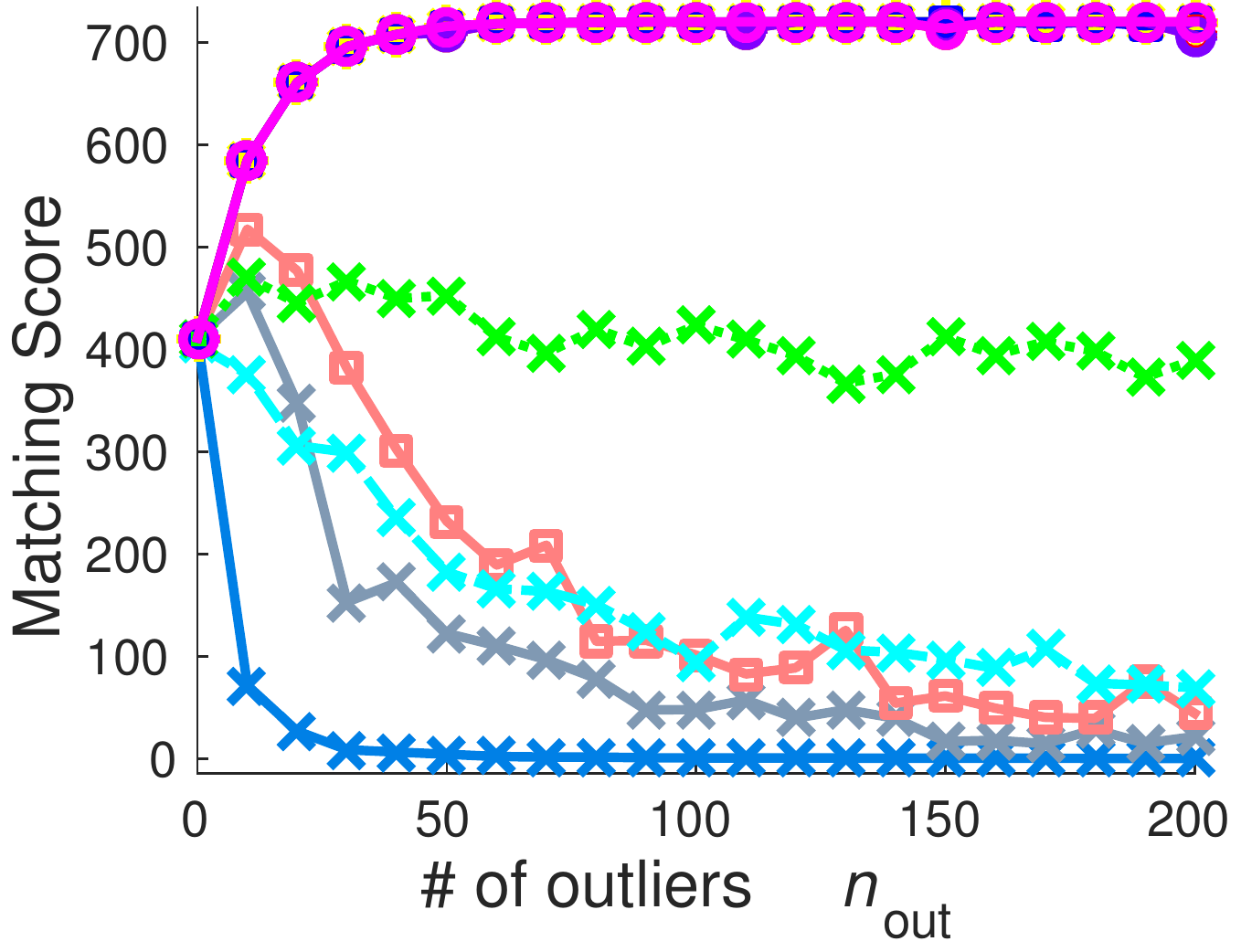}  } 
    \subfloat{\includegraphics[width=0.31\linewidth]{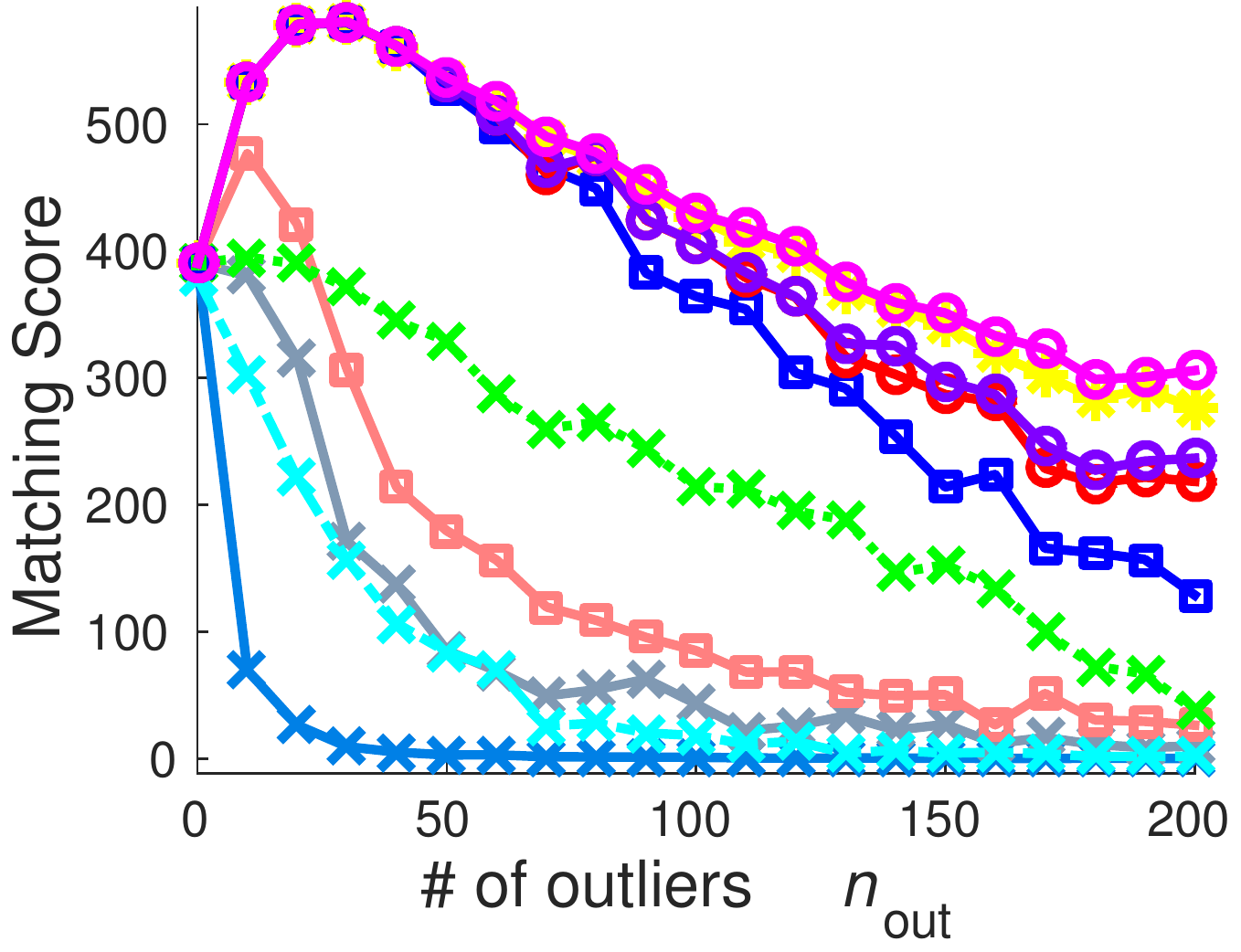} }  
    \subfloat{ \includegraphics[width=0.31\linewidth]{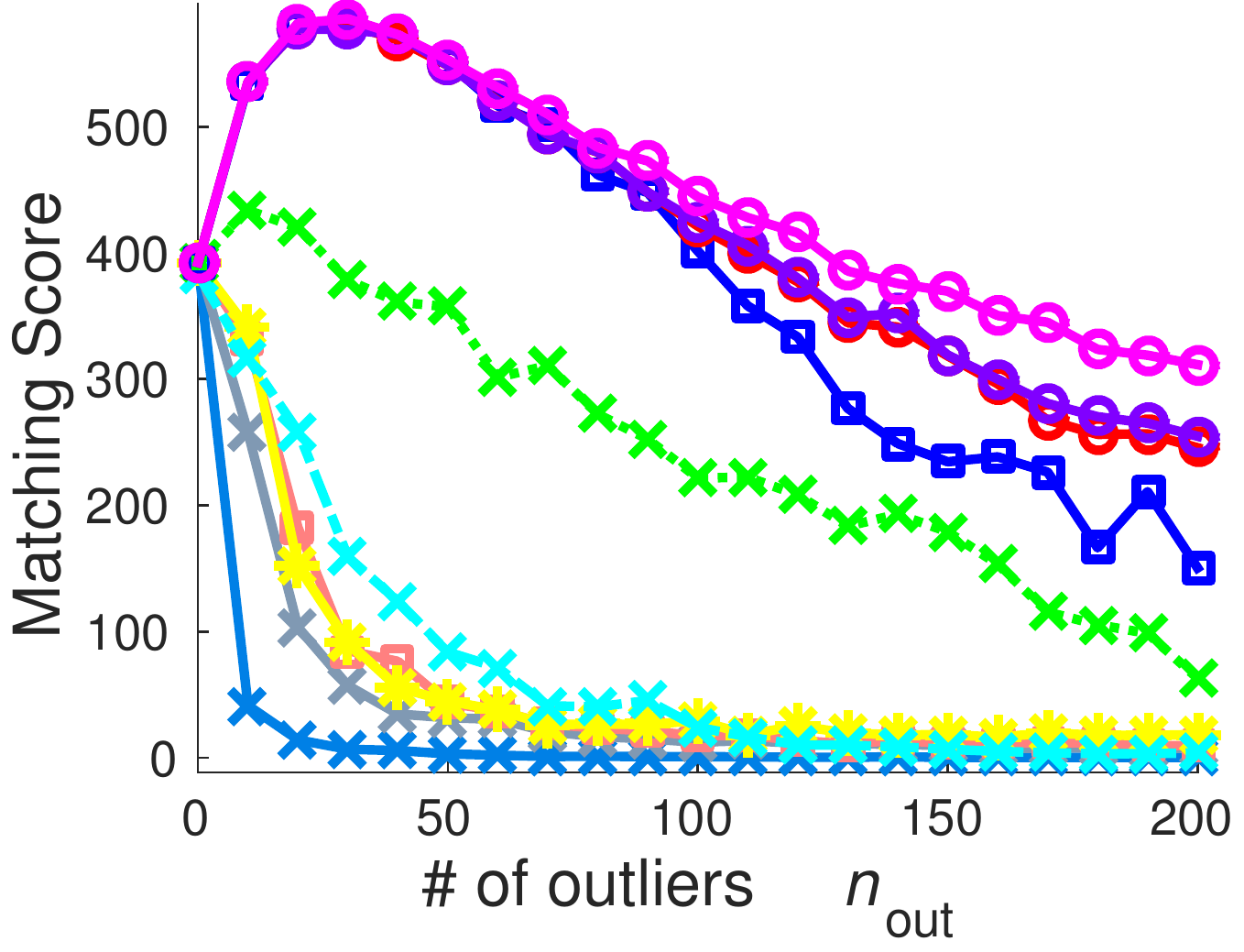} } \\
    
    \setcounter{subfigure}{0}
    \subfloat[$n_{in}=10,\sigma=0,scale=1$]{ \includegraphics[width=0.31\linewidth]{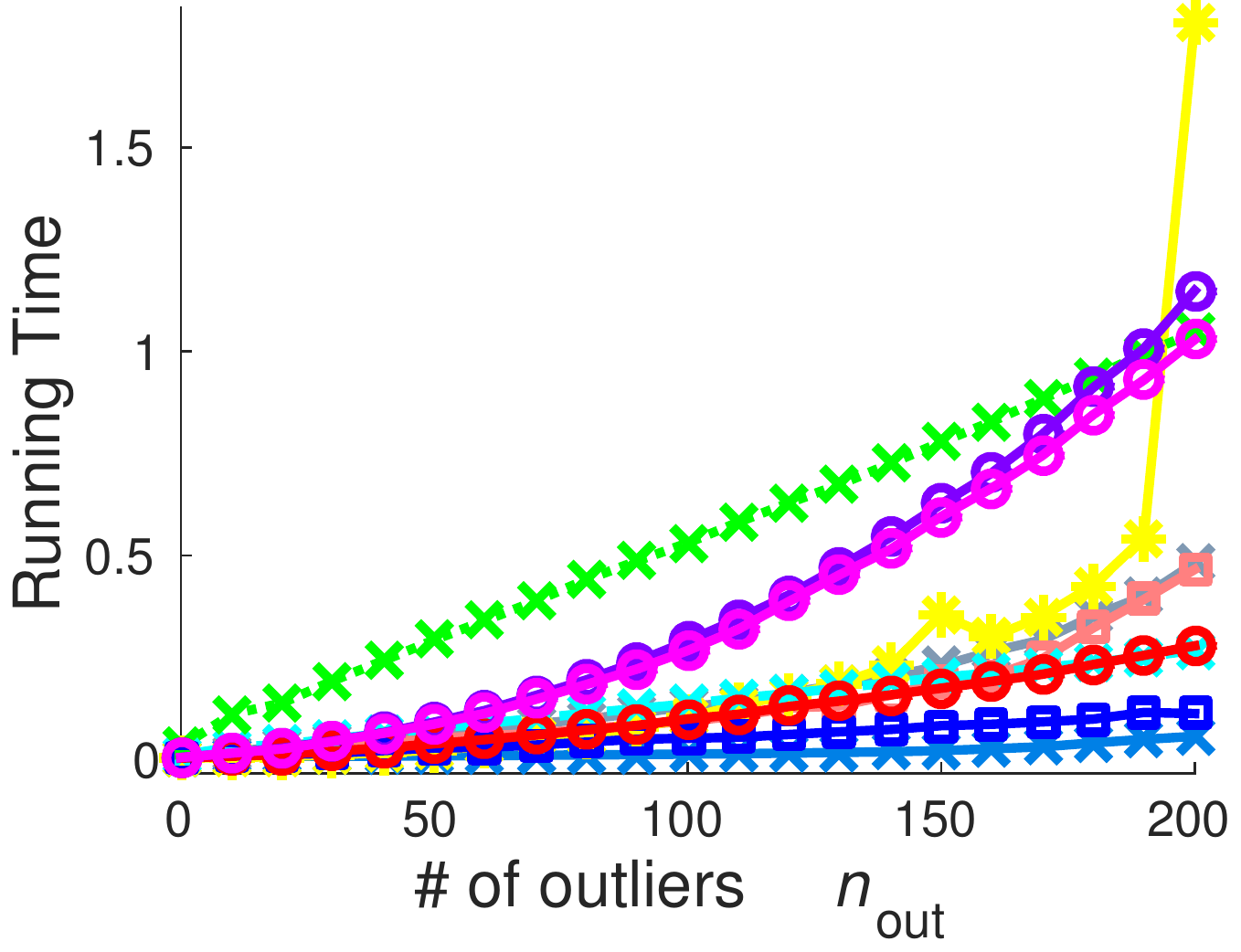} \label{appendix_fig:out_def_noscal} } 
    \subfloat[$n_{in}=10,\sigma=0.01,scale=1.01$]{ \includegraphics[width=0.31\linewidth]{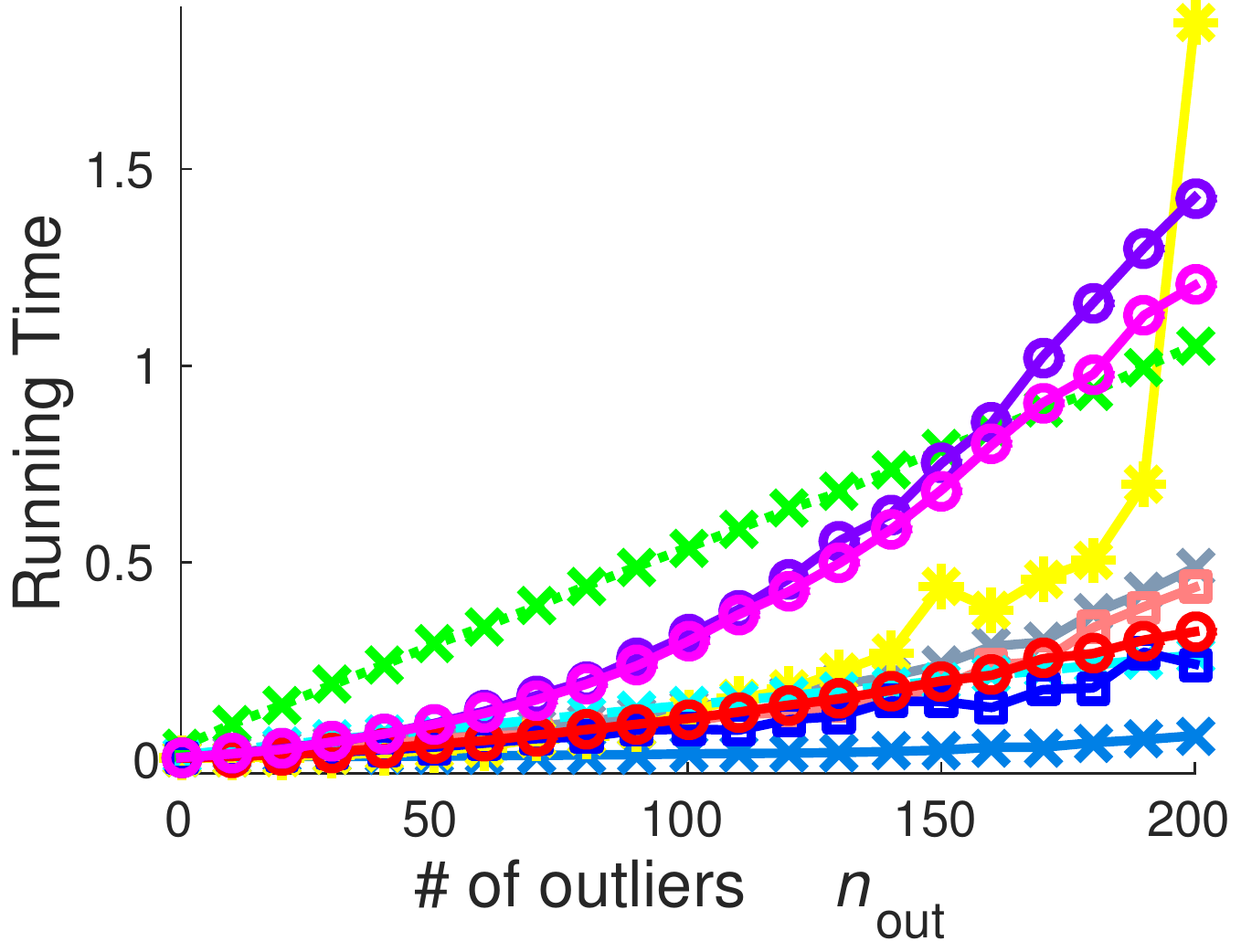} \label{appendix_fig:def} }
    \subfloat[$n_{in}=10,\sigma=0.01,scale=1.1$]{ \includegraphics[width=0.31\linewidth]{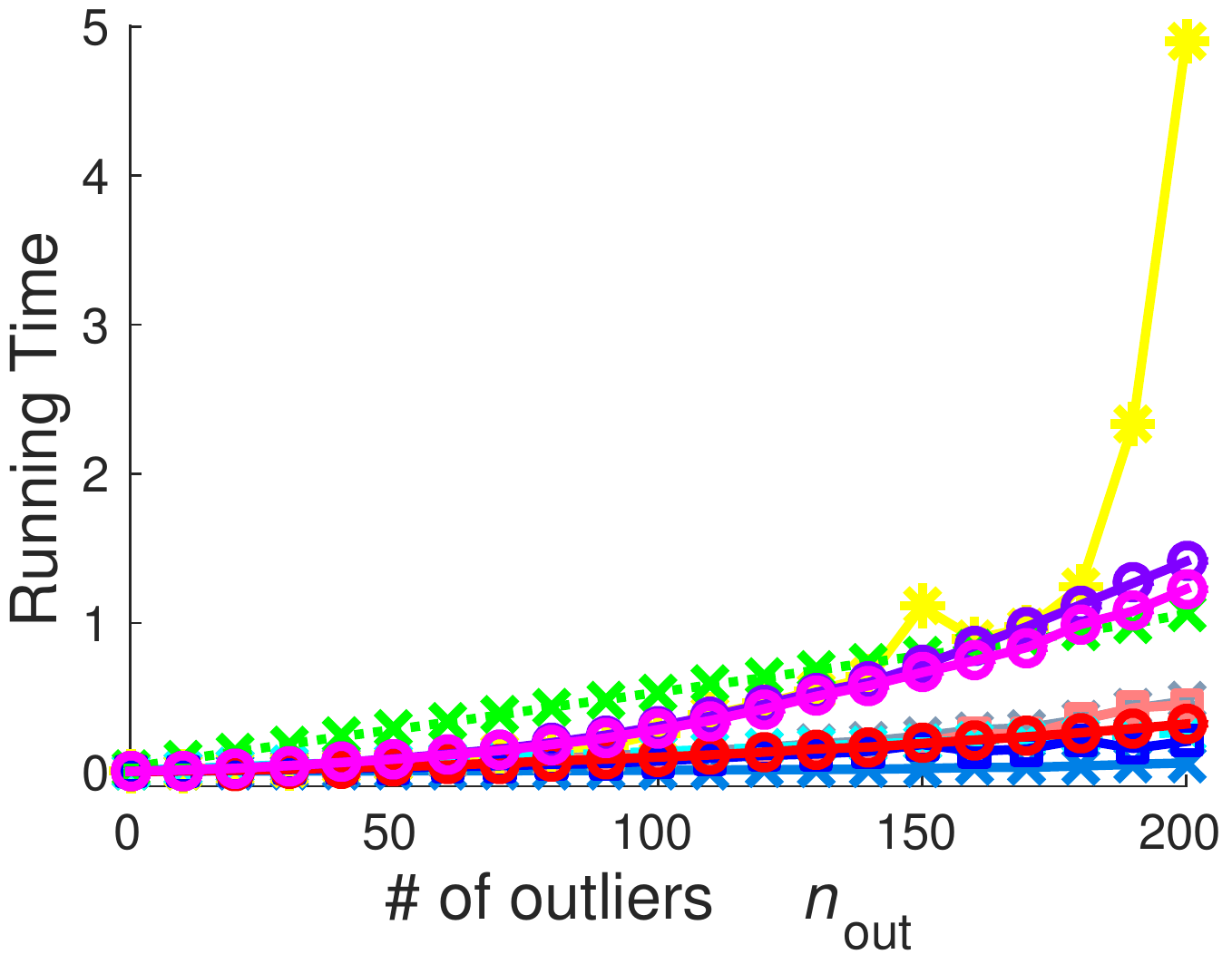} \label{appendix_fig:out_def_scal} } 
\end{center}    
\caption{
Matching point sets in $\R^2$ (outliers test):
The top row shows average accuracy while the middle row shows
the average matching score and the bottom row shows average running time. 
The number of outliers was varied from $0$ to $200$ with the interval $10$. 
(a) Increasing number of outliers without deformation and scaling.
(b) Increasing number of outliers with slight deformation and small scaling.
(c) Increasing number of outliers with slight deformation and large scaling.
(Best viewed in color.)
}
\label{appendix_fig:exp_out}
\end{figure*}

\begin{figure*}[htb!]
\begin{center}    
    \subfloat{\includegraphics[width=0.31\linewidth]{def_20inlier_acc1} }
    \subfloat{\includegraphics[width=0.31\linewidth]{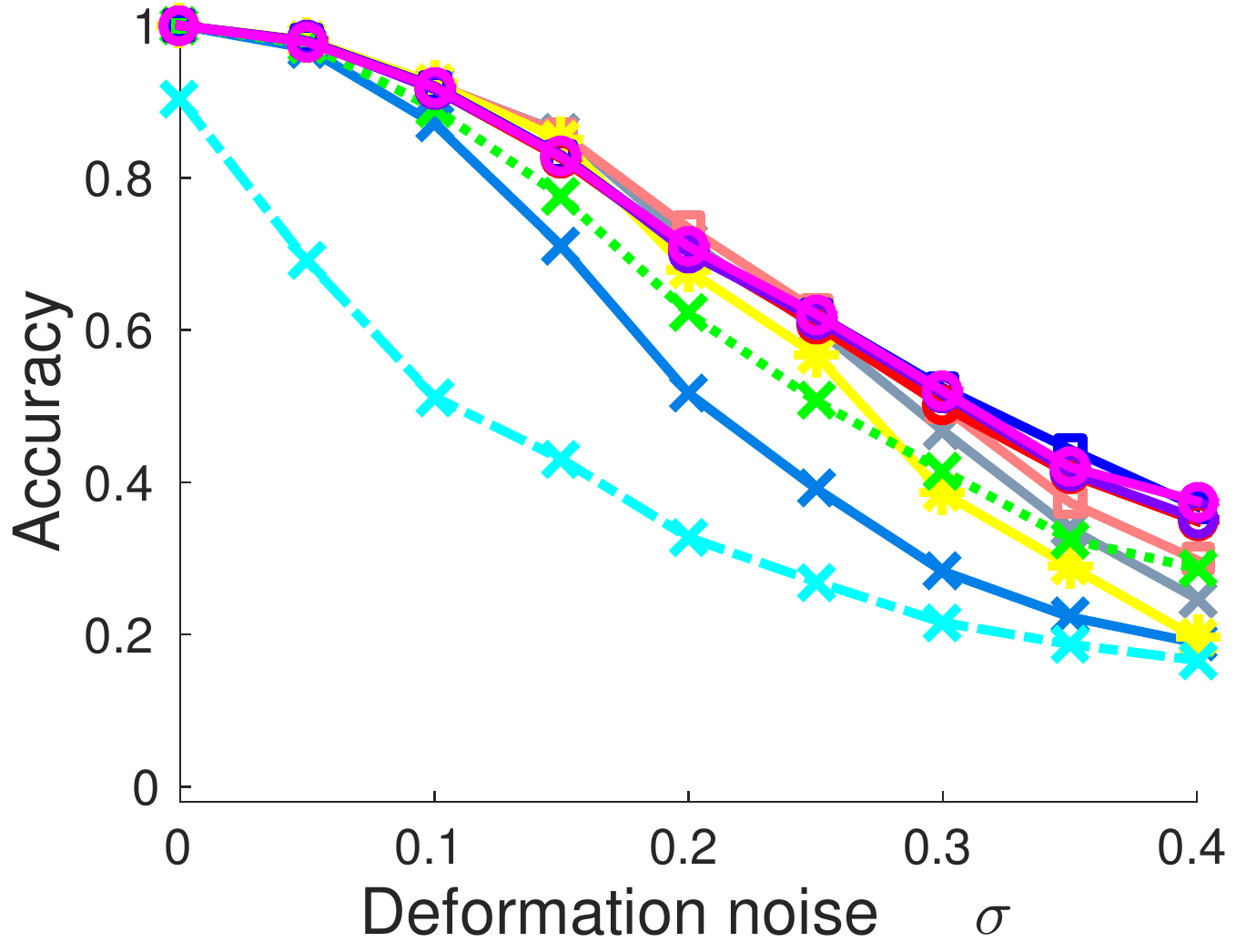} } 
    \subfloat{\includegraphics[width=0.31\linewidth]{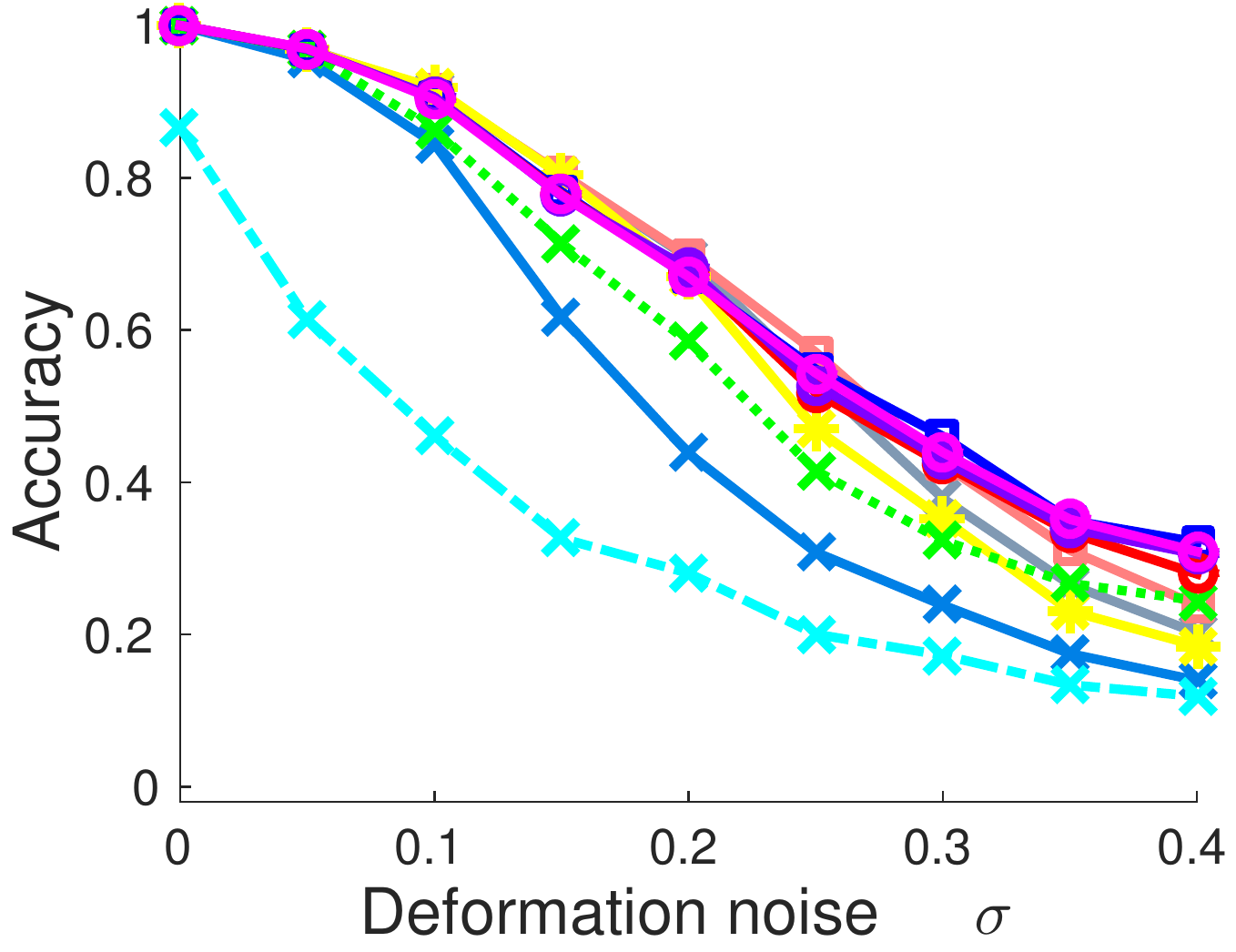} }  \\
    
    \subfloat{ \includegraphics[width=0.31\linewidth]{def_20inlier_score0}  } 
    \subfloat{ \includegraphics[width=0.31\linewidth]{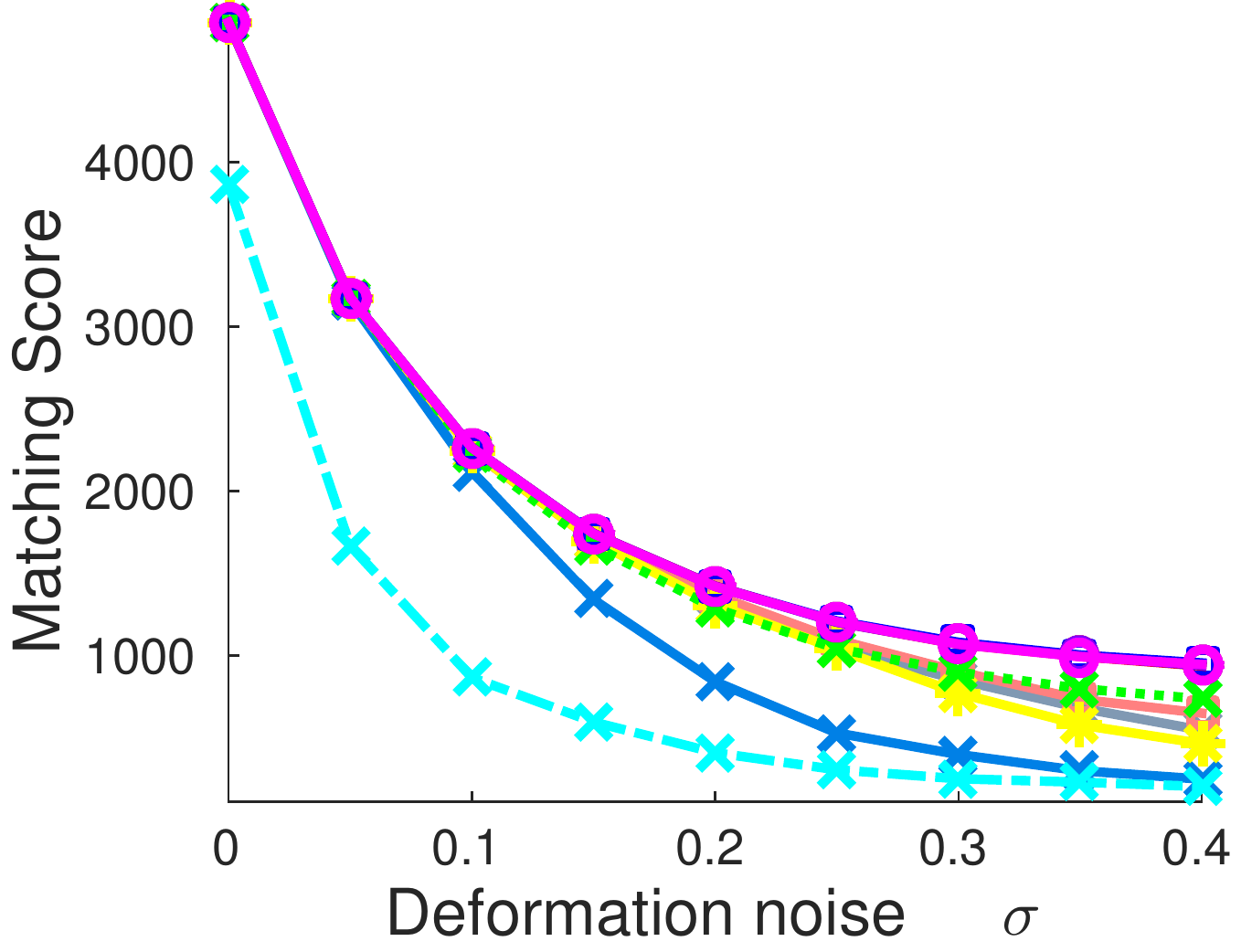} } 
    \subfloat{\includegraphics[width=0.31\linewidth]{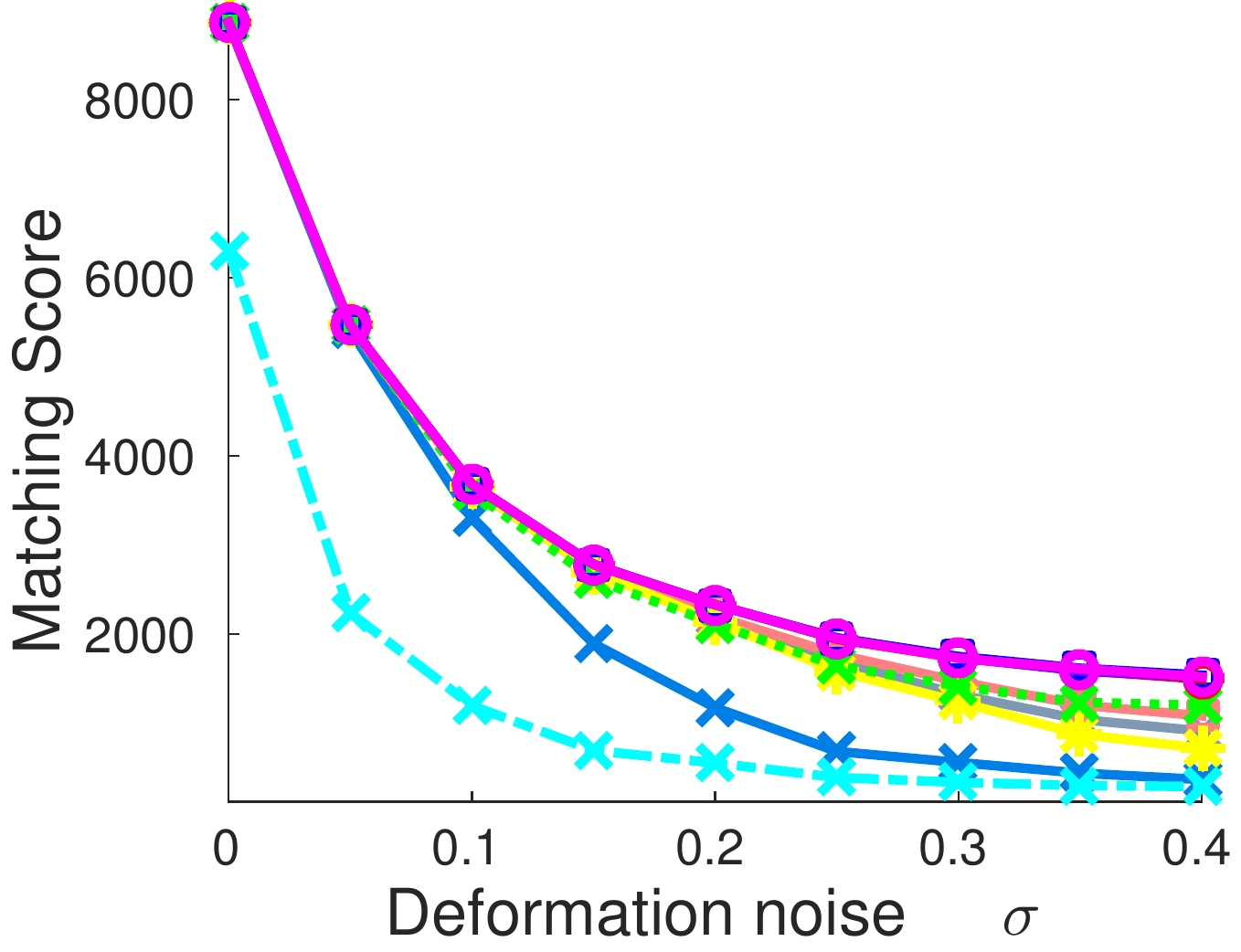} }  \\
    
    \setcounter{subfigure}{0}
    \subfloat[$n_{in}=20$]{ \includegraphics[width=0.31\linewidth]{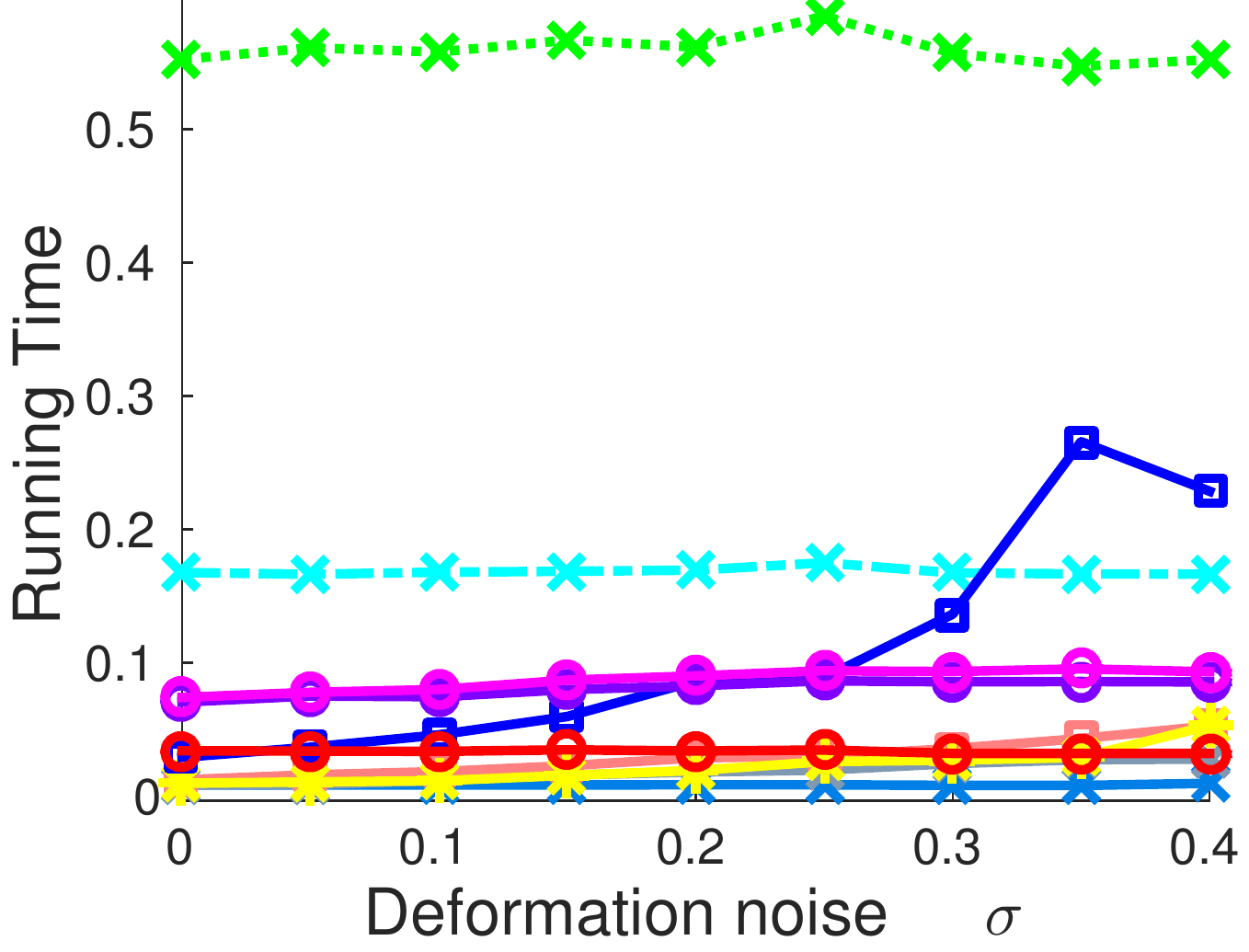} \label{appendix_fig:out_def_noscal} } 
    \subfloat[$n_{in}=30$]{ \includegraphics[width=0.31\linewidth]{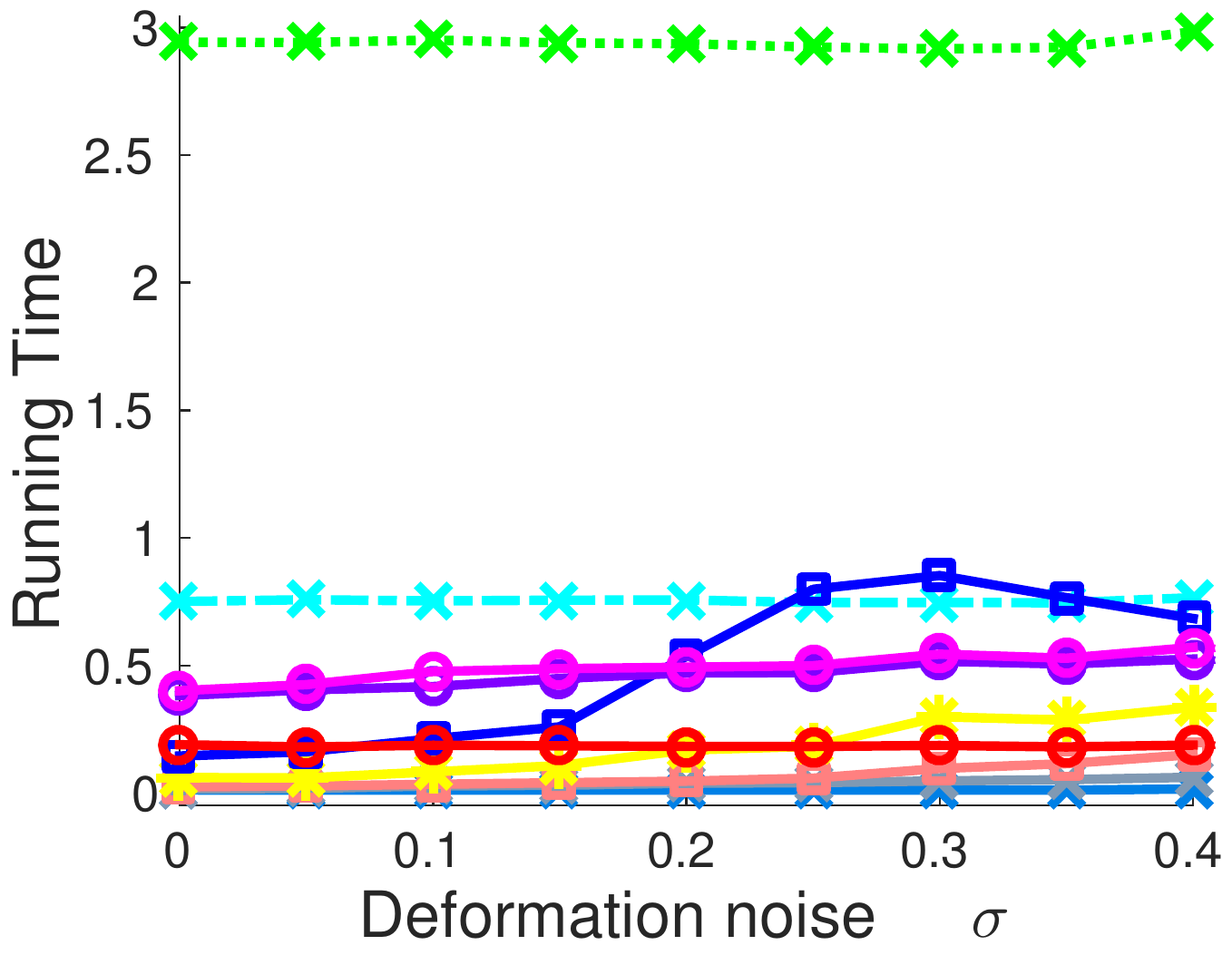} \label{appendix_fig:out_def_scal} } 
    \subfloat[$n_{in}=40$]{ \includegraphics[width=0.31\linewidth]{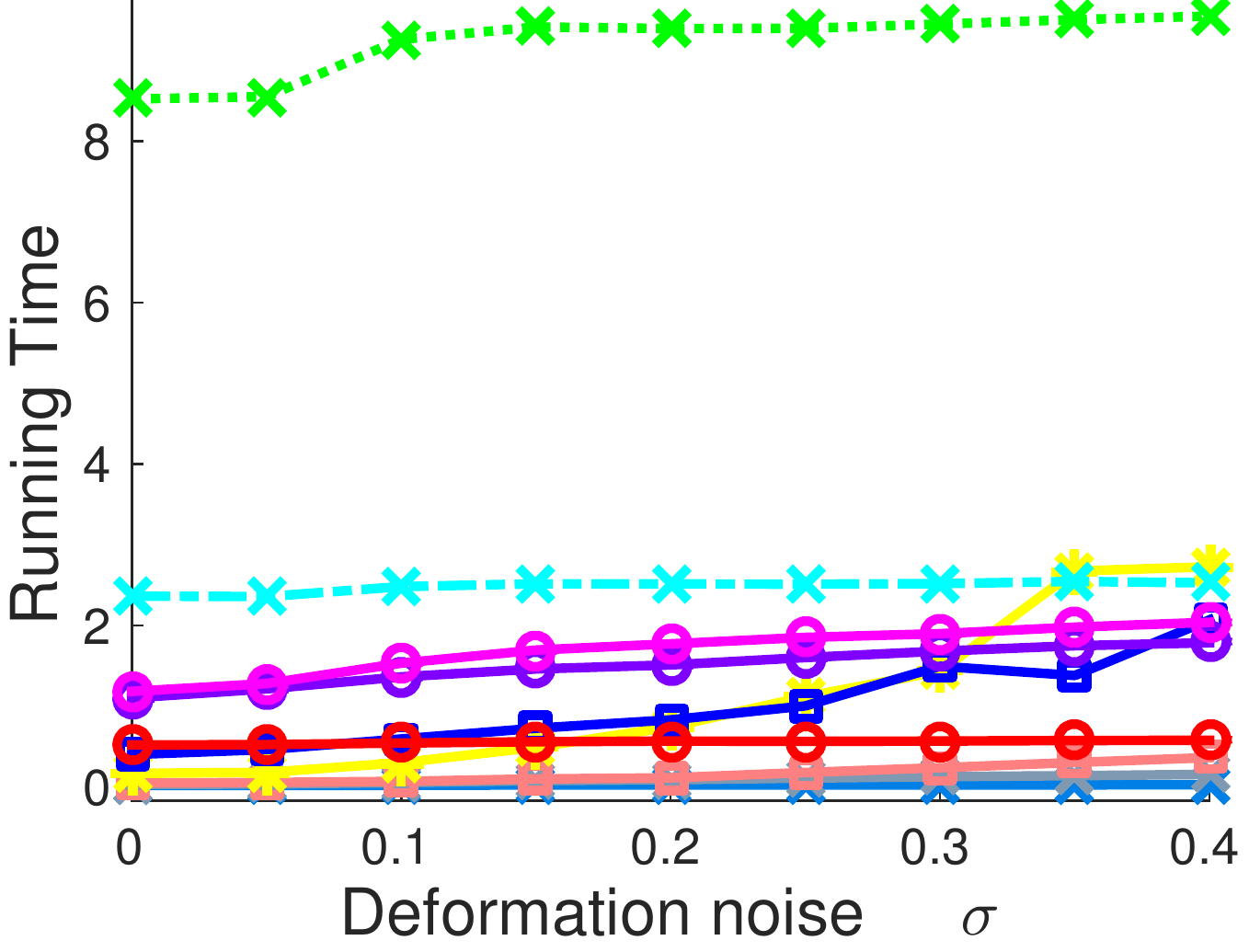} \label{appendix_fig:def} }
\end{center}    
\caption{
Matching point sets in $\R^2$ (deformation test):
The top row shows average accuracy while the middle row shows
average objective score and the bottom row shows average running time. 
(a) 20 inliers
(b) 30 inliers
(c) 40 inliers
(Best viewed in color.)
}
\label{appendix_fig:exp_def}
\end{figure*}

\section{CMU House Dataset}
Similar to the experiments done in Section \ref{subsec:exp_house}, we evaluate GM algorithms on two tasks 
where we match $15$ points to $30$ points and $25$ points to $30$ points in two corresponding images.
For each task, we match all the possible image pairs and compute the average result for each baseline.
The results in Figure \ref{appendix_fig:exp_house} show that our algorithms achieve competitive or better results than other methods for all the baselines.
\begin{figure*}[htb!]
\begin{center}
    \subfloat{ \includegraphics[width=0.31\linewidth]{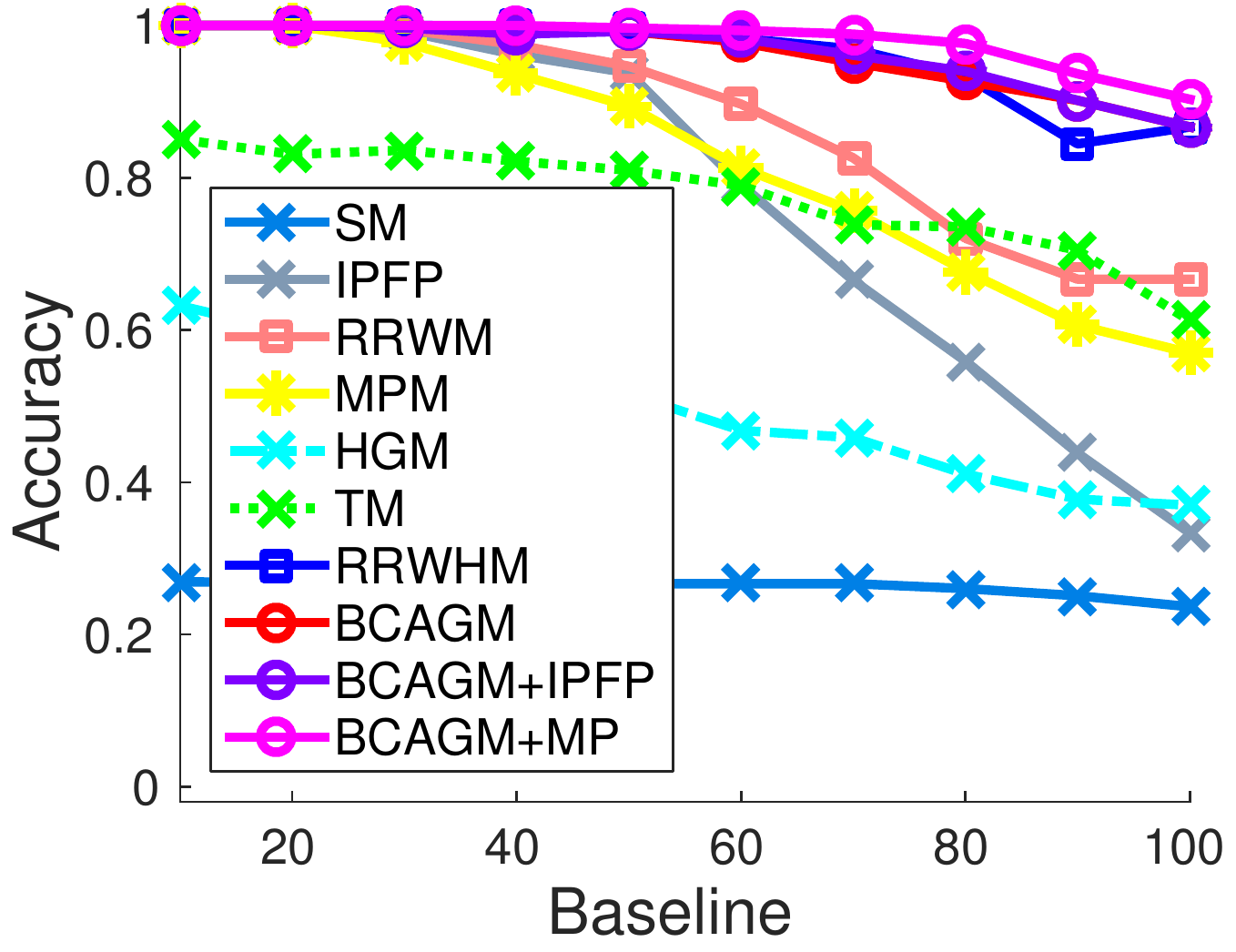} }  
    \subfloat{ \includegraphics[width=0.31\linewidth]{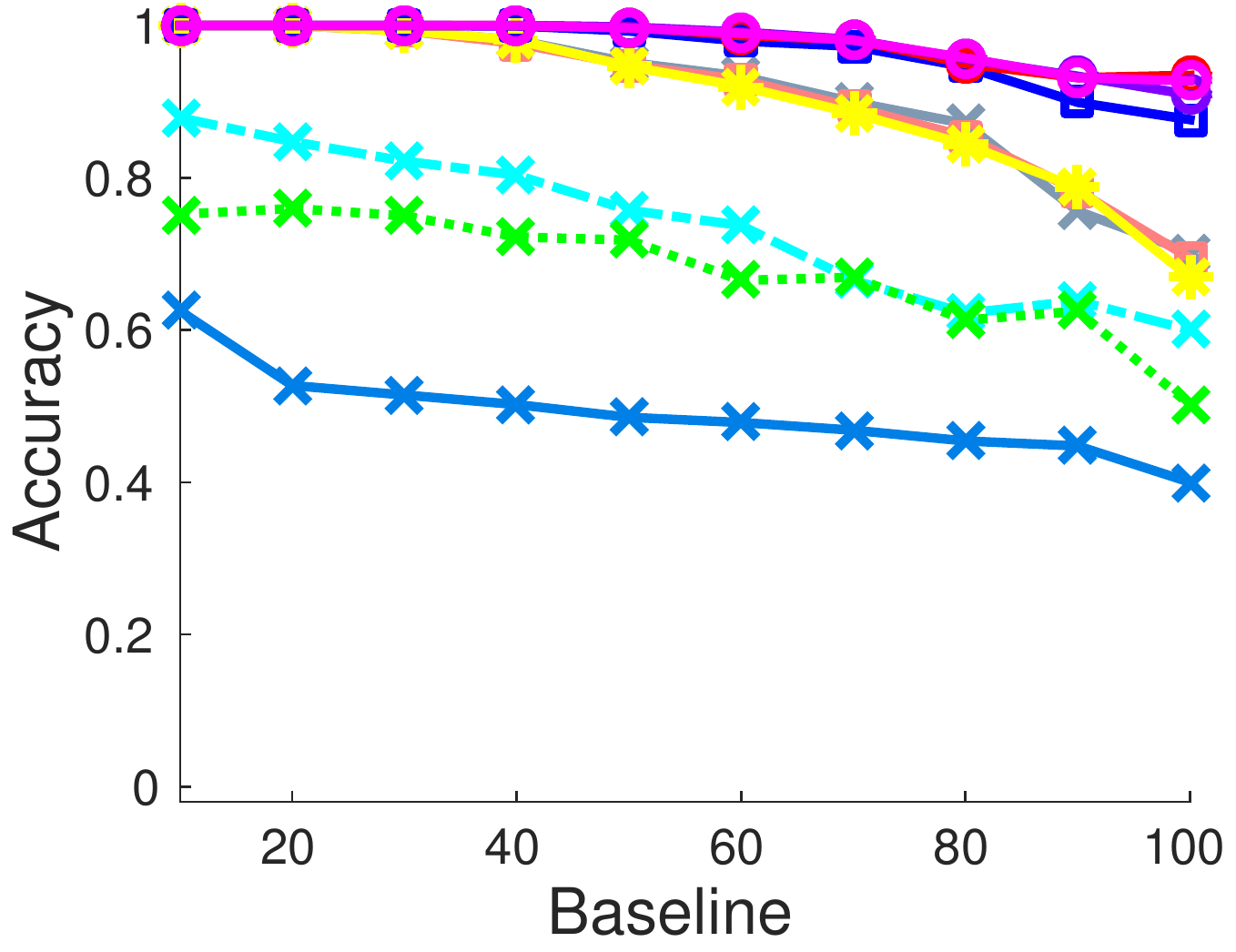} } \\
    \subfloat{ \includegraphics[width=0.31\linewidth]{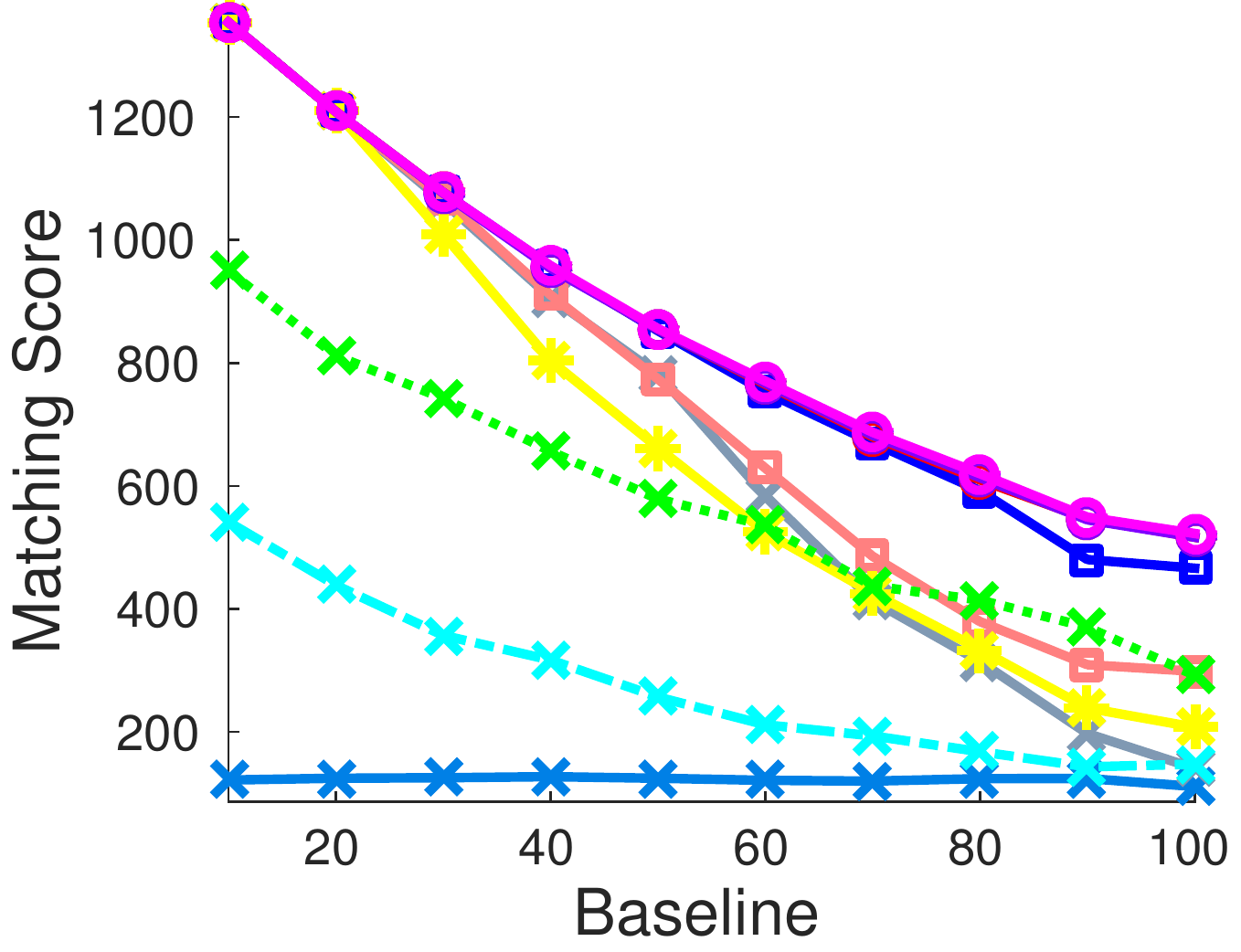} }  
    \subfloat{ \includegraphics[width=0.31\linewidth]{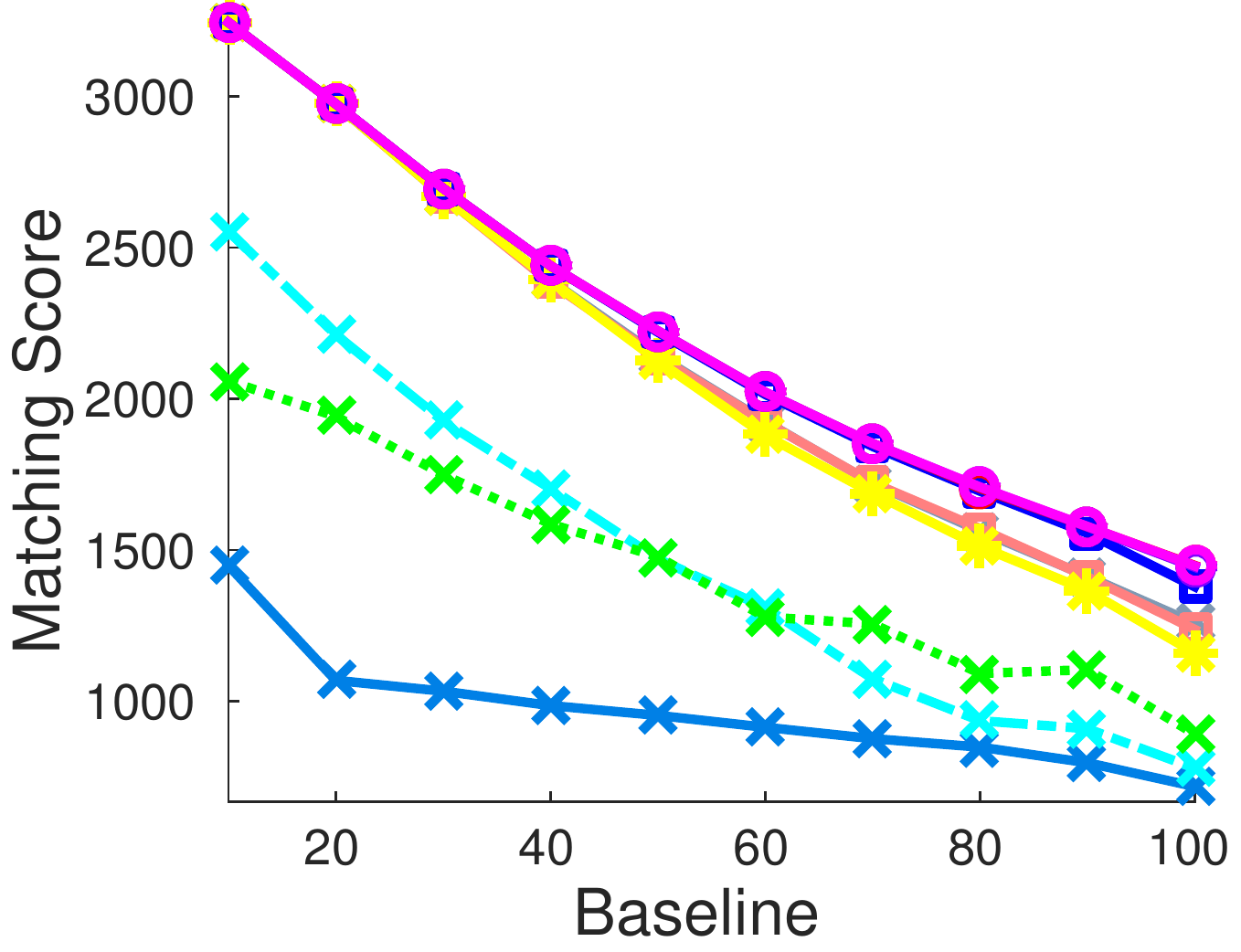} } \\
    \setcounter{subfigure}{0}
    \subfloat[15 pts vs 30 pts]{ \includegraphics[width=0.31\linewidth]{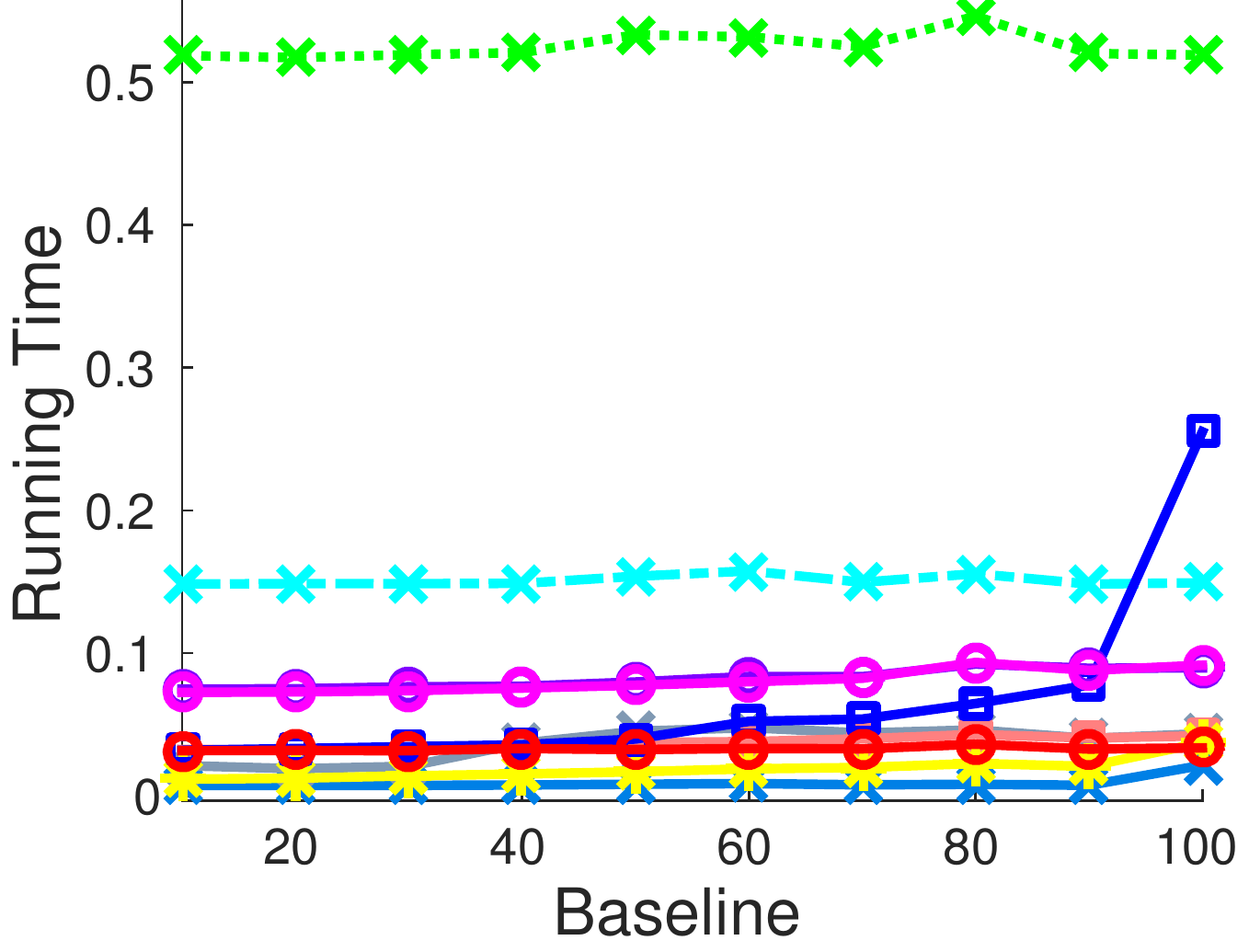} }  
    \subfloat[25 pts vs 30 pts]{ \includegraphics[width=0.31\linewidth]{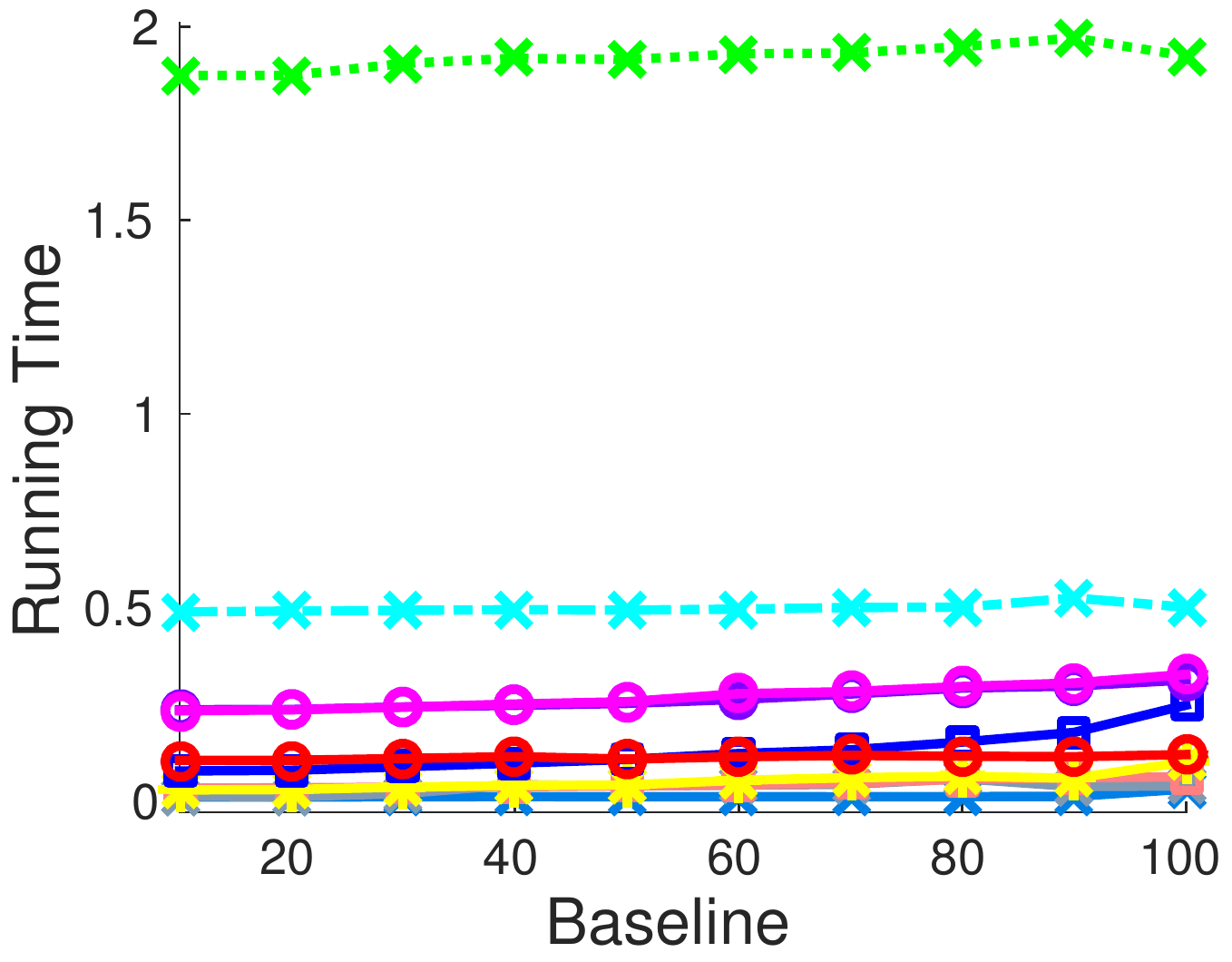} } 
\end{center}
\caption{
CMU house dataset:
The top row shows average matching accuracy while the middle row shows average 
objective score and the bottom row shows average running time. 
(Best viewed in color.)
}
\label{appendix_fig:exp_house}
\end{figure*}

\section{Car Motorbike Dataset}
Figure \ref{appendix_fig:exp_car_motor} shows the running time of higher order methods for the experiments done in 
Section \ref{subsec:exp_car_motor} on the Car and Motorbike dataset.
\begin{figure*}[htb!]
\begin{center}
    \subfloat[Car dataset]{ \includegraphics[width=0.31\linewidth]{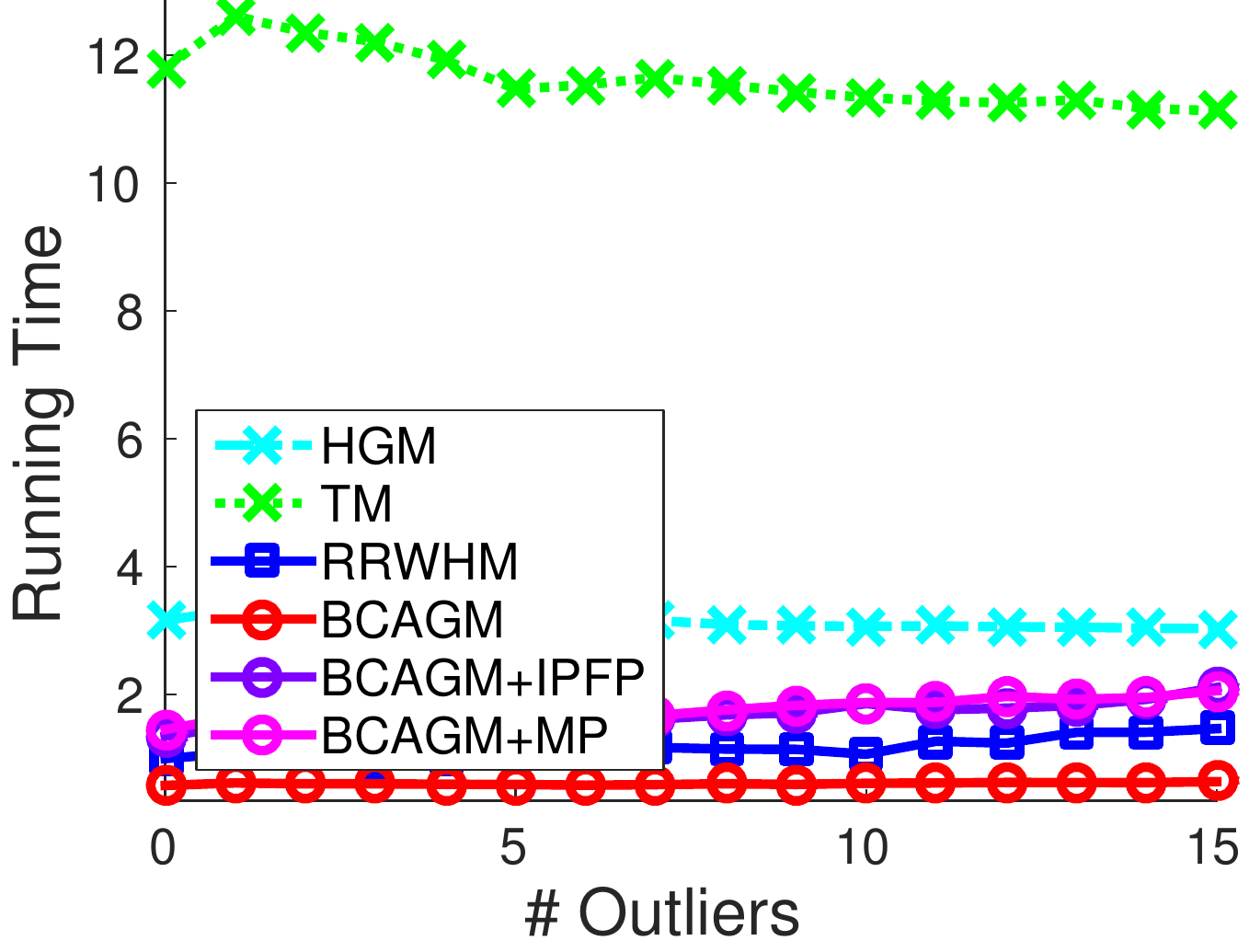} } 
    \subfloat[Motorbike dataset]{ \includegraphics[width=0.31\linewidth]{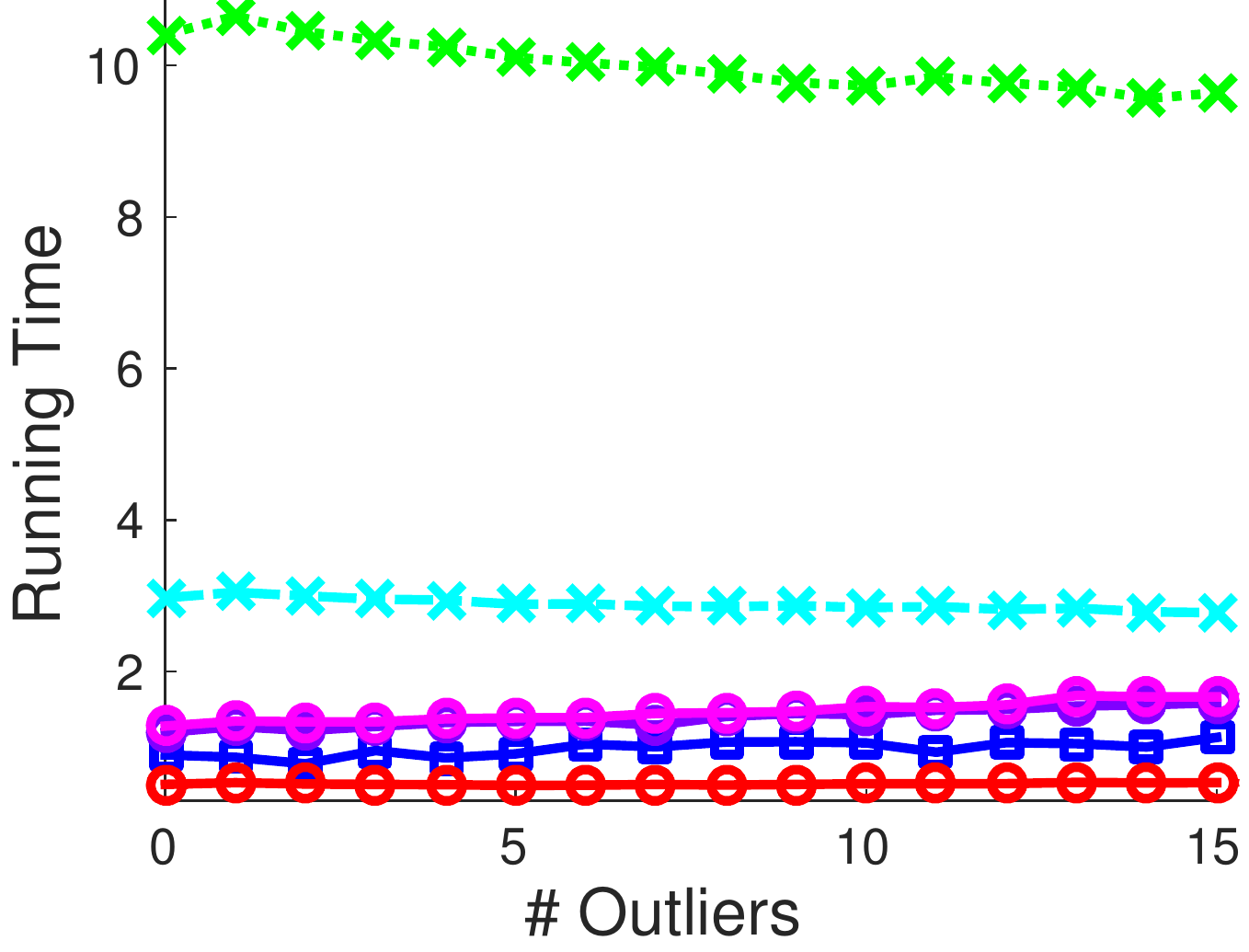} } 
\end{center}
\caption{
Car and Motorbike dataset: Running time of the third order methods. 
(Best viewed in color.)
}
\label{appendix_fig:exp_car_motor}
\end{figure*}

\end{document}